\documentclass{article}

\usepackage{arxiv}

\usepackage{amsmath,amsfonts,bm}

\def\eqref#1{equation~\ref{#1}}

\def\1{\bm{1}}

\DeclareMathAlphabet{\mathsfit}{\encodingdefault}{\sfdefault}{m}{sl}
\SetMathAlphabet{\mathsfit}{bold}{\encodingdefault}{\sfdefault}{bx}{n}

\newcommand{\E}{\mathbb{E}}

\newcommand{\R}{\mathbb{R}}

\newcommand{\Var}{\mathrm{Var}}

\DeclareMathOperator*{\argmin}{arg\,min}

\DeclareMathOperator{\sign}{sign}

\usepackage[utf8]{inputenc}
\usepackage[T1]{fontenc}
\usepackage[nopatch=footnote]{microtype}
\usepackage{graphicx}
\usepackage{booktabs}
\usepackage{amsfonts}
\usepackage{nicefrac}
\usepackage{xcolor}
\usepackage{hyperref}
\usepackage{url}
\usepackage{stfloats}
\usepackage[numbers,sort&compress]{natbib}

\usepackage{amsmath,amssymb,amsthm,mathtools}
\usepackage{subcaption}
\usepackage{caption}
\usepackage{multirow}
\usepackage{adjustbox}
\usepackage{tabularx}
\usepackage{array}
\usepackage{longtable}
\usepackage{siunitx}
\usepackage{algorithm}
\usepackage{algpseudocode}
\usepackage{enumitem}
\usepackage[capitalize,noabbrev]{cleveref}
\usepackage[textsize=tiny]{todonotes}

\makeatletter

\@ifundefined{theHALG@line}{}{%
  \renewcommand{\theHALG@line}{\thealgorithm.\arabic{ALG@line}}%
}
\makeatother

\newcommand{\norm}[1]{\left\lVert#1\right\rVert}
\newcommand{\V}[1]{\boldsymbol{#1}}
\newcommand{\clip}{\mathrm{clip}}
\newcommand{\phiPLE}{\phi}
\newcommand{\DeltaT}{\Delta}
\newcommand{\ReLU}{\mathrm{ReLU}}
\newcommand{\Pp}{\mathbb{P}}
\newcommand{\ip}[2]{\langle #1,#2\rangle}
\newcommand{\abs}[1]{\lvert #1\rvert}
\newcommand{\ind}{\mathbf{1}}
\newcommand{\prox}{\mathrm{prox}}
\newcommand{\half}{\tfrac{1}{2}}
\newcommand{\soft}{\mathcal{S}}

\newtheorem{example}{Example}
\theoremstyle{plain}
\newtheorem{theorem}{Theorem}[section]
\newtheorem{proposition}[theorem]{Proposition}
\newtheorem{lemma}[theorem]{Lemma}
\newtheorem{corollary}[theorem]{Corollary}
\theoremstyle{definition}
\newtheorem{definition}[theorem]{Definition}
\newtheorem{assumption}[theorem]{Assumption}
\theoremstyle{remark}
\newtheorem{remark}[theorem]{Remark}

\title{LassoFlexNet: a Flexible Neural Architecture for Tabular Data}
\date{}

\author{
Kry Yik Chau Lui\thanks{Equal contribution} \\
RBC Borealis \\
\texttt{yikchau.y.lui@borealisai.com}
\And
Cheng Chi\footnotemark[1] \\
RBC Borealis \\
\texttt{cheng.chi@borealisai.com}
\And
Kishore Basu \\
RBC \\
\texttt{kishore.basu@rbc.com}
\And
Yanshuai Cao \\
RBC Borealis \\
\texttt{yanshuai.cao@borealisai.com}
}

\renewcommand{\headeright}{Preprint}
\renewcommand{\undertitle}{Preprint}
\renewcommand{\shorttitle}{LassoFlexNet: Flexible Neural Architecture for Tabular Data}

\hypersetup{
  pdftitle={LassoFlexNet: a Flexible Neural Architecture for Tabular Data},
  pdfauthor={Kry Yik Chau Lui, Cheng Chi, Kishore Basu, Yanshuai Cao},
  pdfkeywords={Machine Learning, Tabular Data, Sparse Deep Learning}
}

\begin{document}
\maketitle

\begin{abstract}
Despite their dominance in vision and language, deep neural networks often underperform relative to tree-based models on tabular data. 
To bridge this gap, we incorporate five key inductive biases into deep learning: robustness to irrelevant features, axis alignment, 
localized irregularities, feature heterogeneity, and training stability. We propose \emph{LassoFlexNet}, 
an architecture that evaluates the linear and nonlinear marginal contribution of each input via Per-Feature Embeddings, 
and sparsely selects relevant variables using a Tied Group Lasso mechanism. 
Because these components introduce optimization challenges that destabilize standard proximal methods, 
we develop a \emph{Sequential Hierarchical Proximal Adaptive Gradient optimizer with exponential moving averages (EMA)} to 
ensure stable convergence. Across $52$ datasets from three benchmarks, LassoFlexNet matches or outperforms leading tree-based models, 
achieving up to a $10$\% relative gain, while maintaining Lasso-like interpretability. We substantiate these empirical results with ablation studies and theoretical proofs 
confirming the architecture's enhanced expressivity and structural breaking of undesired rotational invariance.
\end{abstract}


\section{Introduction}
\label{sec:intro}

Despite the dominance of deep neural networks in vision and language \citep{he2016deep, vaswani2017attention}, 
tree-based methods such as XGBoost and CatBoost \citep{chen2016xgboost, prokhorenkova2018catboost} 
remain the default choice for tabular data. \citet{grinsztajn2022tree} identified three primary reasons for this performance gap: 
(1) robustness to uninformative features, (2) preservation of data orientation, and (3) the ability to learn irregular, localized functions.

We argue that input heterogeneity-the mix of varying abstraction levels-is a fourth, overlooked but equally critical factor. 
Unlike homogeneous vision or text data, tabular datasets frequently mix raw physical measurements (e.g., headcount) with 
complex derived metrics (e.g., price-to-equity ratio). Standard neural networks linearly blend these disparate inputs into 
a shared representation space during the first layer, irreparably obscuring their marginal information. 
Tree-based models, conversely, recursively partition the input space in an axis-aligned manner, 
natively isolating heterogeneity and modeling sharp, step-like decision boundaries.

Fusing these four inductive biases into deep learning while ensuring robust optimization requires moving fundamentally beyond 
architectures amenable to standard stochastic gradient descent. We propose \emph{LassoFlexNet}, an architecture engineered to 
mimic the coordinate-wise localized flexibility of trees, coupled with a novel optimization algorithm to resolve the resulting 
non-convex non-smooth training dynamics. 
Our primary contributions are:
\begin{enumerate} 
    \item \textbf{Tabular-Native Architecture}: We introduce a Tied Group Lasso mechanism that operates on non-linear 
    Per-Feature Embeddings rather than raw inputs. Furthermore, we rigorously prove that Piecewise Linear Encoding (PLE) 
    fundamentally breaks the undesired rotational invariance of neural networks \citep{ng2004feature} and 
    amplifies functional expressivity for localized irregularities (
        Propositions \ref{prp:informal_ple_not_invariant} and \ref{prp:informal_ple_more_expressivie}).
    \item \textbf{Stable Proximal Optimization}: 
    We derive a \textit{Sequential Hierarchical Proximal Adaptive Gradient optimizer with exponential moving averages (Seq-Hier-Prox-EMA)}. 
    This algorithm decouples the optimization trajectory, ensuring the Lasso selection mechanism and the deep MLP-Mixer module 
    \citep{tolstikhin2021mlp} learn synergistically, overcoming the severe training instabilities of prior sparse neural methods 
    like LassoNet \citep{lemhadri2021lassonet}.
    \item \textbf{State-of-the-Art Empirical Performance}: Across 52 datasets spanning three recent benchmarks 
    \citep{cherepanova2023performance, rubachev2025tabred}, \emph{LassoFlexNet} consistently matches or surpasses leading tree-based models, 
    achieving up to a 10\% relative improvement, while providing high-fidelity feature importance metrics.
\end{enumerate}
Overall, \emph{LassoFlexNet} bridges the representational and optimization gaps between neural networks and decision trees, 
offering a principled, robust foundation for deep tabular learning.

\section{Related Work}
\label{sec:related_work}

\paragraph{Lasso-Based Approaches.} Lasso regression \citep{tibshirani1996regression} utilizes $\ell_1$ regularization for variable selection, 
typically solved via proximal gradient methods. 
LassoNet \citep{lemhadri2021lassonet} extends this paradigm to deep learning by employing a linear skip connection to guide a neural module. 
While theoretically sound for inducing feature sparsity,
LassoNet frequently underperforms on modern tabular benchmarks (as shown in Sec.\ \ref{sec:result}) 
and fails to serve as a reliable alternative to tree-based models. 
\emph{LassoFlexNet} addresses these deficiencies through two critical innovations:
\begin{itemize}
    \item \textbf{Architecture}: 
    Rather than performing selection on raw inputs, we introduce Per-Feature Embeddings (PFE) combined with a Tied Group Lasso mechanism. 
    This enables feature selection on learned, nonlinear representations, 
    capturing marginal predictive contributions that raw inputs might obscure.
    \item \textbf{Optimization}: We replace LassoNet's joint optimization with a sequential procedure utilizing 
    proximal adaptive gradients stabilized via exponential moving averages (EMA).
\end{itemize}

\paragraph{Neural Networks on Tabular Data.} The broader landscape of deep tabular learning has focused heavily on 
adapting interaction modules to rival decision trees. Models such as TabNet \citep{arik2021tabnet} and 
FT-Transformer \citep{gorishniy2021revisiting} leverage attention and feature-interaction mechanisms. 
Recently, TabM \citep{gorishniy2025tabm} enhanced these backbones via parameter-efficient ensembling, 
improving generalization through diverse predictions. 

\paragraph{In-Context and Retrieval-Based Methods.} An orthogonal class of modern methods focuses on retrieval and in-context learning. 
TabR \citep{gorishniy2024tabr} utilizes a retrieval-augmented design to query similar training examples at inference. 
Alternatively, TabPFN \citep{hollmann2023tabpfntransformersolvessmall} and its successor 
TabPFN-v2.5 \citep{grinsztajn2025tabpfn25advancingstateart} operate as pre-trained in-context learners optimized over diverse synthetic 
and real datasets, an approach that TabICL \citep{qu2025tabicl} scales further into a foundation model. 
However, the quadratic complexity of their attention mechanisms, coupled with inherent context window limitations, 
restricts these approaches primarily to small-to-medium datasets, leaving scalable, 
tree-competitive representation learning an open problem that \emph{LassoFlexNet} directly targets.

\begin{table}[ht]
\centering
\caption{LassoFlexNet components addressing tabular data challenges.}
\label{tab:contributions_summary}
\small
\setlength{\tabcolsep}{2pt} 
\renewcommand{\arraystretch}{0.9}
\begin{tabular}{l@{\hspace{6pt}}p{3.3cm}@{\hspace{6pt}}p{2.9cm}}
\toprule
\textbf{Error} & \textbf{Inductive Bias} & \textbf{Solution} \\ 
\midrule
Generalization & Robustness to irrelevant features; axis alignment & Lasso Linear Skip Connection with Tied Group Lasso \\ 
Optimization & Stable training & Sequential Hierarchical Proximal Adaptive Gradient optimizer + EMA \\ 
Approximation & Axis alignment, localized irregularities; heterogeneity; complex interactions & PLE; PFE; MLP-Mixer Module \\ 
\bottomrule
\end{tabular}
\end{table}

\section{Design Principles for Tabular Inductive Biases}
\label{sec:design_principle}

To rival the performance of tree-based models, neural architectures must natively encode five critical inductive biases: 
robustness to irrelevant features, axis alignment, localized irregularities, feature heterogeneity, 
and training stability (\ref{sec:intro}). 
We formalize our design approach by decomposing the test loss into three distinct error components—generalization, 
optimization, and approximation:
\begin{align}
\mathcal{L}_{\text{Test}} &= 
\underbrace{\mathcal{L}_{\text{Test}} - \mathcal{L}_{\text{Train}}}_{\text{Generalization Error}} 
+ \underbrace{\mathcal{L}_{\text{Train}} - \mathcal{L}_{\text{Train, Min}}}_{\text{Optimization Error}} 
\phantom{\mathcal{L}_{\text{Test}}} + \underbrace{\mathcal{L}_{\text{Train, Min}}}_{\text{Approximation Error}}
\end{align}

\emph{LassoFlexNet} structurally bounds these errors via targeted architectural and algorithmic interventions 
(summarized in Table~\ref{tab:contributions_summary}):
\begin{itemize}
    \item \textbf{Approximation Error}: To model localized irregularities and manage feature heterogeneity, 
    we utilize \textbf{Piecewise Linear Encoding (PLE)}~\citep{gorishniy2022embeddings} and \textbf{Per-Feature Embeddings (PFE)}. 
    A subsequent \textbf{MLP-Mixer Module}~\citep{tolstikhin2021mlp} captures complex cross-feature interactions.
    \item \textbf{Generalization Error}: To preserve axis orientation and strictly mitigate overfitting to irrelevant features, 
    we construct a \textbf{Linear Skip Connection governed by a Tied Group Lasso} penalty over the learned embeddings.
    \item \textbf{Optimization Error}: To guarantee stable convergence across these high-capacity, non-homogeneous modules, 
    we derive a \textbf{Seq-Hier-Prox-Adam-EMA} optimizer.
\end{itemize}

Sections \ref{sec:model_detail} and \ref{sec:coordinate} detail these components mathematically, 
before Section \ref{app:address_3_questions} analyzes how they jointly satisfy the required inductive biases.

\section{Architecture}
\label{sec:model_detail}

\subsection{Preliminaries: LassoNet}
\label{sec:preliminary_lassonet}
LassoNet \citep{lemhadri2021lassonet} integrates feature selection by augmenting a standard neural module, $\mathrm{NN}_{W}(\V{x})$, with a linear skip connection, $\V{\beta}^\top \V{x}$, yielding the prediction $\hat{y} = \V{\beta}^\top \V{x} + \mathrm{NN}_{W}(\V{x})$. The linear coefficients $\V{\beta}$ and neural weights $W$ are jointly optimized via a hierarchical penalty: 
\begin{align}
\min_{\V{\beta}, W} L(\V{\beta}, W) + \lambda \|\V{\beta}\|_1 \quad \text{s.t.} \quad \|W^{(1)}_{i,\_}\|_\infty \leq M |\beta_i|
\end{align}
where $W^{(1)}_{i,\_}$ denotes the outgoing weights from the $i$-th input. This $\ell_1$ penalty promotes strict sparsity on $\V{\beta}$, while the $\ell_\infty$ constraint ensures that a feature can only be utilized by the neural network if its corresponding linear counterpart is non-zero. However, LassoNet's reliance on raw input correlations severely limits its ability to capture nonlinear feature importance. Furthermore, its joint Hier-Prox optimization often suffers from poor, unstable training dynamics (analyzed in Section \ref{sec:ablation}).

\subsection{LassoFlexNet Overview}
\label{sec:lassoflex_overview}
\emph{LassoFlexNet} (Figure~\ref{fig:lassoflexnet_model}) circumvents these limitations via a three-component design: (1) independent per-feature neural networks $\V{z}_i = \mathrm{NN}_{\V{\phi}_i}(x_i)$ utilizing Piecewise Linear Encoding (PLE) and Per-Feature Embedding (PFE); (2) a Tied Group Linear Skip Connection $\V{\beta}^\top \V{\bar{z}}$, operating on the mean-pooled representation $\V{\bar{z}}_i$ of each embedding; and (3) a deep MLP-Mixer \citep{tolstikhin2021mlp} module $\tau \cdot \mathrm{NN}_{W}(\V{z})$, where $\tau$ acts as a soft-curriculum scale to control the inductive bias toward simplicity if $\tau < 1$ and towards expressivity if $\tau > 1$ (Eq.~\ref{eq:lassoflexnet_tau}). Formally, we define:
\begin{align}
    \V{z} &= [\dots, \V{z}_i, \dots], \quad \V{\bar{z}} = [\dots, \text{mean}(\V{z}_i), \dots] \nonumber \\
    \hat{y} &= \V{\beta}^\top \V{\bar{z}} + \tau \cdot \mathrm{NN}_{W}(\V{z})  \label{eq:lassoflexnet_tau} \\
    \V{\beta}_{\text{Lasso}} &= \arg \min_{\V{\beta}} L(\V{\phi}, \V{\beta}, W) + \lambda \|\V{\beta}\|_1 \nonumber \\ 
    \text{s.t. } &\|W^{(1)}_{i, \_}\|_\infty \leq M | \V{\beta}_{\text{Lasso}; i}| \nonumber
\end{align}
where $i$ indexes the $i$-th feature.

\begin{figure*}[t]
    \centering
    \includegraphics[width=0.9\linewidth]{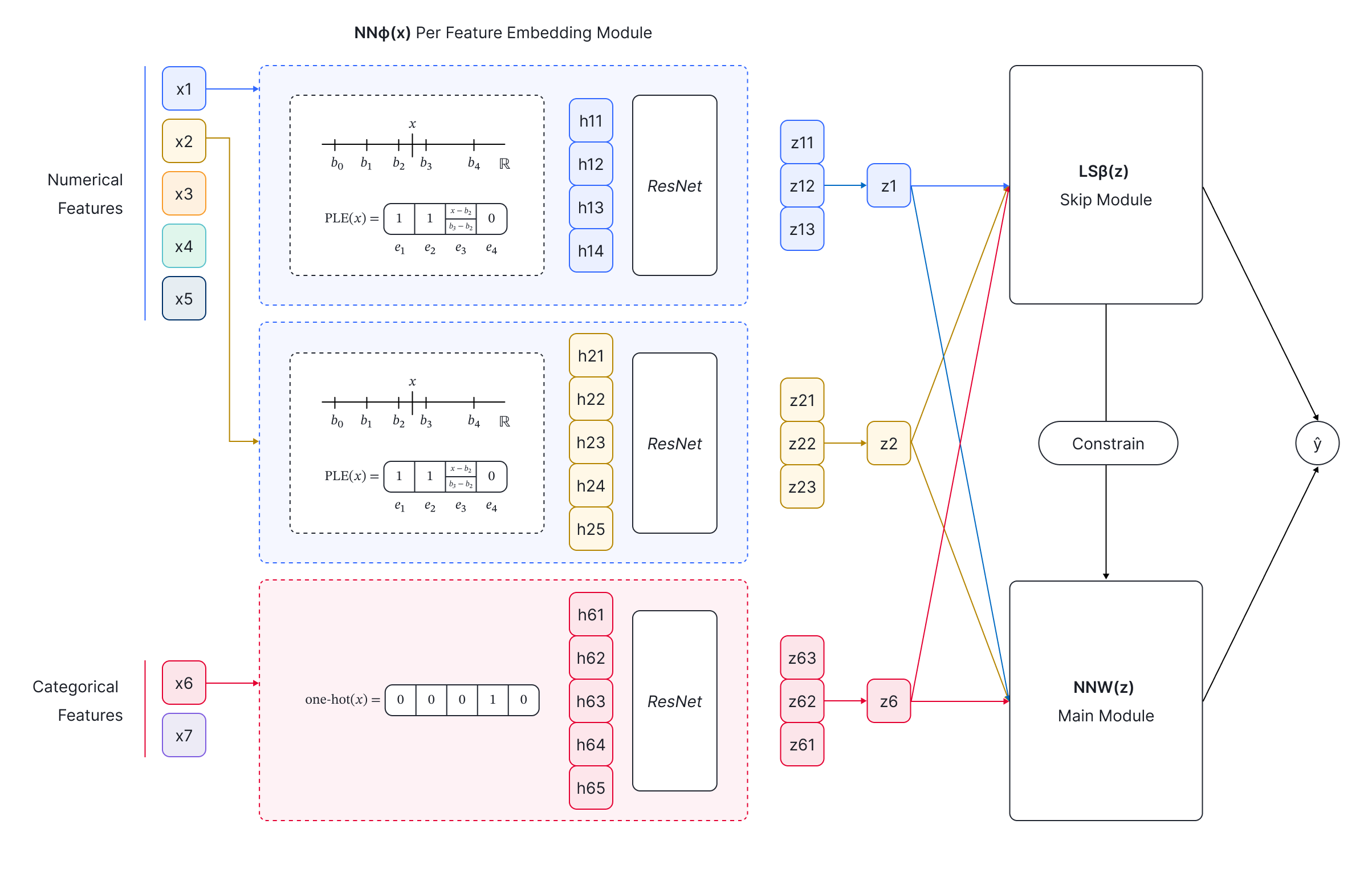}
    \caption{\emph{LassoFlexNet} architecture: Per-Feature Embedding (PLE for numerical, one-hot for categorical), Linear Skip Connection, and the MLP-Mixer Module.}
    \label{fig:lassoflexnet_model}
\end{figure*}

\subsubsection{Piecewise Linear Encoding and Per-Feature Embedding}
\label{sec:motif_nonlinear}
To inherently adapt to tabular heterogeneity, PLE \citep{gorishniy2022embeddings} partitions each numerical feature $x_i$ into $K_i$ distinct intervals via decision tree splits $\{b_{i,k}\}$. The resulting encoding $\text{PLE}(x_i) = \V{e}_i = [e_{i,1}, \dots, e_{i,K_i}]$ is defined as:
\begin{equation}
\label{eq:ple}
e_{i,t} = \max\left(0, \min\left(1, \frac{x_i - b_{i,t-1}}{b_{i,t} - b_{i,t-1}}\right)\right)
\end{equation}
This formulation generates $K_i \cdot k$ localized gradients during backpropagation, lifting the neural capacity at a linear rate to successfully capture highly irregular, localized target dependencies (proved rigorously in Appendix~\ref{app:ple_more_expressivie}). 

Subsequently, PFE transforms $\V{e}_i$ through an independent per-feature ResNet block: $\V{h}_{i,l} = \V{h}_{i,l-1} + \text{LN}(\text{ReLU}(W_{i,l} \V{h}_{i,l-1} + \V{b}_{i,l}))$. To ensure that the feature selection via the magnitude of $\V{\beta}$ remains valid and interpretable \citep{friedman2009elements}, the final embedding $\text{PFE}(\V{e}_i)$ is normalized via parameter-free BatchNorm ($\mu=0, \sigma=1$). Unlike existing soft regularization attempts \citep{sudo2021lassolayer}, PFE supports arbitrary architectural depth and allows exact feature removal via proximal gradients.

\subsubsection{Tied Group Lasso Mechanism}
\label{sec:skip_and_module}

LassoNet’s fundamental reliance on raw inputs ignores nonlinear dependencies, risking a scenario where "the blind guide the blind" if the linear skip connection converges slower than the highly-expressive neural module. \emph{LassoFlexNet} resolves this tension through two mechanisms:

\paragraph{Selection via Nonlinear Embeddings.} We apply the linear skip connection strictly to the mean-pooled embeddings $\V{\bar{z}}_i$. This mathematically implements a \emph{Tied Group Lasso}, wherein a single scalar $\beta_i$ controls the entire $i$-th embedded group. This weight-sharing severely reduces parameter count and enables feature selection based on rich nonlinear representations, while preserving interpretability since $\V{z}$ is rigorously pre-normalized.

\paragraph{Soft Curriculum Learning.} To prevent the expressive MLP-Mixer module from instantly dominating the training trajectory and bypassing feature selection, we may scale its output by $\tau < 1$ (Eq.~\ref{eq:lassoflexnet_tau}). This structurally downscales the Mixer gradients, forcing early convergence of the linear, selection-guiding components. This Lasso-Mixer pairing is vital for tabular generalization; without it, the Mixer rapidly overfits small tabular datasets, while substituting a simpler MLP fatally reduces representational capacity (demonstrated in Sec.~\ref{sec:ablation}).

\section{Optimization}
\label{sec:coordinate}

LassoNet's foundational methodological contribution is its Hier-Prox-Gradient (HPG) optimization (Algorithm~\ref{alg:lassonet_prox} in App. \ref{appendix:lassonet_prox}), which enables strict feature sparsity under stochastic optimization. Let $(\tilde{\V{\theta}}, \tilde{W})$ denote the proximal-processed parameters. HPG operates by interleaving two steps: (1) an unconstrained regular update via SGD, $(\tilde{\V{\theta}}_{t-1}, \tilde{W}_{t-1}) \to (\V{\theta}_{t}, W_{t})$, and (2) a proximal operation enforcing the hierarchical constraint, $(\V{\theta}_{t}, W_{t}) \to (\tilde{\V{\theta}}_{t}, \tilde{W}_{t})$. If $\lambda = 0$, the HPG reduces to SGD. While \emph{LassoFlexNet} adopts this alternating structure (Algorithm~\ref{alg:Lasso_train}), we fundamentally redesign the proximal operator to overcome three severe failure modes of HPG.

\begin{algorithm}
\caption{{Update Algorithm for LassoNet / LassoFlexNet}}
\label{alg:Lasso_train}
\begin{algorithmic}[1]
\State \textbf{Input:} training dataset $X \in \mathbb{R}^{n \times d}$, training labels $Y$, Lasso parameter $\beta$, neural model $g_W(\cdot)$, number of epochs $N$, hierarchy multiplier $M$, path multiplier $\varepsilon$, learning rate $\eta$
\State Initialize and train the feed-forward network on the loss $L(\beta, W)$
\State Initialize the penalty path, $\boldsymbol{\lambda} = [0, \lambda_1 \cdots, ]$, and the number of active features, $k = d$.

\For{$\lambda \in \boldsymbol{\lambda}$}
    \While{$k > 0$}
        \For{$b = 1$ to $N$}
            \State Compute $\nabla L(\beta, W)$
            \State Update $(\beta, W)$ via SGD (LassoNet or LassoFlexNet) or Adam (LassoFlexNet only)
            \If{$\lambda \neq 0$}
                \State Update $(\beta, W^{(1)})$ by Hier-Prox (LassoNet) or Seq-Hier-Prox (LassoFlexNet)
            \EndIf
        \EndFor
        \State Update $k$ to be the number of non-zero coordinates of $\beta$
    \EndWhile
\EndFor
\end{algorithmic}
\end{algorithm}
We improve in three ways: 1) joint→sequential optimization; 2) proximal gradient→proximal adaptive gradient; 3) EMA stabilization. Algorithm \ref{alg:seq-hier-prox} details our method:
\begin{algorithm}[H]
\caption{LassoFlexNet Seq-Hier-Prox Adam EMA Operator}
\label{alg:seq-hier-prox}
\begin{algorithmic}[1]
\Procedure{Seq-Hier-Prox-Adam-EMA}{$\beta, W^{(1)}, \beta_{\text{EMA}}, W^{(1)}_{\text{EMA}}; \lambda, M$} \Comment{$\beta, W^{(1)}$ can come from Adam}
  \For{$j = 1$ \textbf{to} $d$}
    \State {$\tilde\beta_j = \mathrm{Sign}(\beta) \cdot |\mathcal{S}_\lambda(\beta_{\text{EMA}, j})|$} \Comment{Sequential: Lasso first}
    \State
    \(\displaystyle
      \begin{aligned}[t]
        \text{Sort }|W_{\text{EMA}, j}^{(1)}|\text{ into }
        &\;|W_{\text{EMA}, (j,1)}^{(1)}|\ge\dots\ge|W_{\text{EMA},(j,K)}^{(1)}|
      \end{aligned}
    \)
    \For{$m = 0$ \textbf{to} $K$}
      \State
      \(\displaystyle
        \begin{aligned}[t]
          w_m
          &= \tfrac{M}{1 + m\,M^2}
             \,\mathcal{S}_\lambda\Bigl(
               |\tilde\beta_j|
               + M \sum_{i=1}^m |W_{\text{EMA},(j,i)}^{(1)}|
             \Bigr)
        \end{aligned}
      \)
    \EndFor
    \State
    \(\displaystyle
      \begin{aligned}[t]
        \text{Find first }\tilde m:\quad
        |W_{\text{EMA}, (j,\tilde m+1)}^{(1)}|\le w_{\tilde m} \le|W_{\text{EMA}, (j,\tilde m)}^{(1)}|
      \end{aligned}
    \)
    \State \(\tilde W_j^{(1)}
      = \mathrm{Sign}(W_j^{(1)})\,
        \min\bigl(w_{\tilde m},\,|W_{\text{EMA}, j}^{(1)}|\bigr)\)
  \EndFor
  \State \Return $(\tilde\beta,\tilde W^{(1)})$
\EndProcedure
\end{algorithmic}
\small\itshape
Conventions:
$W_{(j,K+1)}=0$, $W_{(j,0)}=+\infty$, min and $|\cdot|$ are elementwise.
\end{algorithm}

\paragraph{Joint Optimization Blinds Lasso Guidance.} LassoNet relies on linear skip connections to robustly guide the neural module away from irrelevant features. However, its joint optimization of $\V{\theta}$ and $W^{(1)}$ mathematically undermines this intent. HPG computes a shared scalar $w_{\tilde{m}}$ to update both parameters:
\[
\tilde{\theta}_j \!=\! \frac{1}{M} \cdot \text{sign}(\theta_j) \cdot w_{\tilde{m}}, \quad \tilde{W}_j^{(1)} \!=\! \text{sign}(W_j^{(1)}) \cdot \min(w_{\tilde{m}}, |W_j^{(1)}|)
\]
Because $w_{\tilde{m}}$ is dominated by the massive parameter count of $W^{(1)}$, the update trajectory biases heavily toward the neural module, effectively sidelining the skip connection. This does not contradict LassoNet's proximal global optimum lemma \citep{lemhadri2021lassonet}, as that lemma only guarantees optimality for the points returned by the optimizer \textit{along specific trajectories}, not the global learning objective itself.

\paragraph{Sequential Optimization: Lasso-First Guidance.} \emph{LassoFlexNet} structurally prioritizes the skip connection via a sequential optimization scheme. Let $\mathcal{S}_{\lambda \cdot \eta_t }(\cdot)$ denote the soft-thresholding operator \citep{parikh2014proximal}. Given the unconstrained update $\V{\beta}$, we first solve for the independent Lasso parameter $\tilde{\V{\beta}}$:
\begin{align}
\label{eqn:lasso_beta_prox}
    \tilde{\V{\beta}} 
    = 
    \mathcal{S}_{\lambda \cdot \eta_t }(\V{\beta}) 
    = 
    \arg\min_{\V{b}} \left[ \frac{ 1 }{2 \eta_t} \|\V{\beta} - \V{b}\|_2^2 + \lambda \| \V{b} \|_1 \right]
\end{align} 
Subsequently, we optimize $\tilde{W}$ under a constraint strictly guided by the newly computed $\tilde{\V{\beta}}$:
\begin{equation*}
    \tilde{W} 
    \!=\!
    \arg \min_{U} \left[ \frac{1}{2} \|U - W\|_2^2 + \bar{\lambda} \|U\|_1 \right] \quad \text{s.t.} \quad \|U\|_\infty \!\leq\! M \cdot \|\tilde{\V{\beta}}\|_2
\end{equation*}
This decoupling (Algorithm~\ref{alg:seq-hier-prox}) ensures $\tilde{\V{\beta}}$ dictates feature selection autonomously, while $W$ strictly obeys the resulting boundary.

\paragraph{Adaptive Integration: Proximal Adam.}
LassoNet relies on a constant-rate SGD with Nesterov momentum, which exhibits high variance and instability. We extend the proximal gradient framework to accommodate Adam's adaptive, per-coordinate learning rates $\eta_{i, t}$. The unconstrained update step becomes:
\begin{align*}
    \beta_{i, t} 
    = 
    \beta_{i, t-1} - \eta_{i, t} \cdot m_t
    =
    \beta_{i, t-1} - \eta \frac{m_t}{\sqrt{\nu_t} + \epsilon }
\end{align*}
Crucially, the proximal operation (Equation \ref{eqn:lasso_beta_prox}) remains separable \citep{boyd2004convex} under this diagonal preconditioning. The constraint enforcement thus decomposes into a per-coordinate soft-thresholding step:
\begin{align*}
    \tilde{\beta}_{i, t+1} 
    &= \arg\min_{b_i} \left[ \frac{(b_i - \beta_{i, t} )^2}{2\eta_{i, t}} + \lambda|b_i| \right] \\
    &= \operatorname{sign}(\beta_{i, t})\max\left(|\beta_{i, t}|-\lambda\eta_{i, t},\,0\right)
\end{align*}

\paragraph{Stable EMA-Enhanced Proximal Gradients.} Proximal methods were historically designed for full-batch deterministic optimization and often suffer under stochastic regimes. While Adam mitigates gradient variance, instability persists in sparse neural models due to constraint boundaries. We introduce exponential moving averages (EMA) applied directly to the parameters \emph{during} the sequential proximal evaluation—a novel stabilization technique for sparse networks. Unlike standard post-hoc parameter EMA \citep{izmailov2018averaging}, we embed it natively within the backward pass parameter update (Algorithm~\ref{alg:seq-hier-prox}), guaranteeing stable convergence when combined with the $\tau$-induced curriculum learning.

\section{Theoretical Analysis of Tabular Inductive Biases}
\label{app:address_3_questions}

As formalized in Sections~\ref{sec:intro} and \ref{sec:design_principle}, robust deep tabular models must satisfy five criteria: (1) resilience to irrelevant features, (2) preservation of data orientation, (3) capacity for localized irregularities, (4) processing of heterogeneous inputs, and (5) stable training dynamics. Here, we theoretically and empirically validate how \emph{LassoFlexNet} fulfills these requirements.

\textbf{(1) Resilience to Irrelevant Features.} \citet{ng2004feature} (Theorem 3.1) demonstrated that $\ell_1$ regularization robustly learns even when irrelevant features exponentially outnumber training samples. By applying a Tied Group Lasso over the nonlinear per-feature embeddings (Sec.~\ref{sec:skip_and_module}), \emph{LassoFlexNet} inherits this strict feature-selection capability, extending the principles of LassoNet \citep{lemhadri2021lassonet} into nonlinear representation spaces.

\textbf{(2) Preservation of Data Orientation.} Neural networks traditionally suffer from rotational invariance, whereas decision trees natively respect the original, meaningful coordinate axes—a critical factor in their tabular superiority. While Lasso regularization implicitly breaks this invariance \citep{ng2004feature}, premature dense feature interactions would destroy this property. We rigorously prove that the coordinate-wise encoding alone (PLE) actively breaks rotational invariance (detailed proof in App.~\ref{app:ple_breaks_invariance}):
\begin{proposition}[PLE Breaks Rotational Invariance (Informal)]
\label{prp:informal_ple_not_invariant}
If a learning algorithm $A$ is rotationally invariant, but an encoding scheme $\Phi$ is not rotationally equivariant, then the composite algorithm $A \circ \Phi$ loses rotational invariance. We provide an explicit counter-example in the infinite Neural Tangent Kernel (NTK) regime.
\end{proposition}
This is consistent with the empirical observations in \citep{gorishniy2022embeddings, gorishniy2025tabm}.

\textbf{(3) Capacity for Localized Irregularities.} Beyond breaking invariance, we mathematically establish that PLE yields expressivity benefits without inflating the parameter count, making it ideal for the highly irregular manifolds common in tabular data (proof in App.~\ref{app:ple_more_expressivie}):
\begin{proposition}[PLE Expressivity (Informal)] 
\label{prp:informal_ple_more_expressivie}
Allocating $T_i$ segments for feature $x_i$ via PLE yields a $T_i$-fold increase in functional expressivity without requiring additional learnable parameters, actively facilitating the modeling of localized irregularities.
\end{proposition}
Sec.\ \ref{sec:ablation} empirically validates {the benefits of PLE}.

\textbf{(4) Processing Heterogeneous Inputs.} Tabular variables often originate from distinct sensing modalities or abstraction levels. While Transformers and MLPs indiscriminately blend these inputs via early linear projections, \emph{LassoFlexNet} utilizes independent PFE modules. Because the downstream Mixer requires a homogeneous vector space, end-to-end backpropagation structurally forces the individual PFE blocks to map highly heterogeneous raw inputs into a unified, normalized latent manifold.

\textbf{(5) Stabilizing Training Dynamics.} While proximal gradient operators grant robustness to irrelevant features, they notoriously destabilize stochastic optimization. To formalize this, we analyze the training dynamics of LassoNet's joint Hier-Prox operator (proofs in App.~\ref{app:sec:hier-prox-converges}):
\begin{proposition}[Hier-Prox Hard-Constraint Convergence (Informal)]
\label{prp:hard-hier-prox-converges}
Under standard stochastic optimization regularity assumptions, the convergence of LassoNet's Hier-Prox is bounded by:
    \begin{align*}
        \mathbb E\big[F(y_t)-F(y^\star)\big]
        \;\le\;
        \rho^t \,\big[F(y_0)-F(y^\star)\big] 
        + C_{\text{Noise}}\eta^2\sigma_{\theta, W_1}^2
        + C_{\text{Sign Flip}}\,\mathbb P[\mathsf{Jump}_t], \quad (\rho < 1).
    \end{align*}
\end{proposition}
While the first two terms represent standard stochastic proximal convergence, the non-convex nature of the hierarchical constraint induces jump discontinuities, yielding the unique $C_{\text{Sign Flip}}\,\mathbb P[\mathsf{Jump}_t]$ error term. This penalty is particularly devastating for \emph{LassoFlexNet}: because $\V{\beta}$ operates on dynamically evolving PFE representations (rather than static inputs), the variance $\mathbb{V}\mathrm{ar} (\nabla \V{\beta})$ inflates, driving up the probability of sign flips ($\mathbb P[\mathsf{Jump}_t]$) and rupturing the training trajectory.

To eliminate this pathological jump discontinuity, we analyze a \textit{convex relaxation} of the hierarchical constraint (App.~\ref{app:sec:soft-hier-prox}):
\begin{proposition}[Hier-Prox Convex Relaxation Convergence (Informal)]
\label{prp:soft-hier-prox-converges}
Under identical conditions, the convex relaxation of the Hier-Prox operator yields:
    \begin{align*}
        \mathbb E\big[F(y_t)-F(y^\star)\big]
        \;\le\;
        \rho^t \,\big[F(y_0)-F(y^\star)\big] 
        + C_{\text{Noise}}\eta^2\sigma_{\theta, W_1}^2
    \end{align*}
\end{proposition}
Crucially, the relaxed formulation not only annihilates the jump-discontinuity error term, but the resulting algorithmic derivation \emph{mathematically decouples} the optimization of the skip coefficients $\V{\beta}$ from the neural weights $W$. 

This reveals a fundamental tradeoff: the hard constraint (Prop.~\ref{prp:hard-hier-prox-converges}) strictly enforces feature sparsity, while the convex relaxation (Prop.~\ref{prp:soft-hier-prox-converges}) guarantees stable convergence. Our proposed Seq-Hier-Prox-EMA (Algorithm~\ref{alg:seq-hier-prox}) resolves this tension. By retaining the hard $\ell_\infty$ constraint but adopting the \emph{sequential update strategy} dictated by the convex relaxation, we safely decouple the optimization dynamics.

\begin{remark}
While Propositions~\ref{prp:hard-hier-prox-converges} and \ref{prp:soft-hier-prox-converges} rely on stylized assumptions, they provide concrete theoretical intuition: coupled hierarchical proximal constraints fail under high variance due to sign-flip discontinuities. Our sequential, EMA-smoothed algorithm systematically neutralizes this failure mode, a theoretically motivated claim comprehensively validated by our ablation studies in Section~\ref{sec:ablation}.
\end{remark}

\subsection{Ablation Study over Architecture and Optimization Design Choices}
\label{sec:ablation}

We conduct an ablation study on the California Housing dataset to evaluate \emph{LassoFlexNet}’s components, reporting RMSE (lower is better) in Table~\ref{tab:ablation}. 

\begin{table}[ht]
\scriptsize
\centering
\caption{Ablation study on the California Housing dataset. Ablations done on Linear LassoNet and \textit{LassoFlexNet} regression. Each value is reported as mean $\pm$ std (min, max) over 10 runs with different seeds.}
\label{tab:ablation}
\setlength{\tabcolsep}{3pt}
\renewcommand{\arraystretch}{0.9}
\begin{tabular}{p{4cm}c}
\toprule
\textbf{Ablation} & \textbf{Test RMSE} \\
\midrule
\multicolumn{2}{c}{\textbf{(Linear) LassoNet}} \\
\midrule
Baseline Linear LassoNet & 0.5263 $\pm$ 0.0103 (0.5100, 0.5362) \\
AdamHGD & 0.5294 $\pm$ 0.0088 (0.5073, 0.5352) \\
Adam & 0.5291 $\pm$ 0.0082 (0.5090, 0.6044) \\
\midrule
\multicolumn{2}{c}{\textbf{LassoFlexNet}} \\
\midrule
Baseline LassoFlexNet & 0.4047 $\pm$ 0.0069 (0.3954, 0.4180) \\
$\tau = 1$ & 0.4004 $\pm$ 0.0095 (0.3812, 0.4136) \\
Adam Warmup + Nesterov & 0.4050 $\pm$ 0.0089 (0.3874, 0.4190) \\
Adam Optimizer with HGD & 0.4089 $\pm$ 0.0077 (0.3934, 0.4183) \\
Hybrid Adaptive Lambda & 0.4109 $\pm$ 0.0129 (0.3967, 0.4314) \\
Linear PFE Layers & 0.4145 $\pm$ 0.0068 (0.4064, 0.4241) \\
Learnable BatchNorm Layers & 0.4191 $\pm$ 0.0096 (0.4034, 0.4357) \\
No EMA & 0.4195 $\pm$ 0.0108 (0.4010, 0.4435) \\
$\tau = 1$, no skip (Mixer only) & 0.4221 $\pm$ 0.0117 (0.4027, 0.4325) \\
Original Proximal Grad. w/ no EMA & 0.4320 $\pm$ 0.0097 (0.4208, 0.4532) \\
Adam w/ no Proximal Gradient & 0.4272 $\pm$ 0.0098 (0.4146, 0.4516) \\
No PLE & 0.4387 $\pm$ 0.0124 (0.4086, 0.4538) \\
Linear LassoNet + PLE & 0.5226 $\pm$ 0.0098 (0.5050, 0.5343) \\
No PFE & 0.5371 $\pm$ 0.0166 (0.5144, 0.5770) \\
MLP 1 (same width/depth) w/ HGD & 0.4768 $\pm$ 0.0079 (0.4649, 0.4859) \\
MLP 2 (same \# params) w/ HGD & 0.4548 $\pm$ 0.0101 (0.4338, 0.4655) \\
MLP 1 (same width/depth) & 0.8453 $\pm$ 0.2016 (0.5815, 1.0330) \\
MLP 2 (same \# params) & 0.8406 $\pm$ 0.2009 (0.5815, 1.0329) \\
\bottomrule
\end{tabular}
\end{table}

The full model achieves an RMSE of $0.4047$ (on average). Removing EMA from Seq-Hier-Prox increases RMSE to $0.4195$, confirming EMA’s role in stabilizing training. Replacing the Mixer with an MLP backbone of similar depth (MLP 1) or similar numbers of parameters (MLP 2) yields RMSE of $0.8453$, and 0.8406, respectively. This instability is remedied substantially by the inclusion of HyperGrad, but is still not competitive with LassoFlex. Replacing Seq-Hier-Prox with LassoNet’s Hier-Prox harms performance (RMSE $0.4320$), as the joint optimization undermines the linear skip connection. Using Adam without proximal gradients yields $0.4272$, highlighting the necessity of proximal methods for feature selection. Setting $\tau=1$ (disabling curriculum learning) marginally improves RMSE to $0.4004$, although the variance in this regime is higher. Using learnable BatchNorm causes overfitting ($0.4191$). The Mixer alone (with HyperGrad) performs poorly ($0.4221$), and the original LassoNet (with Adam or HyperGrad) yields $0.5294$ and $0.5291$, respectively. Each component—EMA, Seq-Hier-Prox, fixed BatchNorm and the Mixer component contributes significantly to \emph{LassoFlexNet}’s performance.

\section{Comprehensive Benchmarking}
\label{sec:result}
We evaluate \emph{LassoFlexNet} across $52$ datasets spanning three rigorous benchmarks: a feature selection suite~\citep{feature_selection_benchmark}, TabZilla~\citep{performance_benchmark}, and TabRed~\citep{rubachev2025tabred}. Across these diverse regimes, \emph{LassoFlexNet} consistently outperforms state-of-the-art deep tabular models (e.g., FT-Transformer~\citep{gorishniy2021revisiting}, TabM~\citep{gorishniy2025tabm}) and gradient-boosted trees (e.g., XGBoost~\citep{chen2016xgboost}, CatBoost~\citep{prokhorenkova2018catboost}). Comprehensive experimental settings are detailed in Appendix~\ref{app:exp_settings}. Additional analyses regarding hyperparameter robustness and the generalization impact of Seq-Hier-Prox are provided in Appendices~\ref{app:hyperparam} and \ref{app:la_improvement}, respectively.


\subsection{Feature Selection Under Noise}
\label{sec:feature_selection_result}
We first assess the capability of \emph{LassoFlexNet} to isolate relevant signals amidst severe noise. We utilize the feature selection benchmark from \citet{feature_selection_benchmark}, which corrupts $12$ standard datasets by injecting either purely random noise features (Table~\ref{tab:fs_method_result}) or highly correlated second-order features (Table~\ref{tab:fs_method_result_c}) at extraneous rates of $50\%$ and $75\%$. 
\textbf{Baselines and Protocol.} We compare our end-to-end architecture against traditional sequential ``two-stage'' pipelines. In these baselines, a dedicated feature selection algorithm (e.g., Lasso, LassoNet, XGBoost) first filters the inputs before passing them to a downstream deep model (MLP or FT-Transformer). Crucially, we evaluate these baselines in an \textit{oracle setting}: they are explicitly provided with the ground-truth number of relevant features ($k$). In stark contrast, \emph{LassoFlexNet} operates blindly, autonomously discovering the sparse support set during training without any prior knowledge of $k$.
\textbf{Results.} Despite this severe informational disadvantage, \emph{LassoFlexNet} consistently outperforms the oracle two-stage pipelines. On the random noise benchmark, it achieves top-3 performance in $8/12$ datasets at the $50\%$ noise rate and $7/12$ at the $75\%$ rate. This confirms that selecting features based on learned, nonlinear representations via a Tied Group Lasso is more robust than linear pre-filtering.

\begin{table*}[t!]
\scriptsize
\renewcommand{\arraystretch}{0.65}
\centering
\begin{adjustbox}{max width=\textwidth}
\begin{tabular}{c l |c c c c c c c c c c c c}
\toprule
\% & FS Method & AL & CH & CO & EY & GE & HE & HI & HO & JA & MI & OT & YE \\
\midrule
\multirow{20}{*}{50} 
& No FS + MLP & 0.941 & -0.480 & 0.961 & 0.538 & 0.466 & 0.366 & 0.798 & -0.622 & 0.703 & -0.911 & 0.773 & -0.801 \\  
& Univariate + MLP & 0.960 & -0.447 & 0.970 & 0.575 & 0.515 & 0.379 & 0.811 & \textbf{-0.549} & 0.715 & \textbf{-0.891} & 0.808 & -0.776 \\  
& Lasso + MLP & 0.949 & -0.454 & 0.969 & 0.547 & 0.458 & 0.378 & 0.812 & -0.575 & 0.715 & -0.907 & 0.805 & \textbf{-0.774} \\  
& 1L Lasso + MLP & 0.952 & -0.453 & 0.969 & 0.564 & 0.473 & 0.380 & 0.811 & -0.568 & 0.715 & -0.897 & 0.796 & \textbf{-0.773} \\  
& AGL + MLP & 0.958 & -0.512 & 0.969 & 0.578 & 0.495 & 0.386 & 0.813 & -0.557 & 0.718 & -0.898 & 0.791 & -0.783 \\  
& LassoNet + MLP & 0.954 & -0.445 & 0.969 & 0.495 & 0.383 & 0.386 & 0.811 & -0.557 & 0.705 & -0.907 & 0.783 & -0.776 \\  
& AM + MLP & 0.953 & -0.444 & 0.968 & 0.554 & 0.498 & 0.382 & 0.813 & -0.566 & 0.722 & -0.904 & 0.801 & -0.777 \\  
& RF + MLP & 0.955 & -0.444 & 0.969 & 0.554 & 0.498 & 0.382 & 0.813 & -0.566 & 0.722 & -0.904 & 0.801 & -0.777 \\  
& XGBoost + MLP & 0.956 & -0.448 & 0.969 & 0.590 & 0.502 & 0.385 & 0.812 & -0.560 & 0.720 & -0.893 & 0.800 & -0.774 \\  
& Deep Lasso + MLP & 0.959 & -0.443 & 0.968 & 0.573 & 0.485 & 0.383 & 0.814 & \textbf{-0.549} & 0.720 & -0.894 & 0.802 & -0.776 \\ 
& No FS + FT & 0.959 & -0.432 & 0.966 & 0.673 & 0.500 & 0.384 & 0.817 & -0.577 & 0.730 & -0.902 & 0.813 & -0.792 \\  
& Univariate + FT & \textbf{0.963} & -0.424 & 0.970 & 0.700 & 0.519 & 0.389 & \textbf{0.819} & \textbf{-0.554} & 0.733 & \textbf{-0.897} & 0.819 & -0.788 \\  
& Lasso + FT & 0.952 & -0.419 & 0.969 & 0.703 & 0.504 & 0.392 & 0.817 & -0.594 & 0.728 & -0.999 & 0.820 & -0.998 \\  
& 1L Lasso + FT & \textbf{0.963} & -0.423 & 0.969 & \textbf{0.722} & 0.489 & 0.389 & \textbf{0.819} & -0.577 & 0.732 & -0.904 & 0.819 & \textbf{-0.775} \\  
& AGL + FT & 0.899 & \textbf{-0.419} & 0.969 & 0.701 & 0.480 & 0.386 & \textbf{0.822} & -0.586 & 0.733 & -0.915 & 0.814 & -0.586 \\  
& LassoNet + FT & \textbf{0.963} & -0.423 & 0.970 & 0.690 & 0.481 & 0.392 & 0.818 & -0.559 & 0.733 & -0.904 & 0.808 & -0.789 \\  
& AM + FT & 0.962 & -0.432 & 0.969 & 0.703 & 0.504 & 0.392 & 0.817 & -0.568 & \textbf{0.735} & -0.903 & \textbf{0.820} & -0.789 \\  
& RF + FT & 0.962 & -0.432 & 0.970 & 0.715 & \textbf{0.591} & \textbf{0.395} & \textbf{0.820} & -0.558 & \textbf{0.737} & -0.900 & \textbf{0.820} & -0.791 \\  
& XGBoost + FT & \textbf{0.963} & -0.420 & 0.970 & \textbf{0.725} & \textbf{0.572} & \textbf{0.400} & \textbf{0.819} & -0.568 & \textbf{0.734} & -0.898 & \textbf{0.820} & -0.789 \\  
& Deep Lasso + FT & 0.962 & \textbf{-0.419} & 0.969 & 0.703 & 0.504 & 0.392 & 0.817 & -0.560 & 0.733 & -0.900 & 0.817 & -0.788 \\ 
\midrule
& \textbf{Ours: LassoFlexNet} & 0.962 & \textbf{-0.409} & 0.969 & \textbf{0.720} & \textbf{0.609} & \textbf{0.396} & \textbf{0.819} & -0.567 & 0.730 & \textbf{-0.897} & 0.815 & \textbf{-0.775} \\ 
\midrule
\multirow{10}{*}{75} 
& No FS + MLP & 0.925 & -0.527 & 0.955 & 0.502 & 0.417 & 0.348 & 0.778 & -0.674 & 0.671 & -0.917 & 0.749 & -0.812 \\  
& Univariate + MLP & \textbf{0.960} & -0.447 & 0.970 & 0.575 & 0.502 & 0.381 & 0.810 & \textbf{-0.549} & 0.713 & \textbf{-0.890} & \textbf{0.806} & -0.776 \\  
& Lasso + MLP & \textbf{0.959} & -0.454 & 0.967 & 0.543 & 0.491 & 0.381 & \textbf{0.811} & -0.612 & 0.716 & -0.907 & 0.802 & -0.789 \\  
& 1L Lasso + MLP & 0.957 & -0.448 & 0.968 & 0.555 & 0.432 & 0.380 & 0.809 & -0.572 & 0.717 & -0.903 & 0.799 & \textbf{-0.775} \\  
& AGL + MLP & 0.954 & \textbf{-0.447} & 0.968 & 0.561 & 0.429 & 0.386 & 0.809 & \textbf{-0.571} & 0.719 & -0.901 & 0.762 & -0.777 \\  
& LassoNet + MLP & \textbf{0.958} & -0.452 & 0.966 & 0.528 & 0.475 & 0.383 & 0.809 & \textbf{-0.555} & 0.705 & -0.913 & 0.768 & -0.794 \\  
& RF + MLP & 0.949 & -0.453 & 0.968 & \textbf{0.584} & \textbf{0.610} & 0.386 & \textbf{0.814} & -0.585 & 0.718 & -0.902 & \textbf{0.808} & -0.784 \\  
& XGBoost + MLP & \textbf{0.958} & -0.451 & 0.969 & \textbf{0.576} & \textbf{0.583} & 0.382 & 0.810 & -0.568 & \textbf{0.720} & \textbf{-0.892} & 0.804 & \textbf{-0.774} \\  
& Deep Lasso + MLP & 0.957 & \textbf{-0.446} & 0.969 & 0.569 & 0.479 & 0.387 & \textbf{0.814} & -0.559 & \textbf{0.721} & \textbf{-0.893} & 0.800 & \textbf{-0.774} \\ 
\midrule
& \textbf{Ours: LassoFlexNet} & \textbf{0.958} & \textbf{-0.411} & 0.968 & \textbf{0.665} & \textbf{0.602} & \textbf{0.392} & 0.810 & -0.572 & \textbf{0.721} & -0.895 & \textbf{0.805} & -0.776 \\ 
\bottomrule
\end{tabular}
\end{adjustbox}
\caption{Feature selection benchmark on datasets with random extra features. Reports accuracy (classification) and negative RMSE (regression). Bold indicates top-3 performance per dataset at 50\% and 75\%. “No FS” indicates training without FS, as done for \emph{LassoFlexNet}, which performs end-to-end feature selection.} 
\label{tab:fs_method_result}

\scriptsize
\renewcommand{\arraystretch}{0.7}
\centering
\begin{adjustbox}{max width=\textwidth}
\begin{tabular}{c l |c c c c c c c c c c c c }
\toprule
\% & FS method & AL & CH & CO & EY & GE & HE & HI & HO & JA & MI & OT & YE  \\ \midrule
\multirow{20}{*}{50} 
& No FS + MLP      & 0.946 & -0.475 & 0.965 & 0.557 & 0.525 & 0.370 & 0.802 & -0.607 & 0.703 & -0.909 & 0.778 & -0.797 \\
& Univariate + MLP & 0.955 & -0.451 & 0.966 & 0.556 & 0.514 & 0.346 & 0.810 & -0.620 & 0.717 & -0.920 & 0.795 & -0.828 \\
& Lasso + MLP      & 0.955 & -0.449 & 0.968 & 0.548 & 0.512 & 0.382 & 0.813 & -0.602 & 0.713 & -0.903 & 0.796 & -0.795 \\
& 1L Lasso + MLP   & 0.955 & -0.447 & 0.968 & 0.566 & 0.515 & 0.382 & 0.812 & -0.581 & 0.718 & -0.902 & 0.795 & -0.780 \\
& AGL + MLP        & 0.953 & -0.450 & 0.968 & 0.588 & 0.538 & 0.386 & 0.813 & -0.561 & 0.722 & -0.902 & 0.796 & -0.780 \\
& LassoNet + MLP   & 0.955 & -0.452 & 0.969 & 0.570 & 0.556 & 0.382 & 0.811 & -0.551 & 0.719 & -0.905 & 0.795 & \textbf{-0.777} \\
& AM + MLP         & 0.955 & -0.449 & 0.967 & 0.583 & 0.527 & 0.381 & 0.814 & -0.555 & 0.722 & -0.905 & 0.797 & -0.780 \\
& RF + MLP         & 0.951 & -0.453 & 0.967 & 0.574 & 0.568 & 0.383 & 0.810 & -0.565 & 0.724 & -0.904 & 0.788 & -0.786 \\
& XGBoost + MLP    & 0.954 & -0.454 & 0.969 & 0.583 & 0.510 & 0.385 & 0.815 & -0.553 & 0.722 & -\textbf{0.892} & 0.803 & -0.779 \\
& Deep Lasso + MLP & 0.955 & -0.447 & 0.968 & 0.577 & 0.525 & 0.388 & 0.815 & -0.567 & 0.721 & -\textbf{0.895} & 0.801 & -\textbf{0.776}\\ 
\\
& No FS + FT        & 0.960 & -0.430 & 0.967 & 0.686 & 0.576 & 0.345 & 0.812 & -0.574 & 0.733 & -0.920 & 0.809 & -0.826  \\
& Univariate + FT   & 0.963 & -0.422 & 0.965 & 0.681 & 0.574 & 0.387 & 0.820 & -0.586 & 0.732 & -0.937 & 0.812 & -0.915  \\
& Lasso + FT        & 0.952 & -0.422 & 0.936 & 0.697 & 0.556 & 0.389 & 0.820 & -0.570 & 0.731 & -0.899 & \textbf{0.816} & -0.795  \\
& 1L Lasso + FT     & 0.962 & -\textbf{0.419} & 0.969 & \textbf{0.718} & 0.571 & 0.392 & 0.820 & -0.552 & 0.735 & -0.914 & \textbf{0.816} & -0.830  \\
& AGL + FT          & 0.906 & -0.426 & 0.969 & 0.697 & 0.591 & \textbf{0.393} & 0.820 & -0.572 & \textbf{0.736} & -0.903 & 0.813 & -0.790  \\
& LassoNet + FT     & 0.962 & -0.426 & \textbf{0.970} & 0.679 & 0.578 & \textbf{0.392} & 0.820 & -\textbf{0.549} & 0.734 & -0.901 & \textbf{0.817} & -0.790  \\
& AM + FT           & 0.962 & -0.424 & 0.969 & 0.680 & 0.572 & \textbf{0.392} & 0.820 & -0.557 & 0.735 & -0.898 & 0.806 & -0.793  \\
& RF + FT           & 0.962 & -0.422 & 0.969 & 0.711 & \textbf{0.600} & 0.387 & 0.819 & -\textbf{0.548} & 0.735 & -\textbf{0.897} & \textbf{0.816} & -0.790  \\
& XGBoost + FT      & 0.963 & -0.422 & \textbf{0.970} & 0.706 & 0.564 & \textbf{0.392} & \textbf{0.821} & -\textbf{0.561} & \textbf{0.736} & -0.898 & 0.809 & -0.788  \\
& Deep Lasso + FT   & 0.961 & -0.422 & 0.968 & \textbf{0.725} & 0.577 & \textbf{0.393} & \textbf{0.821} & -0.598 & \textbf{0.736} & -0.898 & 0.801 & -\textbf{0.776}  \\
\midrule
& \textbf{Ours: LassoFlexNet} & 0.960 & \textbf{-0.415} & \textbf{0.970} & \textbf{0.712} & \textbf{0.599} & 0.390 & 0.817 & -0.590 & 0.731 & -0.903 & 0.809 & -0.780 \\ 
\midrule
& No FS + MLP        & 0.921 & -0.516 & 0.956 & 0.518 & 0.503 & 0.356 & 0.788 & -0.632 & 0.686 & -0.913 & 0.762 & -0.808 \\
& Univariate + MLP   & 0.955 & -0.569 & 0.941 & 0.510 & 0.495 & 0.347 & 0.742 & -0.620 & 0.686 & -0.921 & 0.779 & -0.838 \\
& Lasso + MLP        & 0.948 & -0.454 & 0.963 & \textbf{0.565} & 0.490 & 0.373 & 0.810 & -0.593 & 0.717 & -0.903 & 0.795 & -0.791 \\
& 1L Lasso + MLP     & \textbf{0.955} & -\textbf{0.444} & 0.967 & 0.549 & 0.495 & 0.380 & \textbf{0.811} & -0.576 & 0.715 & -\textbf{0.903} & 0.797 & -\textbf{0.779} \\
& AGL + MLP          & 0.928 & -0.566 & 0.967 & 0.548 & 0.490 & \textbf{0.382} & \textbf{0.811} & -0.574 & 0.714 & -0.904 & 0.788 & -\textbf{0.780} \\
& LassoNet + MLP     & 0.947 & -0.452 & \textbf{0.969} & 0.539 & \textbf{0.533} & \textbf{0.383} & 0.805 & -\textbf{0.572} & 0.708 & -0.908 & 0.791 & -0.785 \\
& RF + MLP           & 0.952 & -0.450 & 0.963 & 0.547 & \textbf{0.533} & 0.372 & 0.805 & -0.573 & 0.716 & -\textbf{0.903} & 0.765 & -0.788 \\
& XGBoost + MLP      & 0.954 & -0.515 & \textbf{0.968} & \textbf{0.571} & 0.530 & 0.381 & \textbf{0.811} & -\textbf{0.571} & \textbf{0.721} & -\textbf{0.895} & \textbf{0.800} & -0.784 \\
& Deep Lasso + MLP   & \textbf{0.959} & -\textbf{0.441} & \textbf{0.968} & 0.554 & 0.517 & \textbf{0.386} & \textbf{0.813} & -\textbf{0.563} & \textbf{0.718} & -\textbf{0.898} & \textbf{0.804} & -\textbf{0.778} \\
\midrule
& \textbf{Ours: LassoFlexNet} & \textbf{0.955} & \textbf{-0.422} & 0.963 & \textbf{0.675} & \textbf{0.579} & \textbf{0.382} & \textbf{0.811} & -0.588 & \textbf{0.718} & -0.905 & \textbf{0.801} & -0.785 \\ 
\bottomrule
\end{tabular}
\end{adjustbox}
\caption{Benchmarking feature selection methods for MLP and FT-Transformer downstream models on datasets \textbf{with corrupted features}. We report performance of models trained on features selected by different FS algorithms in terms of accuracy for classification and negative RMSE for regression problems. \% refers to percent of extra features in the dataset: either 50\% or 75\% features are second-order. Bold font indicates the top-3 performance and lower rank indicates better overall result, and the comparisons are done for a given dataset at a noise level (50 or 75).}
\label{tab:fs_method_result_c}
\end{table*}

\subsection{TabZilla Performance Benchmarks}
\label{sec:performance_result}

\begin{figure*}[t!]
    \centering
    \includegraphics[width=.6\linewidth]{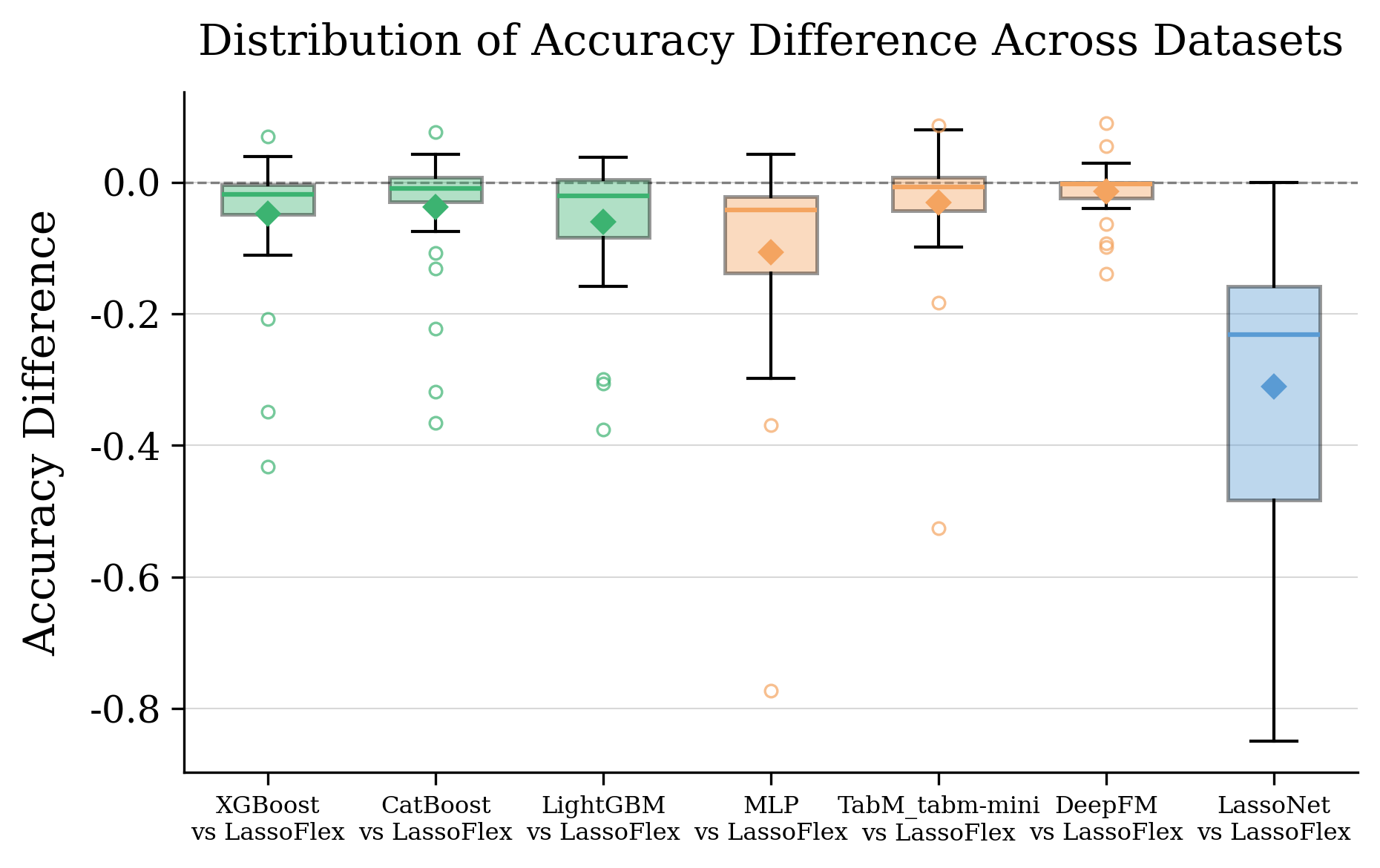}
    \caption{Relative performance distribution of 6 benchmark methods compared to \textit{LassoFlexNet} across 35 TabZilla datasets. The $y$-axis represents the difference in accuracy (Benchmark v.s.
    \textit{LassoFlexNet}). Thus, values below zero indicate that \textit{LassoFlexNet} outperforms the competing method. The boxplots summarize the variance across different datasets.}
    \label{fig:performance_bench_result}
\end{figure*}

\begin{table*}[t!]
\centering
\small
\caption{TabRed Benchmark Results}
\label{tab:tabred}
\begin{tabularx}{\textwidth}{l *{6}{>{\raggedright\arraybackslash}X}}
\toprule
Dataset \textbackslash Model
& CatBoost
& LightGBM
& RF
& XGBoost
& TabM-mini
& LassoFlexNet \\
\midrule
ecom-offers (class. AUC)
& 0.6507 ± 0.0554 [6]
& 0.6317 ± 0.0538 [6]
& 0.6447 ± 0.0529 [6]
& 0.6447 ± 0.0628 [6]
& \textbf{0.6533 ± 0.0603 [6]}
& 0.6322 ± 0.0419 [6] \\
homecredit-default (class. AUC)
& 0.8211 ± 0.0115 [6]
& 0.8092 ± 0.0110 [6]
& 0.7365 ± 0.0168 [6]
& 0.8046 ± 0.0136 [6]
& \textbf{0.8413 ± 0.0102 [6]}
& 0.8365 ± 0.0096 [6] \\
homesite-insurance (class. AUC)
& 0.9559 ± 0.0088 [6]
& 0.9576 ± 0.0075 [6]
& 0.8414 ± 0.0168 [6]
& 0.9568 ± 0.0088 [6]
& 0.9557 ± 0.0083 [6]
& \textbf{0.9605 ± 0.0071 [6]} \\
cooking-time (reg. RMSE)
& 0.4718 ± 0.0084 [6]
& 0.4899 ± 0.0095 [6]
& 0.5192 ± 0.0062 [6]
& 0.4866 ± 0.0108 [6]
& 0.5629 ± 0.0157 [6]
& \textbf{0.4658 ± 0.0093 [6]} \\
delivery-eta (reg. RMSE)
& \textbf{0.5370 ± 0.0094 [6]}
& 0.5641 ± 0.0095 [6]
& 0.5701 ± 0.0079 [6]
& 0.5425 ± 0.0086 [6]
& 0.6121 ± 0.0057 [5]
& 0.5375 ± 0.0104 [6] \\
maps-routing (reg. RMSE)
& 0.1665 ± 0.0015 [6]
& 0.1641 ± 0.0015 [6]
& 0.1746 ± 0.0012 [6]
& 0.1709 ± 0.0018 [6]
& 0.1656 ± 0.0021 [6]
& \textbf{0.1619 ± 0.0013 [6]} \\
sberbank-housing (reg. RMSE)
& 0.2585 ± 0.0171 [6]
& 0.2568 ± 0.0171 [6]
& 0.2908 ± 0.0141 [6]
& 0.2761 ± 0.0254 [6]
& 0.2769 ± 0.0098 [5]\footnotemark
& \textbf{0.2544 ± 0.0138 [6]} \\
weather (reg. RMSE)
& 1.5922 ± 0.0565 [6]
& 1.6272 ± 0.0735 [6]
& 2.4514 ± 0.0994 [6]
& 1.5632 ± 0.0766 [6]
& 2.2434 ± 0.1835 [6]
& \textbf{1.5591 ± 0.0764 [6]} \\
\bottomrule
\end{tabularx}
\end{table*}
\footnotetext{5 seeds because one seed leads to unstable training.}

\textbf{Benchmark and Baselines.} To evaluate robustness across highly diverse data distributions, we utilize the TabZilla benchmark suite~\citep{performance_benchmark}. From its collection of the 36 hardest classification tasks, we select the subset of 35 datasets containing numerical or mixed features. We compare \emph{LassoFlexNet} against six formidable competitors, representing both the gradient-boosted decision tree (GBDT) family (XGBoost~\citep{chen2016xgboost}, CatBoost~\citep{prokhorenkova2018catboost}, LightGBM~\citep{ke2017lightgbm}) and state-of-the-art deep tabular models (DeepFM~\citep{guo2017deepfm} and TabM~\citep{gorishniy2025tabm}).
\textbf{Protocol and Results.} We adopt a rigorous two-stage evaluation protocol. For each dataset, we first perform a hyperparameter search on a single fold, selecting the configuration that minimizes validation log-loss on a hold-out split. Using these fixed hyperparameters, we then report the final performance across a separate 5-fold cross-validation. Figure~\ref{fig:performance_bench_result} illustrates the distribution of accuracy differentials between the baselines and our method. Because the $y$-axis represents $(\text{Benchmark Accuracy} - \text{LassoFlexNet Accuracy})$, values below zero indicate that \emph{LassoFlexNet} outperformed the competitor. The results demonstrate consistent superiority: \emph{LassoFlexNet} outperforms all six baselines in the majority of tasks, exhibiting a robust downward shift in the error distributions.

\subsection{TabRed: Robustness to Distribution Shift}

\textbf{Benchmark and Protocol.} Standard tabular benchmarks often rely on i.i.d.\ shuffling, which masks the temporal distribution shifts prevalent in real-world applications. To assess robustness under these challenging conditions, we evaluate \emph{LassoFlexNet} on the TabRed benchmark~\citep{rubachev2025tabred}. Adopting the TabZilla protocol, we aggregate performance across six splits per dataset: three randomized splits and three strict temporal splits (where the training data strictly precedes the test data in time). This rigorously tests the model's ability to generalize to future, shifted distributions. 
\textbf{Results.} Table~\ref{tab:tabred} details the performance of \emph{LassoFlexNet} against leading GBDT methods and the state-of-the-art neural model TabM-mini. \emph{LassoFlexNet} establishes a new state-of-the-art, achieving the best performance in 5 out of 8 tasks. In the remaining tasks, it remains highly competitive; for instance, it trails the top performer (CatBoost) by a negligible margin of 0.0005 RMSE in \emph{delivery-eta}, and ranks second in \emph{homecredit-default}. This consistent performance across both randomized and temporal splits highlights \emph{LassoFlexNet}'s superior capacity to handle real-world tabular heterogeneity and non-stationarity.

\section{Conclusion}
\label{sec:conclusion}

We introduced \emph{LassoFlexNet}, a novel neural architecture accompanied by a novel optimization algorithm that sets the new state-of-the-art for tabular deep learning. 
The method is designed specifically to address five challenges: (i) resilience to irrelevant features; (ii) respect for data orientations; (iii) the ability to learn complex and irregular functions; (iv) the capability to model a larger number of heterogeneous inputs; and (v) stability and robustness in training. LassoFlexNet bridges the gap between tree-based and neural methods, offering a practical solution for tabular data. 

\clearpage

\bibliographystyle{unsrtnat}
\bibliography{references}

\newpage
\appendix
\section{Appendix: Background and More Experiments}
\subsection{LassoNet Training Algorithm}
\label{appendix:lassonet_prox}
\begin{algorithm}[H]
\caption{LassoNet Hierarchical Proximal Operator~\citep{lemhadri2021lassonet}}
\label{alg:lassonet_prox}
\begin{algorithmic}[1]
\Procedure{Hier-Prox}{$\theta, W^{(1)}; \lambda, M$}
  \For{$j = 1$ \textbf{to} $d$}
    \State
    \(\displaystyle
      \begin{aligned}[t]
        \text{Sort }|W_j^{(1)}|\text{ into }
        &\;|W_{(j,1)}^{(1)}|\ge\dots\ge|W_{(j,K)}^{(1)}|
      \end{aligned}
    \)
    \For{$m = 0$ \textbf{to} $K$}
      \State
      \(\displaystyle
        \begin{aligned}[t]
          w_m
          &= \tfrac{M}{1 + m\,M^2}
             \,\mathcal{S}_\lambda\Bigl(
               |\theta_j|
               + M \sum_{i=1}^m |W_{(j,i)}^{(1)}|
             \Bigr)
        \end{aligned}
      \)
    \EndFor
    \State
    \(\displaystyle
      \begin{aligned}[t]
        \text{Find first }\tilde m:\quad
        |W_{(j,\tilde m+1)}^{(1)}|\le w_{\tilde m}
        \le|W_{(j,\tilde m)}^{(1)}|
      \end{aligned}
    \)
    \State $\tilde\theta_j = \tfrac1M\,\mathrm{sign}(\theta_j)\,w_{\tilde m}$
    \State \(\tilde W_j^{(1)}
      = \mathrm{sign}(W_j^{(1)})\,
        \min\bigl(w_{\tilde m},\,|W_j^{(1)}|\bigr)\)
  \EndFor
  \State \Return $(\tilde\theta,\tilde W^{(1)})$
\EndProcedure
\end{algorithmic}
\small\itshape
Notation: $d$ is the number of features; $K$ is first-layer width.
\end{algorithm}

\subsection{Lambda-Training Improvements Over Pre-Trained Models}
\label{app:la_improvement}

In this section, we evaluate the additional generalization gains provided by the $\lambda$-training stage (utilizing our proposed Seq-Hier-Prox-EMA optimizer) compared to the baseline pre-trained models. For each dataset (GE, EY, and CH), we randomly initialized 300 LassoFlexNet models with diverse hyperparameter configurations. We first pre-trained each model for 200 epochs, followed by a $\lambda$-training phase of 100 epochs at each $\lambda$ value (both utilizing early stopping). We then compared the validation metrics of the final $\lambda$-trained models against their pre-trained counterparts. The results are visualized in Figure~\ref{fig:lambda_vs_pretrain}.
\begin{figure}[ht]
    \centering
    \begin{subfigure}{0.48\linewidth}
        \includegraphics[width=\linewidth]{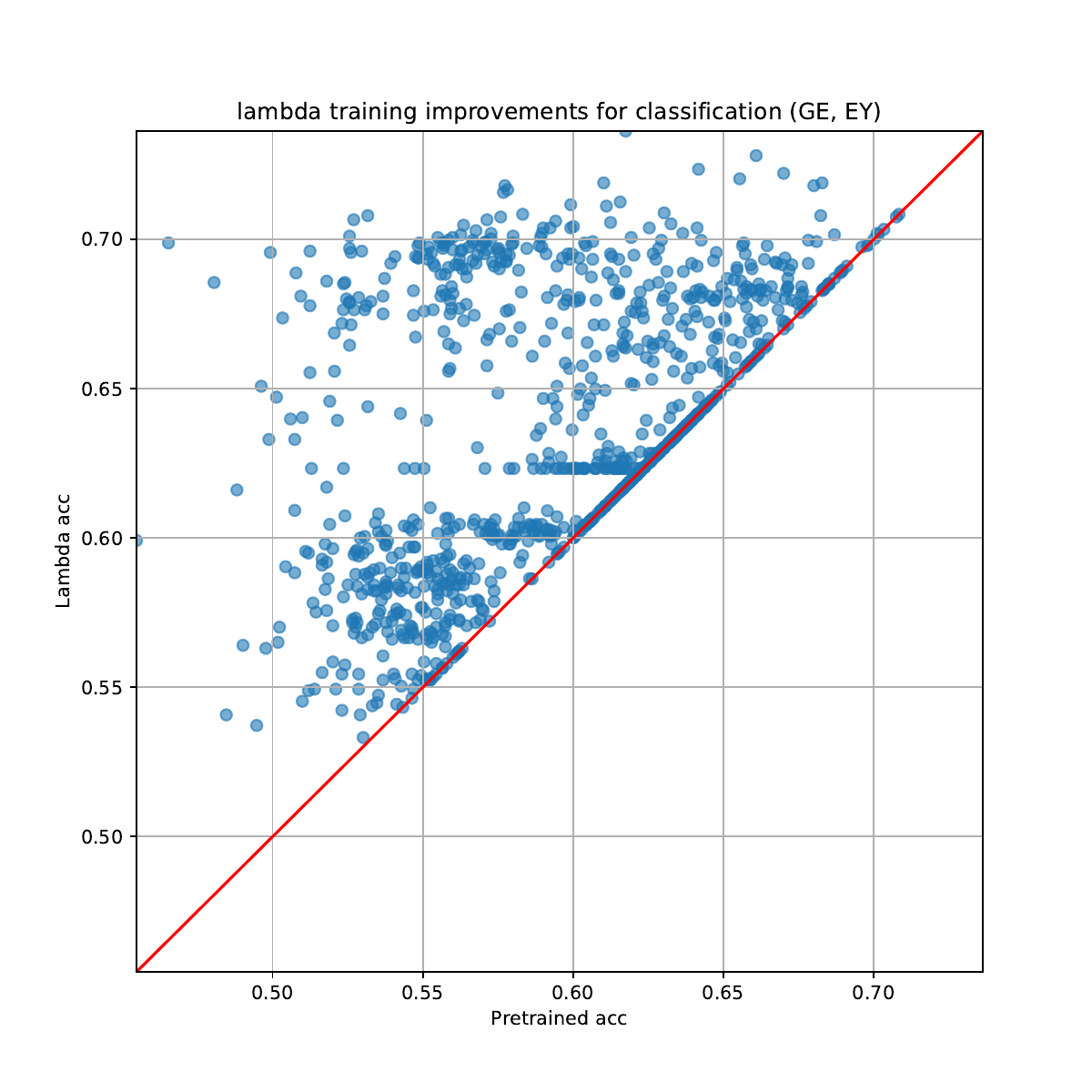}
        \label{fig:lambda_vs_pretrain_class}
    \end{subfigure}
    \hfill
    \begin{subfigure}{0.48\linewidth}
        \includegraphics[width=\linewidth]{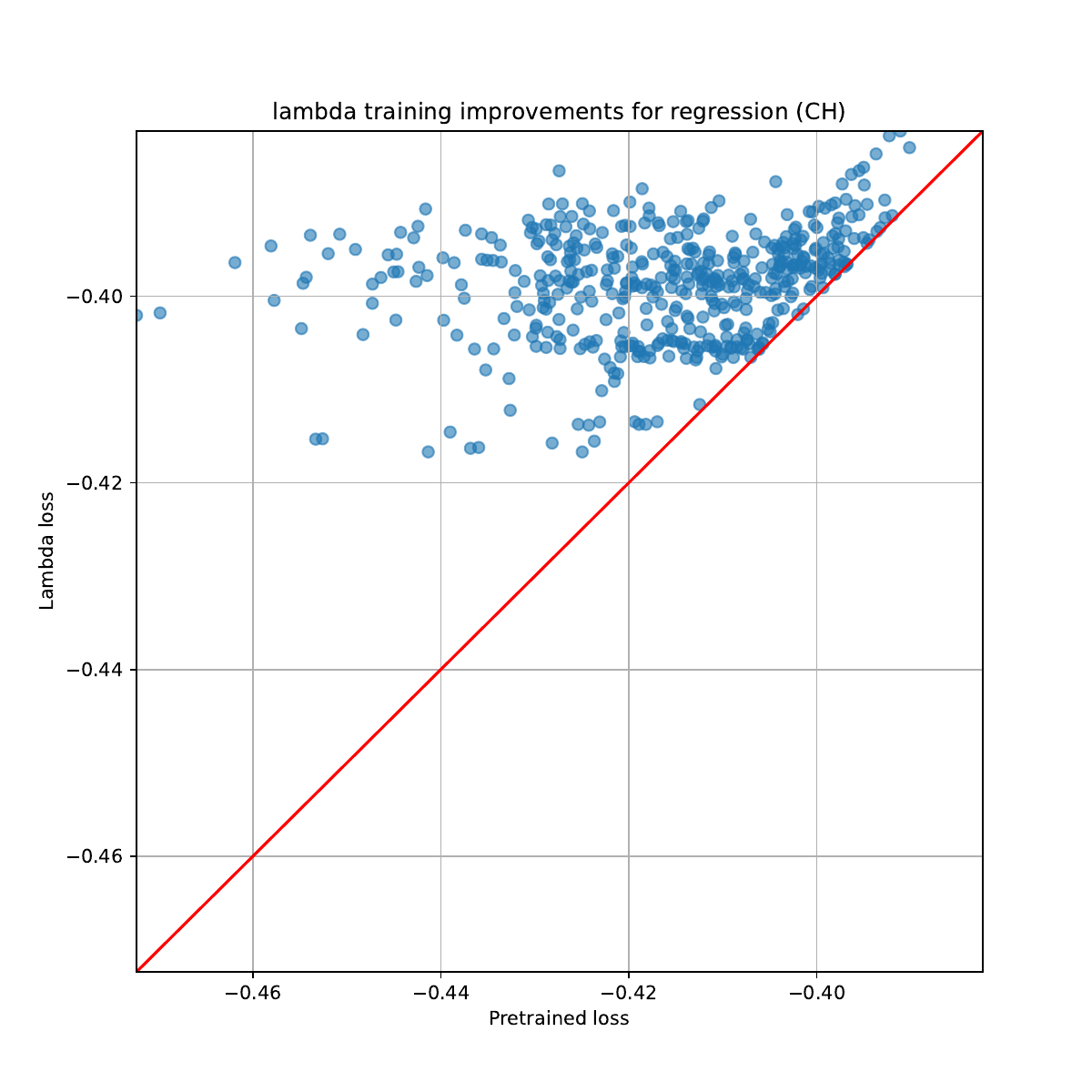}
        \label{fig:lambda_vs_pretrain_reg_CH}
    \end{subfigure}
    \caption{Validation performance comparison between 300 lambda-trained models with its corresponding pre-trained model: GE, EY classification tasks (left, higher acc means better, so \textbf{left of y=x means lambda improves}); CH regression (right, lower the loss magnitude means better, so \textbf{left of y=x means lambda improves}) }
    \label{fig:lambda_vs_pretrain}
\end{figure}

Our empirical results across both classification and regression tasks demonstrate that the vast majority of data points lie above the $y=x$ line (or to the left, depending on the metric orientation), indicating a consistent, strict improvement in validation metrics over the pre-trained baseline. Points situated directly on the $y=x$ identity line correspond to specific hyperparameter configurations where $\lambda$-training yielded no marginal gain; in these instances, the early-stopping mechanism simply reverted the final model to the pre-trained weights.

This consistently positive delta demonstrates the efficacy of our $\lambda$-training phase and the stability of the proposed Seq-Hier-Prox-EMA algorithm. Furthermore, because $\lambda$-training relies entirely on the learned skip weights $\V{\beta}$ accurately capturing feature importance to constrain the neural module $W^{(1)}$, this experiment implicitly confirms that $\V{\beta}$ is successfully learned during the pre-training stage, validating the end-to-end LassoFlexNet pipeline.

Figure~\ref{fig:lasso_improves_gen} provides a detailed trajectory of LassoFlexNet training on the California Housing dataset. Training is governed by early stopping based on validation loss. The pre-training phase halts at the epoch indicated by the red vertical line, after which the $\lambda$-training phase activates and continues to drive the validation loss down, even as the training loss naturally rises due to the increasing $\ell_1$ penalty from the $\lambda$-training (Seq-Hier-Prox-EMA). The EMA validation loss comes from the EMA of the model weights ensemble. It further lowers the validation loss by smoothing out the noise from the model weights ensemble, similar to EMA in diffusion models~\citep{ho2020denoising}.

\begin{figure}[H]
    \centering
    \includegraphics[width=0.5\linewidth]{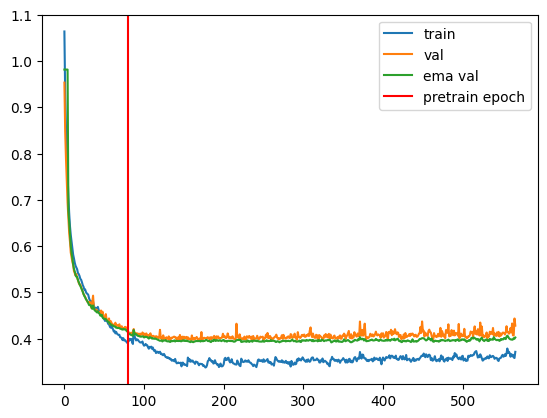}
    \caption{LassoFlexNet model improves generalization gap. Validation continues to go down, as training loss goes up.}
    \label{fig:lasso_improves_gen}
\end{figure}

\subsection{LassoFlexNet Hyperparameter Sensitivity Analysis}
\label{app:hyperparam}

To evaluate the sensitivity of LassoFlexNet to hyperparameter configurations, we conducted extensive experiments across 39 distinct feature selection benchmarks (each representing a unique combination of dataset and noise settings). We restricted this sensitivity analysis to classification tasks to facilitate standardized result aggregation, as accuracy is strictly bounded within $[0, 1]$. For each task, we sampled 100 hyperparameter configurations following the distributions detailed in Table~\ref{tab:hypertuned} (Appendix~\ref{app:exp_settings}). 

\begin{figure}[ht]
    \centering
    \includegraphics[width=\linewidth]{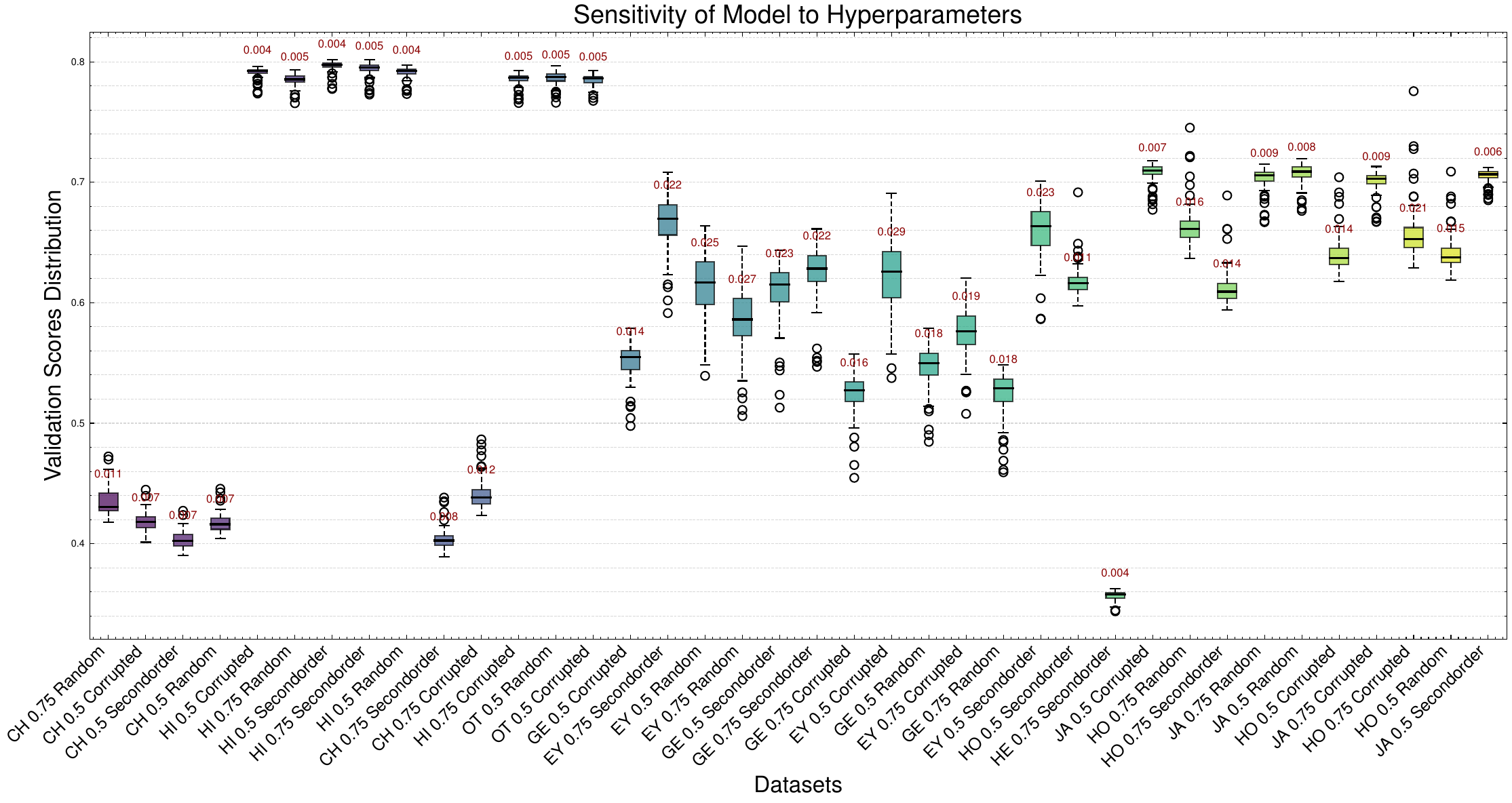}
    \caption{Sensitivity to hyperparameters of LassoFlexNet evaluated on 39 classification tasks} 
    \label{fig:sensitivity}
\end{figure}

Figure~\ref{fig:sensitivity} presents the resulting distribution of validation scores via box plots, with red annotations denoting one standard deviation of the validation accuracy for each task. We observe that LassoFlexNet's performance is highly robust to hyperparameter variations, evidenced by the narrow variance distributions across the vast majority of the benchmarks.

\subsection{Targeted Experiment: The Failure of Joint Optimization}
\label{app:target_exp}
To explicitly test the failure modes of the original LassoNet architecture, we construct an artificial dataset governed by the following generating function:
\begin{equation}
y =  \V{\beta}_{\ast}^\top \V{x} + \alpha \cdot \mathrm{NN}_{R}(\V{x}) + \gamma \cdot \epsilon
\end{equation}
where $\V{\beta}_{\ast}$ represents the ground-truth linear correlations. Only a sparse subset of elements in $\V{\beta}_{\ast}$ are large, while the remainder are exactly or approximately zero. $\mathrm{NN}_{R}$ is a fixed, randomly initialized neural network used to inject non-linearity, and $\epsilon \sim \mathcal{N}(0, 1)$ is standard Gaussian noise. The scalars $\alpha$ and $\gamma \in [0, 1]$ control the strength of the non-linearity and noise, respectively. Notably, by setting $\alpha = 0$, we isolate whether LassoNet can successfully uncover pure linear relations, decoupling the model's performance from its core theoretical assumption (that linear correlations can adequately constrain nonlinear relations).

The success of LassoNet hinges entirely on the pre-training stage, where the linear skip weights $\V{\beta}$ must converge toward the true sparse support $\V{\beta}_{\ast}$. Only with this accurate foundation can the subsequent $\lambda$-training phase effectively utilize $\V{\beta}$ to enforce feature sparsity on the non-linear components.

We trained the original LassoNet model on this artificial dataset with strong linear signals ($\alpha=0$) and visualized the evolution of the training loss, validation loss, and skip weights over time. The results, averaged over ten random seeds (varying the initialization of $\V{\beta}$ and the neural weights, but keeping $\V{\beta}_{\ast}$ fixed), are shown in Figure~\ref{fig:targeted_exp2}.
\begin{figure}[ht]
    \centering
    \begin{subfigure}[t]{0.4\textwidth} 
        \includegraphics[width=\textwidth, height=4cm]{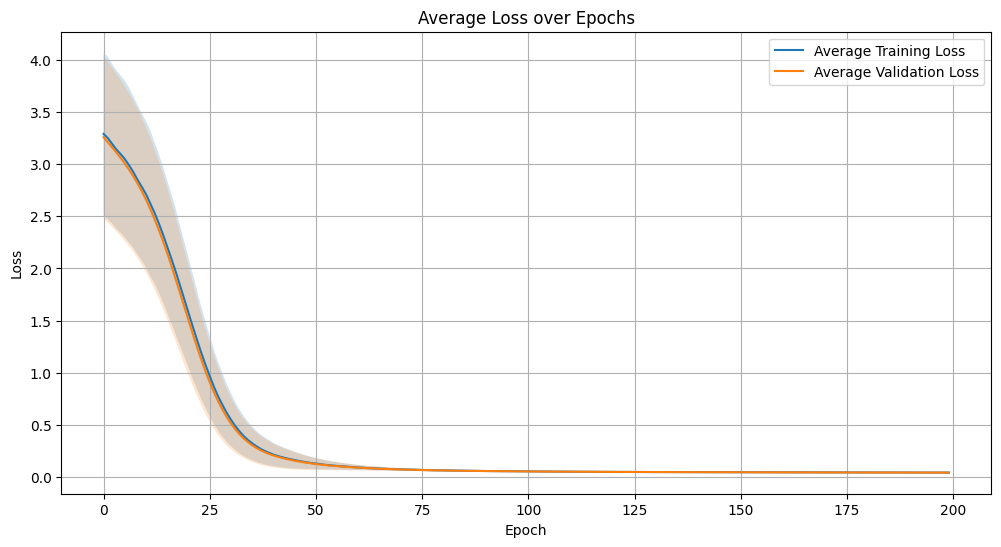} 
        \label{fig:loss_tau1}
    \end{subfigure}
    \begin{subfigure}[t]{0.44\textwidth} 
        \includegraphics[width=\textwidth, height=4cm]{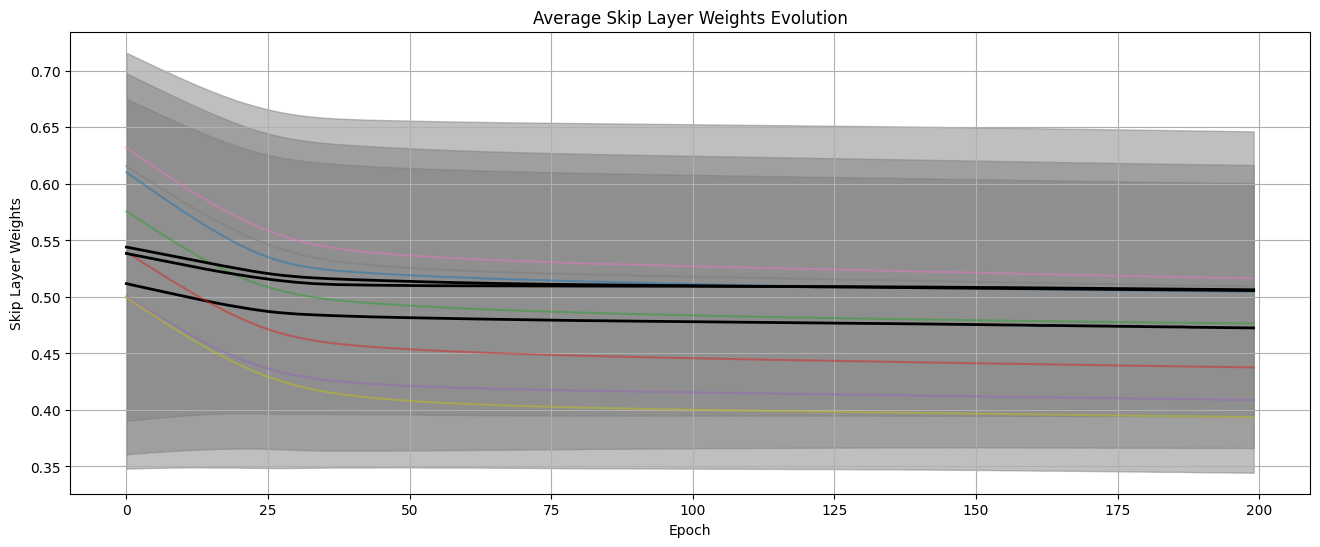} 
        \label{fig:skip_tau1}
    \end{subfigure}
    \caption{Average loss and skip weights evolution for original LassoNet during training across 10 runs}
    \label{fig:targeted_exp2}
\end{figure} \\
As observed in the skip weight evolution plot (Figure~\ref{fig:targeted_exp2}, right), the linear LassoNet fails to isolate the strong ground-truth linear weights; the trajectories of the important features (black lines) remain hopelessly entangled with the non-important features (colored lines). 

This training dynamic fatally undermines the premise of using skip weights to constrain feature participation. If the skip weights $\V{\beta}$ fail to separate important from unimportant features during pre-training, the subsequent $\lambda$-training phase will constrain the hidden layers using corrupted guidance. It is highly unlikely that proximal $\lambda$-training can reverse this adverse initialization, rendering the intended end-to-end feature selection ineffective.

Our intuition for this failure is as follows: the linear component of the model, $\V{\beta}^\top \V{x}$, is a significantly weaker learner than the nonlinear component $\mathrm{NN}_{W}(\V{x})$. During joint training, the optimization trajectory is dominated by $\mathrm{NN}_{W}(\V{x})$, preventing $\V{\beta}^\top \V{x}$ from independently uncovering the ground-truth linear signal. In over-parameterized regimes, the global loss can be minimized even if the skip weights fail to prioritize the correct features. Because there is no structural curriculum enforcing "learn the linear part first, the nonlinear part next", $\V{\beta}$'s ideal gradient share is dispersed into the neural module.

We validate this intuition by modifying the original LassoNet architecture slightly. Instead of $\hat{y} = \V{\beta}^\top \V{x} + \mathrm{MLP}(\V{x})$, we introduce a small scalar $\tau$ such that $\hat{y} = \V{\beta}^\top \V{x} + \tau \cdot \mathrm{MLP}(\V{x})$. This modification acts as a soft curriculum on the joint learning dynamics, suppressing the MLP's gradients and encouraging the linear weights to dominate early training. We plot the identical experiment setting $\tau = 0.001$ in Figure~\ref{fig:targeted_exp}.
\begin{figure}[ht]
    \centering
    \begin{subfigure}[t]{0.4\textwidth} 
        \includegraphics[width=\textwidth, height=4cm]{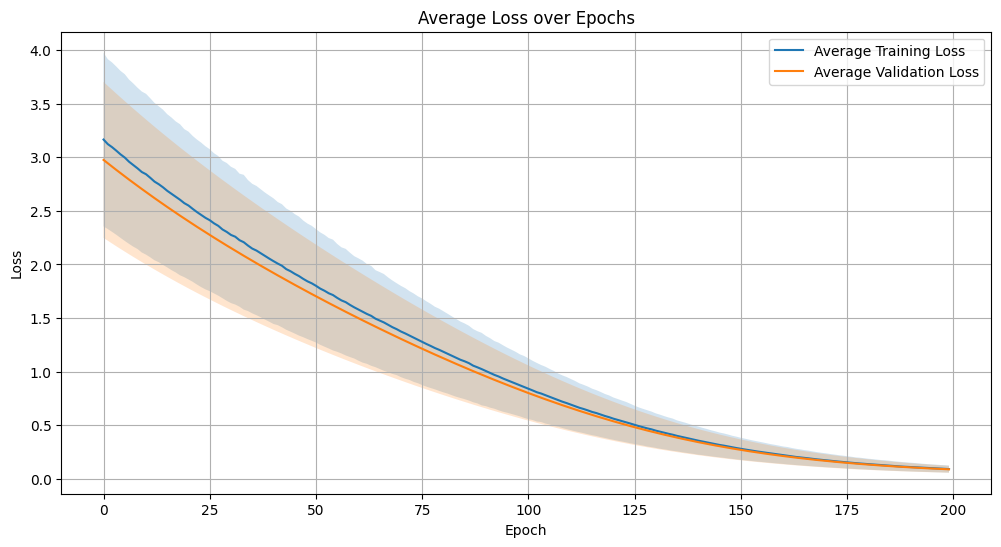} 
        \label{fig:loss_tau001}
    \end{subfigure}
    \begin{subfigure}[t]{0.44\textwidth} 
        \includegraphics[width=\textwidth, height=4cm]{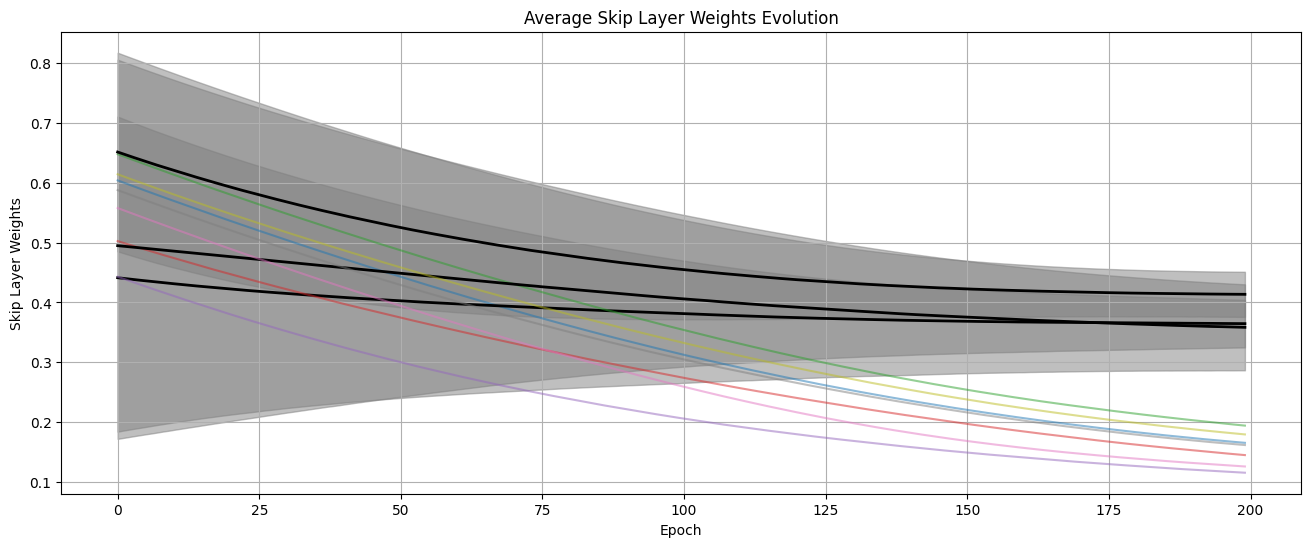} 
        \label{fig:skip_tau001}
    \end{subfigure}
    \caption{Average loss and skip weights evolution for original LassoNet with small scalar on the MLP part during training across 10 runs}
    \label{fig:targeted_exp}
\end{figure}
Under this soft curriculum, the learned skip weights for the important features cleanly separate from the unimportant ones. Note that the train and validation losses decrease at a slower overall pace with $\tau=0.001$. This convergence rate difference further corroborates our hypothesis that the unconstrained MLP was previously dominating the optimization dynamics at the expense of the linear skip connection.

\section{Appendix: LassoFlexNet Main Module Architecture}
\label{app:main_arch}

\paragraph{MLP-Mixer Modification: Enforcing the Hierarchical Bottleneck.}
The main neural module of LassoFlexNet utilizes the MLP-Mixer architecture proposed by \citet{tolstikhin2021mlp}. The original MLP-Mixer consists of per-patch linear embeddings, Mixer blocks, and a classifier head. Each Mixer block contains one feature-mixing MLP and one embedding-mixing MLP, each consisting of two fully-connected layers and a GELU nonlinearity, alongside residual connections, dropout, and LayerNorm. 

In LassoFlexNet, our Per-Feature Embedding (PFE) module connects directly to the Mixer blocks. Crucially, our Tied Group Lasso skip module is designed to constrain the feature-mixing MLP within the \emph{first} Mixer block. To ensure the skip module strictly gates all information flow—analogous to the first hidden layer in the original LassoNet—we designate the feature-mixing MLP of the first Mixer block as the sole computational bottleneck. Consequently, we explicitly omit the LayerNorm and residual connections in this initial block to prevent any information bypass that would circumvent the hierarchical constraint. This architectural modification is illustrated in Figure~\ref{fig:main_module}.

\begin{figure}[ht]
    \centering
    \includegraphics[width=0.8\linewidth]{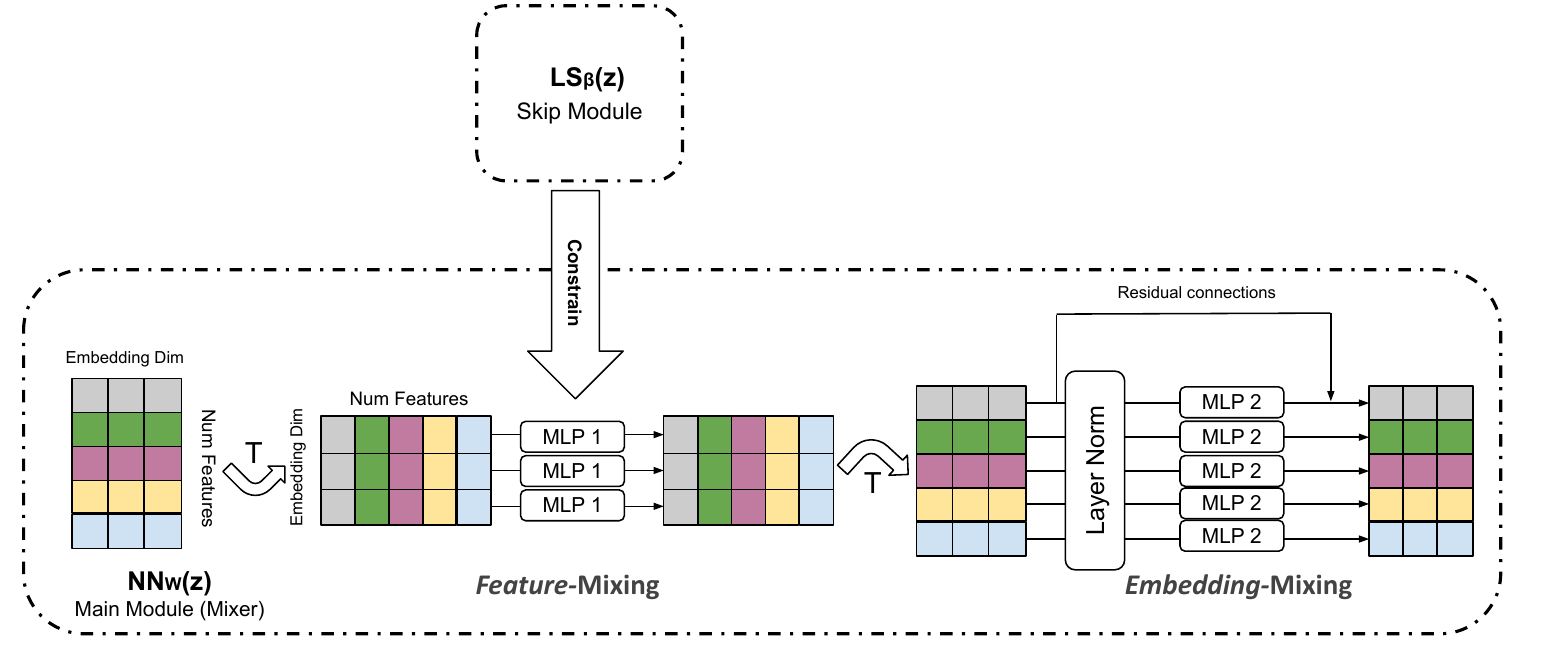}
    \caption{Details of the main module for LassoFlexNet, Mixer architecture is used}
    \label{fig:main_module}
\end{figure}

\subsection{Proximal Gradient Weight Evolution Comparison}
\label{app:prox_grad}

In this section, we empirically compare the training dynamics of LassoNet's original Hier-Prox operator against our proposed Seq-Hier-Prox operator. We initialize a pre-trained LassoFlexNet model and iteratively apply either Hier-Prox or Seq-Hier-Prox to update the skip weights $\V{\beta}$ and the first-layer neural weights $W^{(1)}$. 

The model is pre-trained on the California Housing dataset augmented with 50\% random noise \citep{feature_selection_benchmark}. We set the curriculum scalar $\tau=0.1$ to ensure $\V{\beta}$ is robustly learned during pre-training. For the iterative proximal updates, we use a fixed linear sequence of $\lambda$ values ranging from $10^{-5}$ to $10^{-2}$. To ensure robust observations, we sample different values of the constraint multiplier $M$ and inject Gaussian noise into the loaded pre-trained weights across multiple random seeds. We record the evolution paths for $\V{\beta}$ and $W^{(1)}$ for each seed and report the average trajectory.

The loaded pre-trained skip weights ($\V{\beta}$) and first hidden layer weights ($W^{(1)}$) are visualized in Figure~\ref{fig:ex1_pre-trained_weights}.

\begin{figure}[ht]
    \centering
    \begin{subfigure}[t]{0.55\textwidth}
        \centering
        \includegraphics[width=\textwidth]{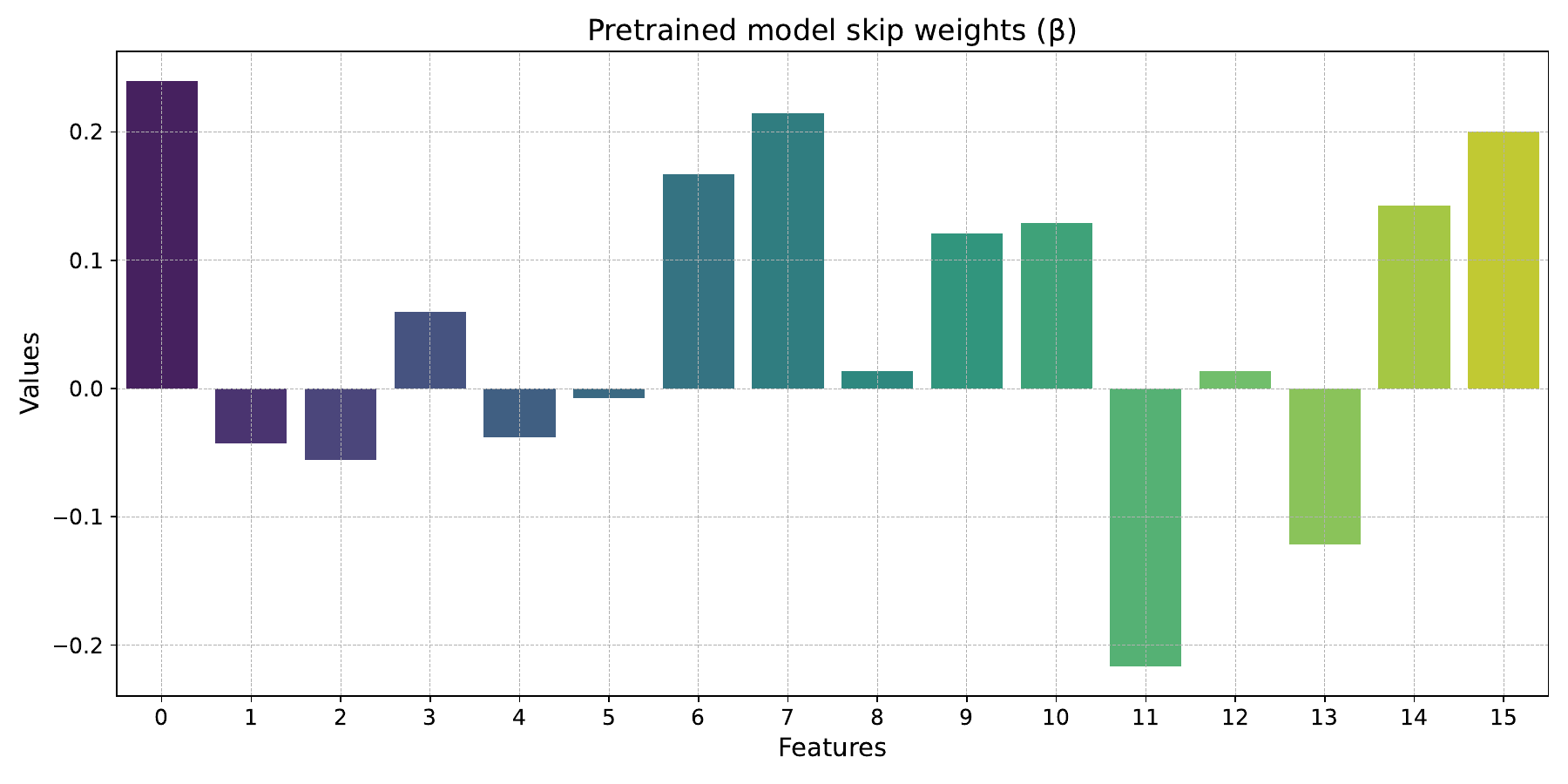}
        \caption{Skip (Feature Importance)}
        \label{fig:skip}
    \end{subfigure}
    \hfill
    \begin{subfigure}[t]{0.35\textwidth}
        \centering
        \includegraphics[width=\textwidth]{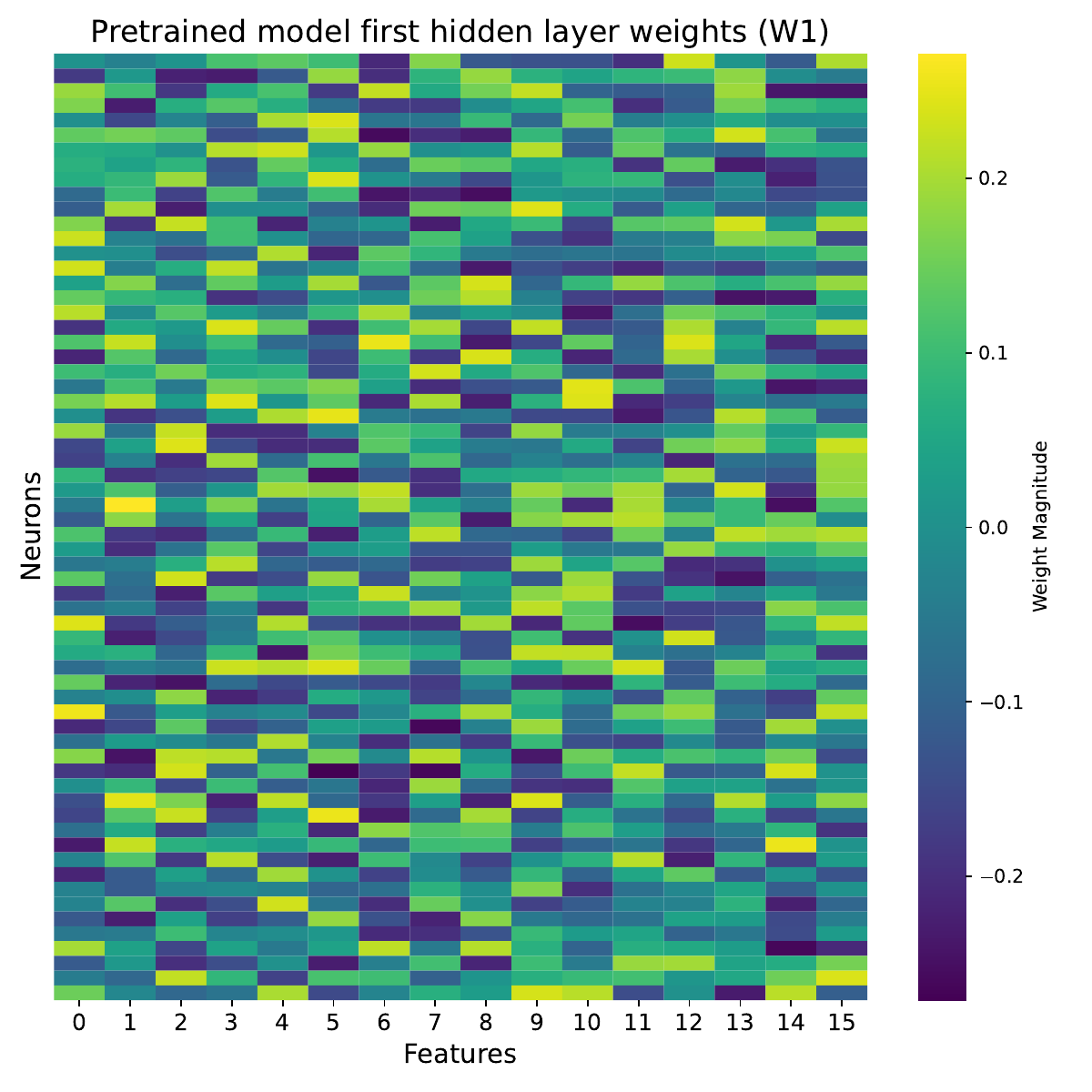}
        \caption{W1 (First Fully Connected Layer Weights)}
        \label{fig:w1}
    \end{subfigure}
    \caption{Visualizations of pre-trained LassoFlexNet model skip weights and W1 weights}
    \label{fig:ex1_pre-trained_weights}
\end{figure}

As shown, the pre-trained skip weights successfully capture feature importance, exhibiting distinct magnitudes (some close to zero, others large) that separate the true features from the injected noise. Conversely, the distribution of the hidden weights $W^{(1)}$ remains relatively uniform. The objective of the subsequent $\lambda$-training is to utilize these informative skip weights to constrain $W^{(1)}$, ensuring that unimportant features do not participate in the neural network's computations.

\textbf{Failure of Original Hier-Prox.} The average evolution of $\V{\beta}$ and $W^{(1)}$ under the original Hier-Prox operator is visualized in Figure~\ref{fig:original_proxgrad_evolve}. Before any proximal updates (iteration 0), the skip weights are cleanly separated. However, the moment the joint proximal gradient update begins, the skip weights $\V{\beta}$ become heavily influenced by the dense, noisier first hidden layer weights $W^{(1)}$. Consequently, the ability of $\V{\beta}$ to guide $W^{(1)}$ toward sparsity is severely weakened, and the critical feature importance learned during the pre-training stage is rapidly washed away.

\begin{figure}[ht]
    \centering
    \begin{subfigure}[t]{0.48\textwidth}
        \centering
        \includegraphics[width=\linewidth]{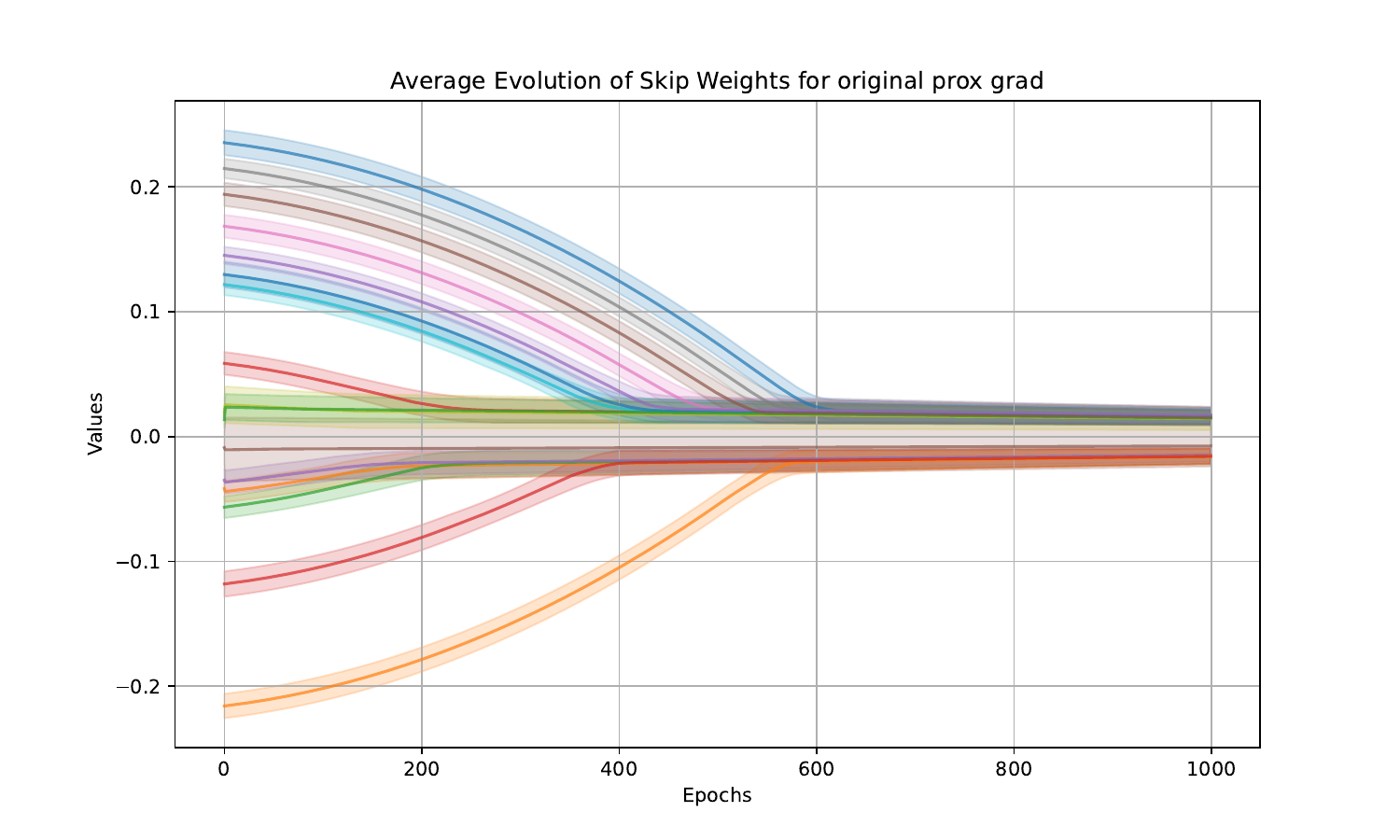}
        \label{fig:plot1}
    \end{subfigure}
    \hfill
    \begin{subfigure}[t]{0.48\textwidth}
        \centering
        \includegraphics[width=\linewidth]{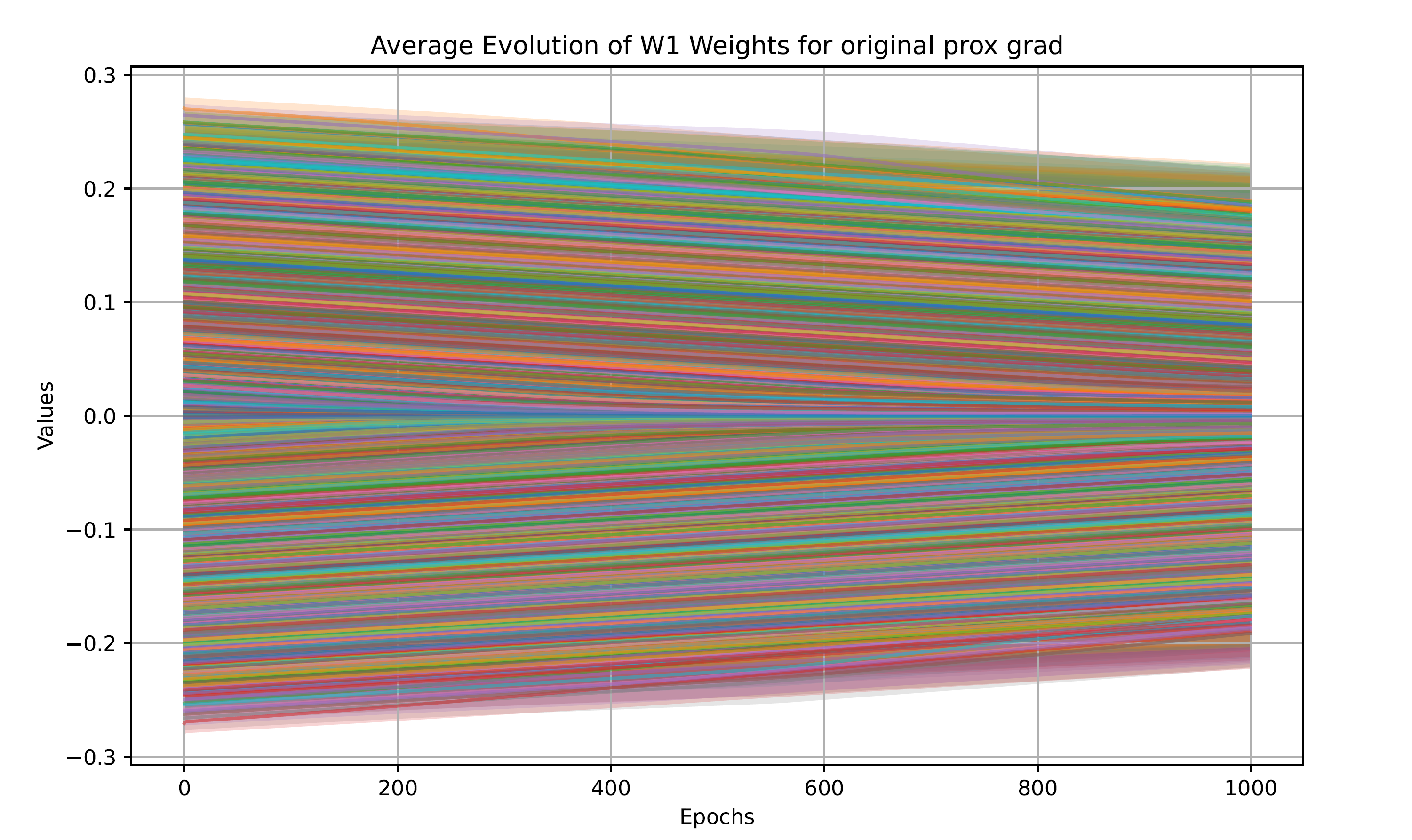}
        \label{fig:plot2a}
    \end{subfigure}
    \caption{Average evolution of skip weights (left) and W1 hidden layer weights (right) over original Hier-Prox updates}
    \label{fig:original_proxgrad_evolve}
\end{figure}

\begin{figure}[ht]
    \centering
    \begin{subfigure}[t]{0.48\textwidth}
        \centering
        \includegraphics[width=\linewidth]{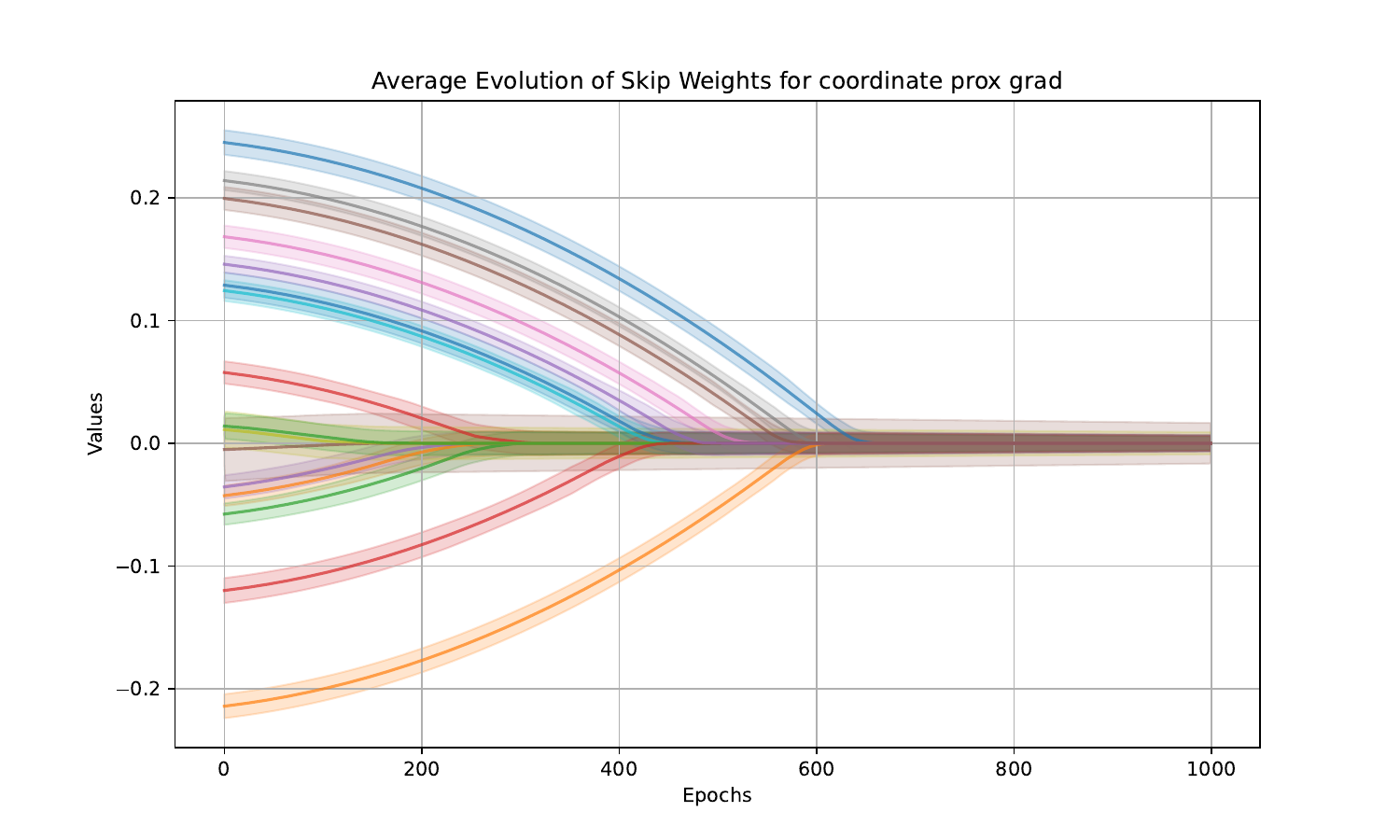}
        \label{fig:plot2b}
    \end{subfigure}
    \hfill
    \begin{subfigure}[t]{0.48\textwidth}
        \centering
        \includegraphics[width=\linewidth]{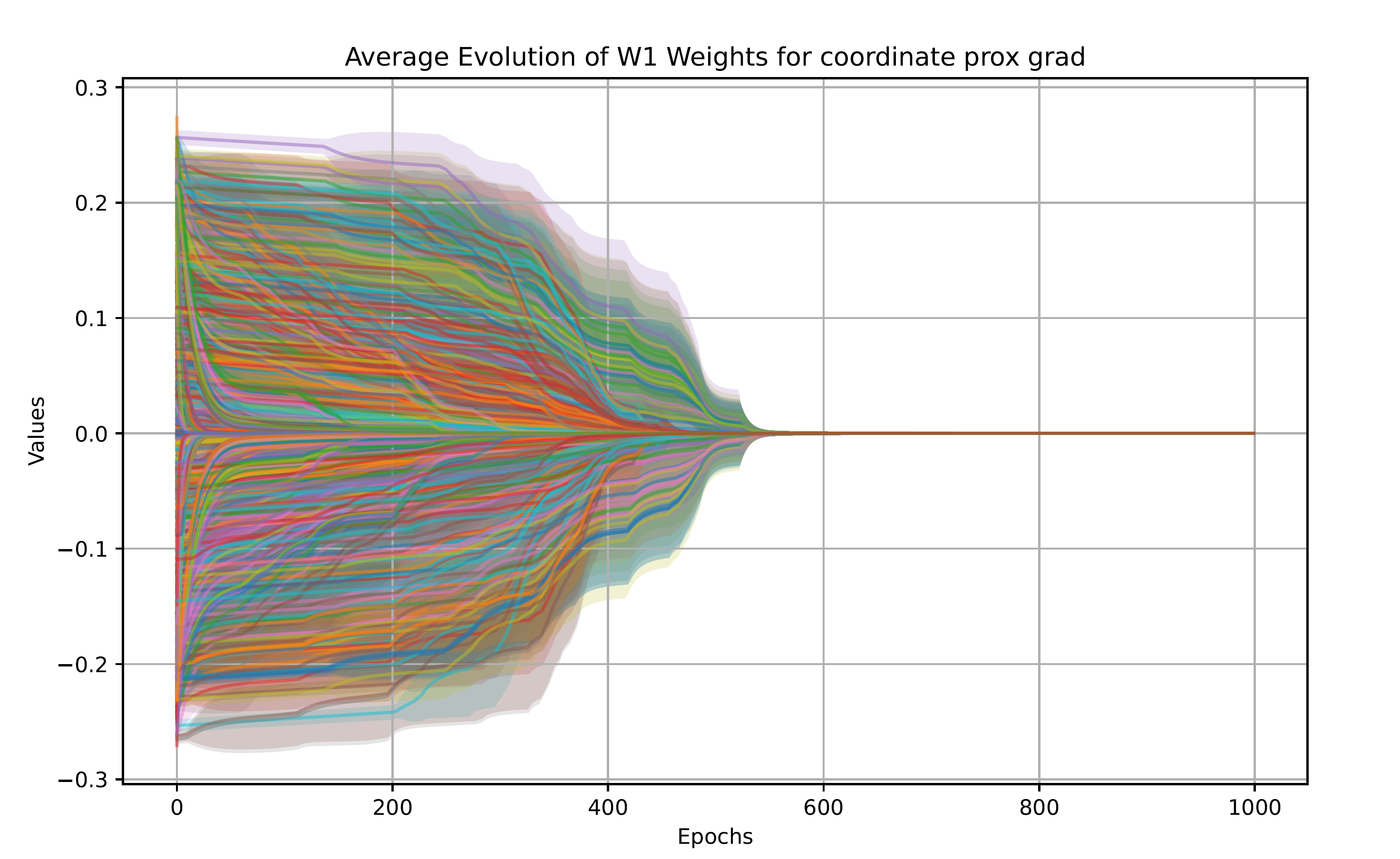}
        \label{fig:plot2c}
    \end{subfigure}
    \caption{Average evolution of skip weights (left) and W1 hidden layer weights (right) over Seq-Hier-Prox updates}
    \label{fig:coordinate_descent_proxgrad}
\end{figure}

\textbf{Success of Seq-Hier-Prox.} We next examine the evolution under LassoFlexNet's Seq-Hier-Prox operator (Figure~\ref{fig:coordinate_descent_proxgrad}). Because the sequential update mathematically decouples the parameters, the skip weights evolve smoothly without any initial jump discontinuities. The $\lambda$-training phase effectively leverages the feature importance learned during pre-training. As a result, the first hidden layer weights $W^{(1)}$ become cleanly sparse over time, successfully gating the neural network computations to ignore irrelevant features. Based on these dynamics, we utilize Seq-Hier-Prox-EMA for all experiments in Section~\ref{sec:result}.

\subsection{Experimental Settings for Section \ref{sec:result}}
\label{app:exp_settings}

In this section, we detail the experimental configurations used in Section~\ref{sec:result}. For LassoFlexNet, the pre-training stage is set to a maximum of 200 epochs, and the $\lambda$-training stage runs for a maximum of 100 epochs per $\lambda$ value. We employ an early-stopping patience of 20 epochs for both stages. During $\lambda$-training, if the validation loss does not improve for 20 epochs at the current $\lambda$, the training automatically advances to the next $\lambda$ value in the sequence.

For \textbf{feature selection benchmark} result in Section \ref{sec:feature_selection_result},  we directly take the result in \cite{feature_selection_benchmark} for all benchmark methods, and we use exact same setting for getting our LassoFlexNet result: We use the default train, val and test split used in \cite{feature_selection_benchmark}, which is [0.65, 0.15, 0.2]; we conduct hyperparameter tuning for LassoFlexNet for 100 iterations using Bayesian optimization engine Optuna \cite{akiba2019optuna}, and we use the Bayesian sampler's suggested hyperparameter from iterations to iteration. The objective for this hyperparameters tuning is again the model's validation metrics. The results presented in Table \ref{tab:fs_method_result} are model's test metrics with best hyperparameters, calculated across 10 random model initializations with different seeds. For FT-transformers and MLP, the hyperparameters and their respective ranges remain consistent with those reported in \cite{feature_selection_benchmark}. The specific hyperparameters for LassoFlexNet are detailed in Table \ref{tab:hypertuned}. 

For \textbf{TabZilla and TabRed benchmark} results in Section \ref{sec:performance_result}, we again use Bayesian optimization engine Optuna \cite{akiba2019optuna} with ASHA (Asynchronous Successive Halving Algorithm). For each algorithm for each dataset, we tune the model only using validation metrics from 1 fold split, and uses this split to tune objective validation loss. For both classification tasks and regression tasks, we use sklearn's log loss as the validation and test metrics for classification and MSE loss for regression. The box plot results in Figure \ref{fig:performance_bench_result} for each algorithm are best test metrics over 40 hyperparameters tuning trials, averaged across another 5 unseen random splits. For all benchmark methods, we use the same implementations as in \cite{performance_benchmark}.

\begin{table}[ht]
\centering
\begin{tabular}{@{}lcc@{}}
\toprule
\textbf{Hyperparameter} & \textbf{Range} & \textbf{Step/Scale} \\
\midrule
Mini batch sizes & 512--2048 & \text{scaled by} 2 \\
Number of Mixer blocks & 2--5 & 1 \\
MLP1\&2 hidden dimension & 64--512 & \text{scaled by} 2 \\
Feature embedding output dimension & 8--32 & 2 \\
Number of bins used in PLE & 4--32 & 2 \\
Tau (scalar multiplied with $NN_W$) & 0.01--2.0 & Uniform \\
ResNet dropout & 0.05--0.25 & Uniform \\
Main module dropout & 0.05--0.25 & Uniform \\
Number of ResNet layers & 2--5 & 1 \\
ResNet layers dimension & 8--32 & 1 \\
Learning rate (lr) & \num{1e-4}--\num{1e-3} & Loguniform \\
AdamW weight decay & \num{1e-6}--\num{1e-3} & Loguniform \\
\bottomrule
\end{tabular}
\caption{Hyperparameters of LassoFlexNet Tuned by Optuna and Their Ranges, in Feature Selection and TabZilla Benchmarks}
\label{tab:hypertuned}
\end{table}

\begin{table}[ht]
  \centering
  \begin{tabular}{@{}lcc@{}}
  \toprule
  \textbf{Hyperparameter} & \textbf{Range} & \textbf{Step/Scale} \\
  \midrule
  Mini batch sizes & 1024--4096 & \text{scaled by} 2 \\
  Number of Mixer blocks & 2--6 & 1 \\
  MLP1\&2 hidden dimension & 64--512 & \text{scaled by} 2 \\
  Feature embedding output dimension & 8--32 & 2 \\
  Number of bins used in PLE & 16--64 & 2 \\
  Tau (scalar multiplied with $NN_W$) & 0.01--2.0 & Uniform \\
  ResNet dropout & 0.05--0.5 & Uniform \\
  Main module dropout & 0.05--0.25 & Uniform \\
  Number of ResNet layers & 2--8 & 1 \\
  ResNet layers dimension & 8--32 & 1 \\
  Learning rate (lr) & \num{1e-5}--\num{1e-3} & Loguniform \\
  AdamW weight decay & \num{1e-6}--\num{1e-3} & Loguniform \\
  \bottomrule
  \end{tabular}
  \caption{Hyperparameters of LassoFlexNet Tuned by Optuna and Their Ranges, in TabRed Benchmark}
  \label{tab:hypertuned_tabred}
  \end{table}

\textbf{Computing environment:} We conduct our experiments in a Slurm managed environment for batched parallel jobs, and for each individual run or experiment we use a single GPU (A100 or V100), depending on dataset sizes.

\textbf{$\lambda$ Path:} In LassoFlexNet, the $\lambda$ path controls the regularization strength schedule. In this section, we describe our simple scheduling.

Concretely, we construct the $\lambda$ path using polynomial spacing. $p = 0.95$ is our experiment, the other hyperpameters are searched by Optuna.

\begin{align*}
\lambda_0 &:= \texttt{lambda\_start}, \quad
S := \texttt{lambda\_seq\_size}, \quad
p := \texttt{power}, \\
\lambda_E &:= 
\begin{cases}
\texttt{lambda\_end}, & \text{if provided} \\
\lambda_0 \cdot \texttt{lambda\_multiplier}, & \text{otherwise}
\end{cases} \\
b_i &:= \lambda_0^{1/p} + \frac{i}{S-1}\left(\lambda_E^{1/p} - \lambda_0^{1/p}\right), \quad i=0,\dots,S-1, \\
\lambda_i &:= b_i^{p}, \quad i=0,\dots,S-1.
\end{align*}

We generate the $\lambda$ sequence using the exponent-scaled schedule defined above. This design is motivated by empirical observations: under a simple linear schedule, we found that $\lambda$-training primarily reduces validation loss during the early, smaller $\lambda$ values, whereas larger $\lambda$ values later in the sequence rarely yield improvements. This indicates that a linear sequence is overly fine-grained at the beginning and too coarse at the end. By applying an exponent $p < 1$, we skew the distribution of $\lambda$ values to grow more rapidly at the beginning and more slowly at the end, thereby uniformizing the marginal impact of each step in the sequence.

In our hyperparameter search, the starting value $\lambda_0$ is sampled from the interval $[10^{-8}, 10^{-6}]$, and the $\texttt{lambda\_multiplier}$ is sampled from $[10^{3}, 10^{5}]$.

\section{Appendix: PLE Breaks Rotational Invariance}
\label{app:ple_breaks_invariance}
Piecewise Linear Encoding (PLE) for tabular data is intentionally \emph{axis-aligned}: it encodes each numerical feature with a coordinate-wise piecewise linear map defined by feature-wise breakpoints. This section explains, in two phases, why this design breaks rotational invariance in the sense of \cite{ng2004feature}, and why adding a powerful downstream model such as an neural nets trained by SGD, does not restore rotational invariance.

\noindent\textbf{Phase 1 (NTK approach).} In the wide/lazy regime, SGD-trained MLPs behave like kernel regression with an orthogonally invariant neural tangent kernel (NTK) \citep{jacot2018neural}. The induced kernel on raw inputs is the NTK applied to encoded vectors. If the encoder is not rotation-equivariant, the induced kernel is not rotation-invariant, which yields an explicit counterexample.

\noindent\textbf{Phase 2 (general obstruction).} Even beyond NTK, if the downstream learning rule is itself symmetric to orthogonal changes of coordinates in feature space, then the only way the \emph{end-to-end} raw-space algorithm can be rotationally invariant is if the encoder intertwines raw rotations with feature-space orthogonal maps. PLE generally does not.

\subsection{Setup and definitions}

\subsubsection{Rotational invariance}
Let \(O(n)=\{M\in\R^{n\times n}: M^\top M=I\}\) denote the orthogonal group \citep{ng2004feature}.
A labeled dataset of size \(m\) is \(S=\{(x^{(i)},y^{(i)})\}_{i=1}^m\) with \(x^{(i)}\in\R^n\).
For \(M\in O(n)\), define the rotated dataset
\[
M S := \{(M x^{(i)}, y^{(i)})\}_{i=1}^m.
\]

\begin{definition}[Rotational invariance of a learning algorithm]
Let \(L\) be a (possibly randomized) learning algorithm. We write \(L[S](x)\) for the prediction (a random variable if \(L\) is randomized) obtained by training \(L\) on \(S\) and evaluating at \(x\).
We say:
\begin{itemize}[leftmargin=2.2em]
\item \(L\) is \emph{rotationally invariant (deterministic)} if for all \(S\), \(M\in O(n)\), and \(x\),
\[
L[S](x) = L[MS](Mx).
\]
\item \(L\) is \emph{rotationally invariant (randomized)} if for all \(S\), \(M\in O(n)\), and \(x\), the random variables \(L[S](x)\) and \(L[MS](Mx)\) have the same distribution.
\end{itemize}
\end{definition}

\subsubsection{A concrete PLE map (coordinate-wise, piecewise linear)}
For a single scalar feature \(x\in[b_0,b_T]\), fix breakpoints \(b_0<b_1<\dots<b_T\) and segment lengths \(\Delta_t=b_t-b_{t-1}>0\).
Define the clipping operator \(\clip(u)=\min\{1,\max\{0,u\}\}\).

\begin{definition}[PLE for one feature]
The piecewise linear encoding (PLE) map \(\phi:\,[b_0,b_T]\to\R^T\) is
\[
\phi(x) := \big(\phi_1(x),\dots,\phi_T(x)\big),\qquad
\phi_t(x) := \clip\!\left(\frac{x-b_{t-1}}{\Delta_t}\right).
\]
\end{definition}

For an \(n\)-dimensional input \(x=(x_1,\dots,x_n)\), we use a coordinate-wise encoder (with possibly different \(T_j\) per coordinate):
\[
\Phi(x) := \big(\phi^{(1)}(x_1),\ \phi^{(2)}(x_2),\ \dots,\ \phi^{(n)}(x_n)\big)\in\R^D,\qquad D:=\sum_{j=1}^n T_j.
\]
In practical PLE for tabular data, the breakpoints are often \emph{data-dependent} (e.g., learned from axis-aligned trees), but not learned end to end with the neural nets. The counterexamples below hold for any fixed breakpoints, consistent with practice.

\subsection{Phase 1: Counterexample via the NTK/kernel-regression view of SGD-trained MLPs}

The main observation in this section is that we can explicitly construct a counterexample for a simple MLP. Because NTK theory approximately reduces the analysis of wide neural networks to kernel regression, demonstrating a symmetry break in the induced kernel is sufficient to prove that the network breaks rotational invariance.
\subsubsection{From wide MLP + SGD to kernel regression}
Consider an MLP \(f_\theta:\R^D\to\R\) trained on encoded data \(\{(z_i,y_i)\}_{i=1}^m\) with \(z_i\in\R^D\) by (stochastic) gradient descent on squared loss.
In the wide-network, lazy-training regime (weights move little from initialization), the training dynamics linearize around the initialization \(\theta_0\) and are governed by the neural tangent kernel (NTK),
\[
K_{\mathrm{NTK}}(z,z') := \ip{\nabla_\theta f_{\theta_0}(z)}{\nabla_\theta f_{\theta_0}(z')}.
\]
In this regime, the learned predictor is (approximately, and exactly in the infinite-width gradient-flow limit) the kernel regression / kernel ridge solution with kernel \(K_{\mathrm{NTK}}\) \citep{jacot2018neural}.

\medskip
\noindent\textbf{Key symmetry property.}
For standard isotropic random initialization of fully-connected networks, the limiting NTK is \emph{orthogonally invariant} in input space:
\begin{equation}
\label{eq:ntk-orth-inv}
K_{\mathrm{NTK}}(Qz,Qz') = K_{\mathrm{NTK}}(z,z')\qquad \forall\,Q\in O(D).
\end{equation}
Intuitively: there is no preferred coordinate system in feature space.

\subsubsection{Applying the NTK to PLE + MLP}
Fix a raw dataset \(S=\{(x_i,y_i)\}_{i=1}^m\subset \R^n\times\R\). Define encoded points
\[
z_i := \Phi(x_i)\in\R^D.
\]
The NTK view says that training the MLP on PLE features is (approximately) kernel regression with kernel \(K_{\mathrm{NTK}}\) applied to \(\{z_i\}\).

This induces a \emph{raw-space} kernel (possibly data-dependent, if \(\Phi\) is data-dependent):
\begin{equation}
\label{eq:induced-kernel}
K_S(x,x') := K_{\mathrm{NTK}}(\Phi(x),\Phi(x')).
\end{equation}

\paragraph{What rotational invariance would force.}
If the end-to-end pipeline \(x\mapsto \Phi(x)\mapsto\) (SGD-trained MLP) were rotationally invariant, then for every \(M\in O(n)\),
the induced kernels would satisfy
\begin{equation}
\label{eq:kernel-rot-condition}
K_S(x_i,x_j) = K_{MS}(Mx_i,Mx_j)\qquad \text{for all training pairs }(i,j),
\end{equation}
because kernel regression predictions are determined by the Gram matrix on the training set (and its interaction with test points).

We now show \eqref{eq:kernel-rot-condition} fails for PLE via a fully explicit computation.

\subsubsection{An explicit counterexample: Rotation leads to different inner products}
We work in raw dimension \(n=2\). Use the same PLE map on each coordinate with breakpoints \(-1,0,1\) (so \(T=2\), \(\Delta_1=\Delta_2=1\)).
For a scalar \(x\in[-1,1]\), the 2-dimensional PLE is
\[
\phi(x) = (\phi_1(x),\phi_2(x)),\qquad \phi_1(x)=\clip(x+1),\quad \phi_2(x)=\clip(x),
\]
and the full encoder is \(\Phi(x_1,x_2)=(\phi(x_1),\phi(x_2))\in\R^4\).

Let the two raw training points be
\[
x^{(1)} := (0.5,0),\qquad x^{(2)} := (0,0.5).
\]
Compute their encodings:
\[
\phi(0.5)=(1,0.5),\qquad \phi(0)=(1,0),
\]
so
\[
z_1=\Phi(x^{(1)})=(1,0.5,\,1,0),\qquad
z_2=\Phi(x^{(2)})=(1,0,\,1,0.5).
\]
Hence
\[
\ip{z_1}{z_1}=2.25,\quad \ip{z_2}{z_2}=2.25,\quad \ip{z_1}{z_2}=2.0.
\]

Now rotate raw space by \(45^\circ\):
\[
M := \frac{1}{\sqrt{2}}
\begin{pmatrix}
1 & -1\\
1 & 1
\end{pmatrix}\in O(2).
\]
Then
\[
Mx^{(1)}=\Big(\tfrac{1}{2\sqrt{2}},\tfrac{1}{2\sqrt{2}}\Big)\approx(0.3536,0.3536),\qquad
Mx^{(2)}=\Big(-\tfrac{1}{2\sqrt{2}},\tfrac{1}{2\sqrt{2}}\Big)\approx(-0.3536,0.3536).
\]
Encode these rotated points:
\[
\phi(0.3536)=(1,0.3536),\qquad \phi(-0.3536)=(0.6464,0),
\]
so
\[
z_1'=\Phi(Mx^{(1)})=(1,0.3536,\,1,0.3536),\qquad
z_2'=\Phi(Mx^{(2)})=(0.6464,0,\,1,0.3536).
\]
A direct dot-product computation gives
\[
\ip{z_1'}{z_1'}=2.25,\quad
\ip{z_2'}{z_2'}\approx 1.5429,\quad
\ip{z_1'}{z_2'}\approx 1.7714.
\]
In particular, the feature-space geometry has changed: for example \(\ip{z_1}{z_2}=2.0\) while \(\ip{z_1'}{z_2'}\approx 1.7714\). The rotation leads to changes on the NTK component. We summarize these as a lemma:

\begin{lemma}[PLE breaks Gram-matrix matching for dot-product kernels]
\label{lem:gram-break}
Let \(k:\R^D \times \R^D\to\R\) be any kernel that is a nonconstant function of the inner product \(\ip{z}{z'}\) and the norms \(\norm{z},\norm{z'}\).
For the PLE encoder $\Phi$, there exists a Gram matrix \(\big(k(z_i,z_j)\big)_{i,j\in\\{1,2\\}}\) that changes under the rotation \(M\).
\end{lemma}

\subsubsection{Turning kernel mismatch into a prediction mismatch}
A kernel regressor (e.g., kernel ridge regression) trained on a dataset is determined by the Gram matrix on the training set.
Therefore, once the Gram matrix differs between \(S\) and \(MS\), we can force a change in predictions by an appropriate choice of labels.

To keep the algebra transparent, we illustrate this using kernel ridge regression with the \emph{linear} kernel \(k(z,z')=\ip{z}{z'}\) and a ridge penalty of \(\lambda=0.1\). Let \(y=(1,-1)\) be the label vector for \((x^{(1)},x^{(2)})\). Substituting any fixed, non-degenerate dot-product kernel would yield the same conclusion.

\begin{example}[Explicit prediction mismatch for kernel ridge regression]
\label{app:ntk_regression_example}
With the setup above, kernel ridge regression trained on \(\{(z_1,y_1),(z_2,y_2)\}\) and evaluated at \(x^{(1)}\) yields a different prediction than kernel ridge regression trained on the rotated dataset and evaluated at \(Mx^{(1)}\).
\end{example}

\begin{proof}
Kernel ridge regression has the closed form
\[
f(x)=k_x^\top (K+\lambda I)^{-1}y,\qquad (k_x)_j := k(z, z_j),\ z:=\Phi(x),
\]
where \(K_{ij}=k(z_i,z_j)\).
For the unrotated dataset,
\[
K=
\begin{pmatrix}
2.25 & 2.0\\
2.0 & 2.25
\end{pmatrix},\qquad
k_{x^{(1)}} = (2.25,2.0).
\]
A direct computation gives \(f(x^{(1)})\approx 0.7143\).

For the rotated dataset,
\[
K'=
\begin{pmatrix}
2.25 & 1.7714\\
1.7714 & 1.5429
\end{pmatrix},\qquad
k_{Mx^{(1)}} = (2.25,1.7714),
\]
and solving the same linear system gives \(f'(Mx^{(1)}) \approx 0.5276\).
Therefore \(f(x^{(1)}) \neq f'(Mx^{(1)})\), so the learning rule is not rotationally invariant.
\end{proof}

\subsubsection{The NTK conclusion}
Example~\ref{app:ntk_regression_example} used the linear kernel only to keep arithmetic transparent.
Lemma~\ref{lem:gram-break} shows that for a broad class of dot-product kernels, including standard fully-connected NTKs, the Gram matrix changes under rotations once we insert a PLE map. We summarize the counter-example discussion into a proposition.

\begin{proposition}[Phase 1 conclusion: not rotationally invariant in the NTK/lazy regime]
\label{thm:phase1}
Consider the end-to-end pipeline: raw input \(x\in\R^n\) is encoded by a coordinate-wise PLE map \(\Phi\), then a fully-connected MLP is trained on \(\Phi(x)\) by SGD in the wide/lazy regime.
Assume the limiting NTK satisfies the orthogonal invariance property \eqref{eq:ntk-orth-inv} and is nonconstant in \(\ip{z}{z'}\) and norms.
Then the resulting learning algorithm is \emph{not} rotationally invariant.
\end{proposition}

\begin{proof}
By the NTK theory, the learned predictor is kernel regression with kernel \(K_{\mathrm{NTK}}\) on encoded points, i.e., the induced kernel \eqref{eq:induced-kernel} on raw inputs.
Lemma~\ref{lem:gram-break} provides a dataset \(S\) and rotation \(M\) for which the Gram matrix on the training set differs between \(S\) and \(MS\).
Since kernel regression predictions depend on the Gram matrix, there exists a label vector for which predictions differ at corresponding test points (Example \ref{app:ntk_regression_example}).
\end{proof}

\begin{remark}[Finite width: why the message typically persists]
Theorem~\ref{thm:phase1} is a statement about the lazy/wide regime where SGD is well-approximated by kernel regression.
At finite width, SGD no longer exactly equals kernel regression. However, the \emph{source} of symmetry breaking remains: \(\Phi\) is axis-aligned and generally not rotation-equivariant.
Thus, unless training uses additional mechanisms that explicitly enforce rotational invariance in raw space, one should not expect invariance to hold in general.
\end{remark}

\subsection{Phase 2: A general obstruction --- invariance cannot be recovered without encoder equivariance}

Phase 1 used the NTK story (SGD behaves like a kernel method) to make the failure concrete.
We now present a more general, kernel-free viewpoint that applies to non-lazy regimes.

\subsubsection{Two symmetry notions (raw space vs feature space)}
Let \(S=\{(x_i,y_i)\}_{i=1}^m\) be a raw dataset.
Let \(\Phi_S:\R^n\to\R^D\) be a (possibly data-dependent) encoder and define the encoded dataset
\[
Z_S := \{(z_i,y_i)\}_{i=1}^m,\qquad z_i := \Phi_S(x_i).
\]
A downstream learning algorithm \(A\) maps an encoded dataset \(Z\) to a predictor \(A(Z):\R^D\to\mathcal Y\).
The end-to-end learner is
\[
L[S](x) := A(Z_S)\big(\Phi_S(x)\big).
\]

\begin{definition}[Orthogonal invariance of the downstream learner]
\label{def:orth-inv-A}
For \(Q\in O(D)\), define \(QZ:=\{(Qz_i,y_i)\}_{i=1}^m\).
We say \(A\) is \emph{orthogonally invariant} if for all datasets \(Z\), all \(Q\in O(D)\), and all feature vectors \(z\in\R^D\),
\[
A(QZ)(Qz)=A(Z)(z).
\]
\end{definition}

This captures the idea that the downstream training has no preferred coordinate system in feature space.
(For example, standard fully-connected networks with isotropic initialization and SGD dynamics satisfy this kind of invariance in an idealized sense.)

\subsubsection{A sufficient condition for end-to-end rotational invariance}
The next definition isolates the \emph{only} way a feature-space orthogonal symmetry can translate back to the original space rotational symmetry.

\begin{definition}[Encoder equivariance up to a feature-space orthogonal map]
\label{def:encoder-equiv}
The encoder family \(\{\Phi_S\}\) is \emph{rotation-equivariant up to feature orthogonals} if for every raw dataset \(S\) and every \(M\in O(n)\), there exists \(Q_{S,M}\in O(D)\) such that for all raw inputs \(x\),
\[
\Phi_{MS}(Mx) = Q_{S,M}\,\Phi_S(x).
\]
(Here \(Q_{S,M}\) may depend on \(S\) and \(M\), but must be a single global orthogonal map, independent of \(x\).)
\end{definition}

\begin{proposition}[Sufficiency: encoder equivariance + feature-space invariance implies raw rotational invariance]
\label{prop:sufficiency}
If \(A\) is orthogonally invariant (Definition~\ref{def:orth-inv-A}) and \(\Phi\) is rotation-equivariant up to a feature-space orthogonal (Definition~\ref{def:encoder-equiv}), then the composite learner \(L[S](x)=A(Z_S)(\Phi_S(x))\) is rotationally invariant.
\end{proposition}

\begin{proof}
Fix \(S\), \(M\), \(x\). Let \(Q:=Q_{S,M}\) be from Definition~\ref{def:encoder-equiv}.
For each training point \(x_i\),
\[
\Phi_{MS}(Mx_i)=Q\Phi_S(x_i),
\]
so \(Z_{MS}=QZ_S\).
Also \(\Phi_{MS}(Mx)=Q\Phi_S(x)\).
Therefore
\[
L[MS](Mx)=A(Z_{MS})(\Phi_{MS}(Mx))=A(QZ_S)(Q\Phi_S(x))=A(Z_S)(\Phi_S(x))=L[S](x),
\]
where we used orthogonal invariance of \(A\) in the third equality.
\end{proof}

\subsubsection{The obstruction: without encoder equivariance, invariance has no reason to hold}
Proposition~\ref{prop:sufficiency} gives a necessary condition for rotational invariance: the encoder must provide a global orthogonal alignment \(Q_{S,M}\) between the two encoded problems.

When the encoder fails this, a symmetric downstream learner has no canonical way to ``undo'' the rotation, because the rotated problem is not simply an orthogonal reparameterization of the original in feature space.

To state this as a rigorous impossibility, we need one mild nondegeneracy: the learner should actually react to changes in the geometry of the input cloud. This extra condition will imply the encoder equivariance is nearly necessary.

\begin{definition}[Non-degeneracy: separation property]
\label{def:separating}
We say a downstream learning algorithm \(A\) is \emph{separating} if the following holds:
whenever two encoded datasets \(Z=\{(z_i,y_i)\}\) and \(Z'=\{(z_i',y_i)\}\) (same labels, different feature vectors) are not related by any single orthogonal map, i.e.,
\[
\not\exists\,Q\in O(D)\ \text{s.t.}\ z_i'=Qz_i\ \forall i,
\]
then there exists some query \(z\in\R^D\) for which the predictions differ:
\[
A(Z)(z)\neq A(Z')(z).
\]
\end{definition}

This property is satisfied by many standard algorithms (e.g., kernel ridge regression with a nondegenerate kernel, linear ridge regression, and wide-MLP training in the NTK regime).

\begin{theorem}[Obstruction: symmetry cannot restore invariance without encoder equivariance]
\label{thm:obstruction}
Assume learning algorithm \(A\) is orthogonally invariant and separating.
If the encoder family \(\{\Phi_S\}\) is \emph{not} rotation-equivariant up to a feature-space orthogonal map (Definition~\ref{def:encoder-equiv}), then the composite learner \(L = A \circ \Phi \) is not rotationally invariant.
\end{theorem}

\begin{proof}
Because encoder equivariance fails, there exist \(S\) and \(M\in O(n)\) such that no \(Q\in O(D)\) satisfies \(\Phi_{MS}(Mx_i)=Q\Phi_S(x_i)\) for all training inputs \(x_i\) in \(S\).
Thus the encoded datasets \(Z_S\) and \(Z_{MS}\) are not related by any orthogonal transform in feature space.

By the separating property, there exists a query feature vector \(z\) such that
\[
A(Z_S)(z)\neq A(Z_{MS})(z').
\]
Choose a raw query \(x\) and set \(z=\Phi_S(x)\), \(z'=\Phi_{MS}(Mx)\).
Then
\[
L[S](x)=A(Z_S)(\Phi_S(x))\neq A(Z_{MS})(\Phi_{MS}(Mx))=L[MS](Mx),
\]
so \(L\) is not rotationally invariant.
\end{proof}

\begin{remark}[Stochastic analogues]
    If we replace the deterministic equality in Definition \ref{def:separating} and Theorem \ref{thm:obstruction} by expectation, we can produce the stochastic analogues that cover neural nets trained by SGD.
\end{remark}

\subsubsection{SGD-trained MLPs cannot unlearn the rotation under PLE}
Theorem~\ref{thm:obstruction} clarifies a few subtle but important points.

\begin{itemize}[leftmargin=2.2em]
\item \textbf{Expressivity is not the issue.} An MLP is a universal approximator; in principle, with enough data, it can learn complicated functions that may correlate with a rotated target concept.

\item \textbf{Algorithmic rotational invariance is the issue.} Rotational invariance is a statement about the \emph{learning procedure}: for every dataset and every rotation, training on \(S\) and testing at \(x\) must agree with training on \(MS\) and testing at \(Mx\).

\item \textbf{PLE breaks the symmetry before SGD begins.} Coordinate-wise PLE creates axis-aligned ``kinks'' (breakpoints) and saturations. A generic raw rotation changes how points fall relative to these kinks. Unless the encoder is rotation-equivariant up to a global orthogonal map in feature space, a symmetric training rule (such as standard SGD on an MLP with isotropic initialization) has no canonical mechanism that forces it to output a rotationally invariant predictor.
\end{itemize}

\subsection{Discussion: why breaking rotational invariance is often the point in tabular learning}
Rotational invariance is attractive when the coordinate system is arbitrary (e.g., images after whitening).
Tabular data is often different: coordinates correspond to semantically meaningful features, and feature selection or axis-aligned splits can be beneficial.
Specifically, tabular datasets frequently contain noisy, irrelevant features; a rotationally invariant model would inextricably mix these with informative features, degrading performance.
In fact, \cite{ng2004feature} gives lower bounds showing that \emph{any} rotationally invariant learner can be forced to require \(\Omega(n)\) samples on an ``easy'' one-feature threshold problem after an adversarial rotation.
From that perspective, PLE's axis alignment is not a bug: it is a deliberate inductive bias toward feature-wise heterogeneity.

\section{Appendix: PLE Brings in Extra Localized Expressivity}
\label{app:ple_more_expressivie}
PLE increases capacity in a manner ideally suited for tabular problems: it creates \textbf{bin-local degrees of freedom} (local slopes) with \textbf{localized training dynamics} (zero cross-bin interference in the slope NTK, which is ideal for capturing localized irregularities), while simultaneously increasing the effective dimension of the standard (value) NTK.

\textbf{Scope}.
We work in the simplest setting of a \emph{single numerical feature} $x\in[b_0,b_T]$ and compare:
(i) a linear model on top of PLE (``linear-on-PLE''); and
(ii) a one-hidden-layer scalar ReLU MLP on the raw scalar (``MLP-without-PLE''). In other words, we investigate the benefits of PLE for a single numerical feature.

\subsection{Setup and notation}

\begin{definition}[Knots, bins, and bin index]
\label{def:knots}
Fix knots
\[
b_0<b_1<\cdots<b_T,
\qquad I\coloneqq[b_0,b_T],
\qquad \DeltaT_t\coloneqq b_t-b_{t-1}>0.
\]
For $x\in(b_0,b_T)$, define the \emph{bin index} $s(x)\in\{1,\dots,T\}$ by
\[
s(x)=t \quad\Longleftrightarrow\quad x\in(b_{t-1},b_t).
\]
Let the boundary set be $\mathcal B\coloneqq\{b_0,\dots,b_T\}$.
\end{definition}

\begin{definition}[Clipping operator]
For $u\in\R$, define $\clip(u)\coloneqq\min\{1,\max\{0,u\}\}\in[0,1]$.
\end{definition}

\begin{definition}[Piecewise linear encoding (PLE) map for one feature]
\label{def:ple}
Define $\phiPLE:I\to\R^T$ by $\phiPLE(x)=(\phiPLE_1(x),\dots,\phiPLE_T(x))$, where for $t\in\{1,\dots,T\}$,
\begin{equation}
\label{eq:ple-app}
\phiPLE_t(x)\coloneqq \clip\!\left(\frac{x-b_{t-1}}{\DeltaT_t}\right).
\end{equation}
Equivalently,
\[
\phiPLE_t(x)=
\begin{cases}
0, & x\le b_{t-1},\\[2pt]
\dfrac{x-b_{t-1}}{\DeltaT_t}, & b_{t-1}\le x\le b_t,\\[8pt]
1, & x\ge b_t.
\end{cases}
\]
\end{definition}

\begin{remark}[``Almost everywhere'' (a.e.) differentiability convention]
\label{rem:ae}
All derivatives with respect to $x$ are understood on $I\setminus\mathcal B$ (and similarly for MLPs, away from their kink set).
At points where a piecewise-linear function is not differentiable, any selection of subgradient can be used; this affects only a measure-zero set in $x$ and does not change the kernel identities stated ``a.e.''.
\end{remark}

\subsection{Phase 1: Linear-on-PLE is exactly the linear spline class (full derivation)}

\begin{definition}[Linear-on-PLE model]
\label{def:lin-on-ple}
Given $(w,c)\in\R^T\times\R$, define
\[
f_{w,c}(x)\coloneqq c+w^\top \phiPLE(x),\qquad x\in I.
\]
\end{definition}

\begin{definition}[Linear spline space with knots $\{b_t\}$]
\label{def:spline-space}
Let $\mathcal S(b_{0:T})$ denote the set of continuous functions $s:I\to\R$ that are affine on each bin $[b_{t-1},b_t]$.
Equivalently, $s\in\mathcal S(b_{0:T})$ iff there exist slopes $(\alpha_1,\dots,\alpha_T)\in\R^T$ and a base value $s(b_0)\in\R$ such that, for $x\in[b_{t-1},b_t]$,
\begin{equation}
\label{eq:spline-param}
s(x)=s(b_0)+\sum_{j=1}^{t-1}\alpha_j\DeltaT_j+\alpha_t(x-b_{t-1}).
\end{equation}
\end{definition}

\begin{theorem}[Spline equivalence with explicit forward formula and converse]
\label{thm:spline-equiv}
The class $\{f_{w,c}:\ (w,c)\in\R^T\times\R\}$ equals $\mathcal S(b_{0:T})$.
Moreover:
\begin{enumerate}
\item \textbf{(Forward / explicit spline form).}
Fix $(w,c)$. For $x\in(b_{t-1},b_t)$ (so $t=s(x)$),
\begin{equation}
\label{eq:f-explicit}
f_{w,c}(x)=c+\sum_{j=1}^{t-1}w_j \;+\; w_t\frac{x-b_{t-1}}{\DeltaT_t}.
\end{equation}
Hence $f_{w,c}$ is a continuous linear spline with per-bin slope
\begin{equation}
\label{eq:slope-ple}
\partial_x f_{w,c}(x)=\frac{w_{s(x)}}{\DeltaT_{s(x)}} \qquad\text{for }x\in I\setminus\mathcal B.
\end{equation}

\item \textbf{(Converse / explicit weight construction).}
Given any spline $s\in\mathcal S(b_{0:T})$ with representation \eqref{eq:spline-param}, define
\begin{equation}
\label{eq:weights-from-slopes}
c\coloneqq s(b_0),\qquad w_t\coloneqq \alpha_t\DeltaT_t \quad (t=1,\dots,T).
\end{equation}
Then $f_{w,c}(x)=s(x)$ for all $x\in I$.
\end{enumerate}
\end{theorem}

\begin{proof}
We proceed in explicit steps.

\paragraph{Step 1: Evaluate $\phiPLE(x)$ inside a bin.}
Fix $x\in(b_{t-1},b_t)$ with $t=s(x)$.
From Definition~\ref{def:ple}:
\[
\phiPLE_j(x)=
\begin{cases}
1, & j\le t-1 \quad(\text{since }x\ge b_j),\\[2pt]
\dfrac{x-b_{t-1}}{\DeltaT_t}, & j=t,\\[8pt]
0, & j\ge t+1 \quad(\text{since }x\le b_{j-1}).
\end{cases}
\]

\paragraph{Step 2: Expand $w^\top \phiPLE(x)$.}
Using the binwise values above,
\[
w^\top\phiPLE(x)
=\sum_{j=1}^{t-1} w_j\cdot 1 \;+\; w_t\cdot \frac{x-b_{t-1}}{\DeltaT_t}\;+\;\sum_{j=t+1}^T w_j\cdot 0
=\sum_{j=1}^{t-1}w_j+w_t\frac{x-b_{t-1}}{\DeltaT_t}.
\]
Adding $c$ gives \eqref{eq:f-explicit}. This shows $f_{w,c}$ is affine on each bin and continuous at the knots (the formula matches at endpoints), hence $f_{w,c}\in\mathcal S(b_{0:T})$.

\paragraph{Step 3: Compute the per-bin slope.}
For $x\in(b_{t-1},b_t)$, differentiating \eqref{eq:f-explicit} yields
\[
\partial_x f_{w,c}(x)=\frac{w_t}{\DeltaT_t},
\]
which is \eqref{eq:slope-ple}.

\paragraph{Step 4: Converse construction.}
Let $s\in\mathcal S(b_{0:T})$ be written as \eqref{eq:spline-param}.
Define $(w,c)$ by \eqref{eq:weights-from-slopes}.
Fix $x\in(b_{t-1},b_t)$. Then using \eqref{eq:f-explicit},
\[
f_{w,c}(x)
=c+\sum_{j=1}^{t-1}w_j+w_t\frac{x-b_{t-1}}{\DeltaT_t}
=s(b_0)+\sum_{j=1}^{t-1}\alpha_j\DeltaT_j+\alpha_t\DeltaT_t\frac{x-b_{t-1}}{\DeltaT_t}.
\]
Canceling $\DeltaT_t$ in the last term gives exactly \eqref{eq:spline-param}:
\[
f_{w,c}(x)=s(b_0)+\sum_{j=1}^{t-1}\alpha_j\DeltaT_j+\alpha_t(x-b_{t-1})=s(x).
\]
By continuity, equality holds for all $x\in I$.
\end{proof}

\begin{corollary}[Kinks at knots and the ``generic kink'' condition]
\label{cor:kink}
Let $f_{w,c}$ be linear-on-PLE. For an interior knot $b_t$ ($t=1,\dots,T-1$), the spline is differentiable at $b_t$ iff the left and right slopes match:
\[
\frac{w_t}{\DeltaT_t}=\frac{w_{t+1}}{\DeltaT_{t+1}}.
\]
Equivalently, $b_t$ is a kink iff $\frac{w_t}{\DeltaT_t}\neq \frac{w_{t+1}}{\DeltaT_{t+1}}$.
In particular, if $w$ is drawn from a distribution with a density on $\R^T$, then with probability $1$ all interior knots are kinks.
\end{corollary}

\begin{proof}
The left derivative at $b_t$ equals the slope on $(b_{t-1},b_t)$, i.e. $w_t/\DeltaT_t$, and the right derivative equals the slope on $(b_t,b_{t+1})$, i.e. $w_{t+1}/\DeltaT_{t+1}$, by \eqref{eq:slope-ple}. Equality of one-sided derivatives is equivalent to differentiability at $b_t$.
The probability-one statement follows because $\{w:\ w_t/\DeltaT_t=w_{t+1}/\DeltaT_{t+1}\}$ is a codimension-$1$ hyperplane in $\R^T$.
\end{proof}

\begin{remark}
Theorem~\ref{thm:spline-equiv} formally identifies linear-on-PLE with the linear spline function class.
Next, we bound the number of hidden units a standard ReLU MLP requires to match this knot complexity, before quantifying PLE's \emph{localized} learning behavior via NTK analysis.
\end{remark}

\subsection{Phase 2: Degrees of freedom vs a 1-hidden-layer scalar ReLU MLP (knot counting)}

\begin{definition}[MLP-without-PLE: one-hidden-layer scalar ReLU net]
\label{def:mlp}
Fix width $m\in\mathbb N$. Consider
\begin{equation}
\label{eq:mlp}
f_\theta(x)=\sum_{r=1}^m a_r\,\ReLU(w_r x+b_r)+c,
\qquad \ReLU(u)=\max\{0,u\},
\end{equation}
with parameters $\theta=(a,w,b,c)\in\R^m\times\R^m\times\R^m\times\R$.
\end{definition}

\begin{proposition}[Knot counting for one-hidden-layer ReLU]
\label{prop:knot-count}
For any fixed parameters $\theta$, the function $x\mapsto f_\theta(x)$ is continuous piecewise linear on $\R$ with at most $m$ interior breakpoints (kinks).
Consequently, to represent a \emph{generic} spline in $\mathcal S(b_{0:T})$ having $T-1$ distinct interior knots, a one-hidden-layer ReLU net must have width
\[
m\ge T-1.
\]
\end{proposition}

\begin{proof}
For each neuron $r$, the map $x\mapsto \ReLU(w_r x+b_r)$ is piecewise linear with a single possible kink at the point where $w_r x+b_r=0$ (if $w_r\neq 0$), i.e.
\[
x_r^\star=-\frac{b_r}{w_r}.
\]
If $w_r=0$, then $\ReLU(w_r x+b_r)=\ReLU(b_r)$ is constant and contributes no kink.
Therefore, the kink set of the sum \eqref{eq:mlp} is contained in the union of the individual kink locations $\{x_r^\star:\ w_r\neq 0\}$, so it has at most $m$ distinct points.
This proves the first claim.

For the consequence, note that a generic spline with knots exactly at $b_1,\dots,b_{T-1}$ has $T-1$ distinct interior kink locations.
Since a width-$m$ network can have at most $m$ distinct kink locations, representing such a spline requires $m\ge T-1$.
\end{proof}

\begin{remark}[Why this comparison is the right one]
Deep ReLU nets can create many breakpoints, but the key point here is structural:
\emph{without PLE, breakpoint locations must be learned by the network};
\emph{with PLE, the knot grid is built into the representation and learning can focus on per-bin slopes}.
\emph{Next, we discuss the localized slope learning given by PLE.}
\end{remark}

\subsection{Phase 3: Training dynamics via kernels --- value NTK vs slope NTK}
The preceding sections demonstrate that PLE increases expressivity by adopting the strengths of tree-based methods to create explicit breakpoints. We now formalize another structural benefit: PLE's training dynamics are inherently more localized, making them ideal for modeling local irregularities in numerical tabular data. We prove this using Neural Tangent Kernel (NTK) techniques.

\subsubsection{Standard (value) NTK: definitions and explicit formulas}

\begin{definition}[Value Jacobian and value NTK]
\label{def:value-ntk}
For a parameterized model $f_\theta(x)$, define the value Jacobian
\[
J_{\mathrm{val}}(x)\coloneqq\nabla_\theta f_\theta(x),
\]
and the (finite-width) value NTK
\[
K_{\mathrm{val}}(x,x')\coloneqq \ip{J_{\mathrm{val}}(x)}{J_{\mathrm{val}}(x')}.
\]
\end{definition}
Further background can be found in \citet{jacot2018neural}.

\begin{proposition}[Value NTK for linear-on-PLE is explicit and nonlocal across bins]
\label{prop:value-ple}
For linear-on-PLE (Definition~\ref{def:lin-on-ple}) with parameters $\theta=(w,c)$,
\[
J_{\mathrm{val}}(x)=(\phiPLE(x)\,|\,1)\in\R^{T+1},
\qquad
K_{\mathrm{val}}^{\mathrm{PLE}}(x,x')=1+\ip{\phiPLE(x)}{\phiPLE(x')}.
\]
Moreover, for $x\in(b_{s-1},b_s)$ and $x'\in(b_{t-1},b_t)$ with $s=s(x)$ and $t=s(x')$, letting
\[
\alpha(x)\coloneqq \frac{x-b_{s-1}}{\DeltaT_s}\in(0,1),\qquad
\alpha(x')\coloneqq \frac{x'-b_{t-1}}{\DeltaT_t}\in(0,1),
\]
we have the closed form
\[
\ip{\phiPLE(x)}{\phiPLE(x')}
=
(\min\{s,t\}-1)+
\begin{cases}
\alpha(x), & s<t,\\
\alpha(x)\alpha(x'), & s=t,\\
\alpha(x'), & s>t.
\end{cases}
\]
In particular, $K_{\mathrm{val}}^{\mathrm{PLE}}(x,x')$ is typically \emph{nonzero} even when $s(x)\neq s(x')$.
\end{proposition}

\begin{proof}
For $\theta=(w,c)$, we have $\nabla_w f_{w,c}(x)=\phiPLE(x)$ and $\partial f_{w,c}(x)/\partial c=1$, hence $J_{\mathrm{val}}(x)=(\phiPLE(x)\,|\,1)$ and the kernel formula follows by dot product.

For the binwise closed form, use the explicit structure of $\phiPLE$ on a bin:
if $x\in(b_{s-1},b_s)$ then
\[
\phiPLE_j(x)=
\begin{cases}
1, & j\le s-1,\\
\alpha(x), & j=s,\\
0, & j\ge s+1,
\end{cases}
\]
and similarly for $x'$.
Taking the dot product, only coordinates $j\le \min\{s,t\}$ contribute.
If $s<t$, then $\phiPLE_s(x')=1$ and $\phiPLE_j(x')=1$ for $j\le s-1$, giving
$\ip{\phiPLE(x)}{\phiPLE(x')}=(s-1)+\alpha(x)$.
If $s=t$, the only fractional coordinate is $j=s$, giving $(s-1)+\alpha(x)\alpha(x')$.
The case $s>t$ is symmetric.
\end{proof}

\begin{proposition}[Exact value NTK for a one-hidden-layer ReLU MLP (no PLE)]
\label{prop:value-mlp}
For the MLP \eqref{eq:mlp}, define $z_r(x)\coloneqq w_r x+b_r$ and $\mathbb I_r(x)\coloneqq\1\{z_r(x)>0\}$.
For $x$ away from the kink set $\{z_r(x)=0\}$, the value Jacobian has entries
\[
\frac{\partial f_\theta(x)}{\partial a_r}=\ReLU(z_r(x)),\qquad
\frac{\partial f_\theta(x)}{\partial w_r}=a_r\,x\,\mathbb I_r(x),\qquad
\frac{\partial f_\theta(x)}{\partial b_r}=a_r\,\mathbb I_r(x),\qquad
\frac{\partial f_\theta(x)}{\partial c}=1.
\]
Consequently, for such $x,x'$,
\begin{equation}
\label{eq:value-mlp}
K_{\mathrm{val}}^{\mathrm{MLP}}(x,x')
=
1+\sum_{r=1}^m\Big[
\ReLU(z_r(x))\,\ReLU(z_r(x'))
+a_r^2(1+xx')\,\mathbb I_r(x)\mathbb I_r(x')
\Big].
\end{equation}
\end{proposition}

\begin{proof}
We compute derivatives neuron-by-neuron.
First, $\partial f/\partial a_r=\ReLU(z_r(x))$ is immediate.
Next, $\partial\ReLU(z)/\partial z=\mathbb I_r(x)$ away from $z=0$, so by chain rule:
\[
\frac{\partial f}{\partial w_r}
=a_r\frac{\partial \ReLU(z_r(x))}{\partial z_r}\cdot \frac{\partial z_r(x)}{\partial w_r}
=a_r\,\mathbb I_r(x)\cdot x,
\qquad
\frac{\partial f}{\partial b_r}
=a_r\,\mathbb I_r(x)\cdot 1.
\]
Also $\partial f/\partial c=1$.

Now take the dot product of value Jacobians at $x$ and $x'$ and sum contributions blockwise:
\begin{itemize}
\item $(a)$-block:
\[
\sum_{r=1}^m \ReLU(z_r(x))\,\ReLU(z_r(x')).
\]
\item $(w)$-block:
\[
\sum_{r=1}^m (a_r x\,\mathbb I_r(x))\,(a_r x'\,\mathbb I_r(x'))
=\sum_{r=1}^m a_r^2\,xx'\,\mathbb I_r(x)\mathbb I_r(x').
\]
\item $(b)$-block:
\[
\sum_{r=1}^m a_r^2\,\mathbb I_r(x)\mathbb I_r(x').
\]
\item $c$-block: $1$.
\end{itemize}
Combining these terms gives \eqref{eq:value-mlp}.
\end{proof}

\begin{remark}[Effective dimension / rank: PLE increases value-kernel rank]
\label{rem:rank}
Consider training on a finite dataset $\{x_i\}_{i=1}^n\subset(b_0,b_T)$.
The Gram matrix of the value kernel for linear-on-PLE is
\[
\Big(K_{\mathrm{val}}^{\mathrm{PLE}}(x_i,x_j)\Big)_{i,j}
=
\mathbf 1\mathbf 1^\top+\Phi\Phi^\top,
\qquad
\Phi\in\R^{n\times T},\ \Phi_{i,:}=\phiPLE(x_i)^\top.
\]
Hence $\mathrm{rank}(K_{\mathrm{val}}^{\mathrm{PLE}})\le T+1$, and for generic data spanning many bins, the rank is typically close to $T+1$.

By contrast, for the raw linear model $f(x)=\beta x+c$ (no PLE), we have
$J_{\mathrm{val}}(x)=(x,1)$ and thus $K_{\mathrm{val}}(x,x')=xx'+1$, whose Gram matrix has rank at most $2$.
This formalizes that PLE increases the number of independent ``directions'' in which function values can be fitted in the kernel (lazy-training) picture.
\end{remark}

\begin{remark}[Refinement under local input changes: PLE's value kernel is cumulative and continuous]
Both $K_{\mathrm{val}}^{\mathrm{PLE}}$ and $K_{\mathrm{val}}^{\mathrm{MLP}}$ can be nonlocal.
However, $K_{\mathrm{val}}^{\mathrm{PLE}}(x,x')$ varies \emph{continuously} with $x$ within each bin and changes only at the fixed knots $\{b_t\}$, giving a cumulative ``similarity'' structure along the scalar feature (Proposition~\ref{prop:value-ple}).
In contrast, for fixed MLP parameters, $K_{\mathrm{val}}^{\mathrm{MLP}}(x,x')$ changes its algebraic form whenever $x$ crosses one of the learned thresholds $-b_r/w_r$ (activation pattern changes), which can produce more abrupt changes in similarity as a function of $x$.
This is a qualitative comparison (not a nonlocality claim): it describes how PLE induces a stable, pre-specified partition of $x$-space.
\end{remark}

\subsubsection{Slope NTK: definitions and exact finite-width formulas}

\begin{definition}[Slope function, slope Jacobian, and slope NTK]
\label{def:slope-ntk}
For a model $f_\theta(x)$ that is differentiable in $x$ on $I\setminus\mathcal B$, define the \emph{slope function}
\[
g_\theta(x)\coloneqq \partial_x f_\theta(x)\qquad (x\in I\setminus\mathcal B),
\]
the \emph{slope Jacobian}
\[
J_{\mathrm{slope}}(x)\coloneqq \nabla_\theta g_\theta(x),
\]
and the \emph{slope NTK}
\[
K_{\mathrm{slope}}(x,x')\coloneqq \ip{J_{\mathrm{slope}}(x)}{J_{\mathrm{slope}}(x')},
\qquad x,x'\in I\setminus\mathcal B.
\]
\end{definition}

\begin{remark}[Why slope NTK is meaningful]
The standard NTK describes learning dynamics of function values $f_\theta(x)$.
The slope NTK describes learning dynamics of the derived scalar quantity $\partial_x f_\theta(x)$, i.e. a (coordinate) partial derivative.
This makes ``local shape'' learning explicit: it quantifies how parameter updates slopes at different inputs.
\end{remark}

\begin{proposition}[Exact slope NTK for a one-hidden-layer ReLU MLP (no PLE)]
\label{prop:slope-mlp}
For the MLP \eqref{eq:mlp}, for $x$ away from the kink set $\{z_r(x)=0\}$,
\[
g_\theta(x)=\partial_x f_\theta(x)=\sum_{r=1}^m a_r w_r\,\mathbb I_r(x).
\]
Moreover, for such $x$ (in the ``a.e.'' sense of Remark~\ref{rem:ae}),
\[
\frac{\partial g_\theta(x)}{\partial a_r}=w_r\,\mathbb I_r(x),\qquad
\frac{\partial g_\theta(x)}{\partial w_r}=a_r\,\mathbb I_r(x),\qquad
\frac{\partial g_\theta(x)}{\partial b_r}=0,\qquad
\frac{\partial g_\theta(x)}{\partial c}=0.
\]
Consequently, for $x,x'$ away from the kink set,
\begin{equation}
\label{eq:slope-mlp}
K_{\mathrm{slope}}^{\mathrm{MLP}}(x,x')
=\sum_{r=1}^m (w_r^2+a_r^2)\,\mathbb I_r(x)\mathbb I_r(x').
\end{equation}
\end{proposition}

\begin{proof}
For $x$ with $z_r(x)\neq 0$, we have $\partial_x \ReLU(z_r(x))=\mathbb I_r(x)\cdot \partial_x z_r(x)=\mathbb I_r(x)\,w_r$.
Thus,
\[
g_\theta(x)=\partial_x f_\theta(x)=\sum_{r=1}^m a_r\,\partial_x \ReLU(z_r(x))
=\sum_{r=1}^m a_r\,w_r\,\mathbb I_r(x).
\]
Treating $\mathbb I_r(x)$ as constant in parameters for this fixed $x$ (valid a.e.), differentiate:
\[
\frac{\partial g}{\partial a_r}=w_r\mathbb I_r(x),\qquad
\frac{\partial g}{\partial w_r}=a_r\mathbb I_r(x),
\]
and $\partial g/\partial b_r=\partial g/\partial c=0$ a.e. (the dependence on $b_r$ enters only through the gate, whose derivative is supported on the measure-zero set $z_r(x)=0$).

Finally, take the dot product of slope Jacobians at $x$ and $x'$.
Only the $(a)$ and $(w)$ blocks contribute, giving
\[
K_{\mathrm{slope}}^{\mathrm{MLP}}(x,x')
=\sum_{r=1}^m \Big[(w_r\mathbb I_r(x))(w_r\mathbb I_r(x'))+(a_r\mathbb I_r(x))(a_r\mathbb I_r(x'))\Big]
=\sum_{r=1}^m (w_r^2+a_r^2)\mathbb I_r(x)\mathbb I_r(x'),
\]
which is \eqref{eq:slope-mlp}.
\end{proof}

\begin{proposition}[Exact slope NTK for linear-on-PLE is diagonal across bins]
\label{prop:slope-ple}
For linear-on-PLE with parameters $\theta=(w,c)$, for $x\in(b_{t-1},b_t)$,
\[
g_{w,c}(x)=\partial_x f_{w,c}(x)=\frac{w_t}{\DeltaT_t},
\qquad
\nabla_w g_{w,c}(x)=\frac{1}{\DeltaT_t}e_t,\qquad
\frac{\partial g_{w,c}(x)}{\partial c}=0.
\]
Therefore, for $x,x'\in I\setminus\mathcal B$,
\begin{equation}
\label{eq:slope-ple-kernel}
K_{\mathrm{slope}}^{\mathrm{PLE}}(x,x')
=
\frac{\1\{s(x)=s(x')\}}{\DeltaT_{s(x)}^2}.
\end{equation}
In particular, if $s(x)\neq s(x')$ then $K_{\mathrm{slope}}^{\mathrm{PLE}}(x,x')=0$ (no cross-bin interference for slope learning).
\end{proposition}

\begin{proof}
For $x\in(b_{t-1},b_t)$, Theorem~\ref{thm:spline-equiv} gives $\partial_x f_{w,c}(x)=w_t/\DeltaT_t$.
Differentiating $g(x)=w_t/\DeltaT_t$ with respect to $w$ yields $\nabla_w g(x)=(1/\DeltaT_t)e_t$, and $\partial g/\partial c=0$.
Taking dot products between the one-hot vectors gives \eqref{eq:slope-ple-kernel}.
\end{proof}

\begin{remark}[Slope NTKs reveal that PLE expressivity is decoupled and localized]
Comparing the slope NTKs $K_{\mathrm{slope}}^{\mathrm{MLP}}(x,x')$ and $K_{\mathrm{slope}}^{\mathrm{PLE}}(x,x')$ provides a clear picture of PLE's expressivity benefits. The PLE slope kernel is strictly localized: two inputs $x$ and $x'$ interact during slope updates if and only if they fall within the same bin. This enables highly precise, decoupled learning of local sensitivities, which is perfectly suited for modeling localized irregularities in numerical features.
\end{remark}

\begin{remark}[Extension: one-hidden-layer MLP on top of PLE (why diagonality can fail)]
If a ReLU MLP is applied to $\phiPLE(x)$, e.g.
\[
h_\Theta(x)=\sum_{r=1}^m a_r\,\ReLU(u_r^\top \phiPLE(x)+b_r),
\]
then (a.e.) $\partial_x h_\Theta(x)=\sum_r a_r\,\1\{u_r^\top\phiPLE(x)+b_r>0\}\,(u_r^\top \phiPLE'(x))$.
Since $\phiPLE'(x)$ is one-hot by bin, $u_r^\top\phiPLE'(x)$ selects a bin-dependent coordinate of $u_r$.
However, the gate $\1\{u_r^\top\phiPLE(x)+b_r>0\}$ can still couple different bins, so exact bin-diagonality like \eqref{eq:slope-ple-kernel} is not guaranteed in general.
\end{remark}

\section{Appendix: Hier-Prox Convergence to a Noisier Ball and its Local Optimality}
\label{app:sec:hier-prox-converges}

The practical LassoNet update applies a stochastic gradient step with stepsize $\eta>0$,
then applies the hierarchical proximal operator with \emph{fixed} $(\lambda,M)$ for many epochs.

\begin{algorithm}[ht]
\caption{Stochastic Hier-Prox update with fixed $(\eta,\lambda,M)$}
\label{alg:main}
\begin{algorithmic}[1]
\Require initial $y_0=(\theta_0,W_0)$, stepsize $\eta>0$, prox weight $\lambda>0$, hierarchy parameter $M>0$
\For{$t=0,1,2,\dots$}
  \State obtain stochastic gradient $g_t$ (e.g.\ minibatch gradient) of $f$ at $y_t$
  \State $z_t \gets y_t - \eta g_t$
  \State $y_{t+1} \gets \Psi_{\lambda,M}(z_t)$ \Comment{LassoNet Hier-Prox operator (blockwise sorting per feature)}
\EndFor
\end{algorithmic}
\end{algorithm}

We analyze the constant-hyperparameter regime used in practice (learning rate $\eta$ and regularization $\lambda$ held independent and constant over many epochs), which naturally leads to a Lyapunov analysis in \emph{function values} for the composite potential $F(y)=\eta f(y)+\Omega(y)$.
Under (i) exact global solution of each proximal subproblem by \textup{\textsc{Hier-Prox}}, (ii) boundedness of iterates on a visited domain, (iii) $L$-smoothness of $f$, and (iv) a local \emph{proximal-PL} inequality stated for the exact hierarchical prox-gradient mapping, we prove a master recursion
\[
\Delta_{t+1}\le (1-c)\Delta_t + C_{\mathrm{noise}}\,\eta^2\sigma^2 + C_{\mathrm{sign}}\,R^2 q_t,
\]
where $\Delta_t=\E[F(y_t)-F(y^\star)]$ and $q_t$ is the probability that stochastic gradient noise flips the active sign/branch between the noiseless and noisy pre-prox points.
We then give \emph{SNR/margin} bounds on $q_t$ (Chebyshev), making explicit how weak signals (small margin) lead to large branch-switch variance, causing training instability.

\subsection{Practical algorithmic scaling and the Lyapunov function}

\subsubsection{The stochastic update}
We focus on the block of parameters to which \textup{\textsc{Hier-Prox}} is applied:
\[
y = (\theta, W)\in \R^{d}\times \R^{d\times K},
\]
where $\theta\in\R^d$ act as feature ``gates'' and $W\in\R^{d\times K}$ are the first-layer weights (row $j$ is $W_{j,\cdot}$) of the MLP NN.
Let $f:\R^{d+dK}\to\R$ be the (possibly nonconvex) smooth training loss as a function of this block.\footnote{In the full neural network,
other parameters may also be present; one can view them as frozen for our analysis.}

Training uses a stochastic gradient $g_t$ (e.g.\ a minibatch gradient) and then applies the hierarchical proximal map:
\begin{equation}
\label{eq:update}
z_t = y_t - \eta g_t,\qquad y_{t+1} = \Psi(z_t),
\end{equation}
where the proximal mapping is
\begin{equation}
\label{eq:psi-def}
\Psi(z)\in \argmin_{y}\left\{\Omega(y)+\frac{1}{2}\norm{y-z}^2\right\}.
\end{equation}
Where $\Omega(y)$ is defined in the subsequent section.

\subsubsection{The hierarchical regularizer/constraint}
The LassoNet hierarchical structure couples each $(\theta_j, W_{j,\cdot})$ through
\[
\norm{W_{j,\cdot}}_\infty \le M\abs{\theta_j}\quad (j=1,\dots,d),
\]
and adds an $\ell_1$ penalty $\lambda\|\theta\|_1$. We package these into an extended-valued function
$\Omega:\R^{d+dK}\to \R\cup\{+\infty\}$:
\begin{equation}
\label{eq:Omega}
\Omega(\theta,W) \;:=\; \lambda \norm{\theta}_1 \;+\; \sum_{j=1}^d \ind\left\{\norm{W_{j,\cdot}}_\infty \le M\abs{\theta_j}\right\},
\end{equation}
where $\ind\{\cdot\}$ is $0$ if the condition holds and $+\infty$ otherwise.

\begin{remark}[Why this scaling differs from standard proximal-gradient textbooks]
Classical proximal-gradient methods typically use
$\argmin_y \{\Omega(y) + \frac{1}{2\eta}\|y-z\|^2\}$, so that $\Omega$ is scaled by $\eta$.
In the original LassoNet training regime, $\eta$, $\lambda$, and $M$ are typically treated as \emph{independent constants}
over many epochs. The proximal step \eqref{eq:psi-def} therefore has a fixed quadratic weight $1/2$.
To obtain a clean descent identity, we analyze the Lyapunov function $F$ in the next section, rather than $f+\Omega$.
\end{remark}

\subsubsection{The Lyapunov function}
Define the potential
\begin{equation}
\label{eq:Fdef}
F(y) := \eta f(y) + \Omega(y).
\end{equation}
The key intuition is that \eqref{eq:update} is precisely a stochastic descent method for $F$ when $\eta$ is small enough,
up to additional terms coming from noise and branch switching.

\subsubsection{Assumptions}

\begin{assumption}[Bounded iterates / bounded prox outputs]
\label{ass:bounded}
There exists $R>0$ such that $\norm{y_t}\le R$ almost surely for all $t$.
Equivalently (and slightly more robustly), we may assume $\norm{\Psi(z_t)}\le R$ almost surely.
\end{assumption}

\begin{assumption}[Unbiased gradients and bounded conditional variance]
\label{ass:noise}
The stochastic gradients satisfy $\E[g_t\mid y_t]=\nabla f(y_t)$ and
\[
\E\bigl[\norm{\xi_t}^2\mid y_t\bigr] \le \sigma^2,
\qquad
\xi_t:=g_t-\nabla f(y_t).
\]
\end{assumption}

\begin{assumption}[Coordinate-wise noise variances]
\label{ass:coord-var}
For each $t$ and $j$, the conditional variance of the gate-noise coordinate is bounded:
\[
\E\bigl[\xi_{\theta,j,t}^2 \mid y_t\bigr] \le \nu_{j,t}^2.
\]
\end{assumption}

\begin{assumption}[$L$-smoothness of the loss]
\label{ass:smooth}
The function $f$ has $L$-Lipschitz gradient on a neighborhood containing $\{y_t\}$:
\[
f(y') \le f(y) + \ip{\nabla f(y)}{y'-y} + \frac{L}{2}\norm{y'-y}^2.
\]
\end{assumption}

\begin{assumption}[Proximal-PL inequality for $F$]
\label{ass:proxpl}
There exists $\mu_{\mathrm{PL}}>0$ and a reference value $F^\star$ such that for all iterates $y_t$,
\[
\frac12 \norm{G(y_t)}^2 \ge \mu_{\mathrm{PL}}\bigl(F(y_t)-F^\star\bigr).
\]
\end{assumption}

\subsection{Geometry of \textup{\textsc{Hier-Prox}}: union of cones and branchwise convex prox}

The mapping $\Psi$ in \eqref{eq:psi-def} is separable across features $j$ because $\Omega$ in \eqref{eq:Omega} is a sum of
featurewise terms and the squared norm decomposes into a sum.
Thus it suffices to study the one-feature problem and then lift the results to the full map.

\subsubsection{One feature: the feasible set is a union of two convex cones}
Fix one feature and drop index $j$.
Write the input to the proximal map as $x=(v,u)\in\R\times \R^K$,
and the proximal variable as $(b,w)\in\R\times \R^K$, corresponding to $(\theta_j, W_{j,\cdot})$.

Define the feasible set
\[
\mathcal C := \left\{(b,w)\in\R\times\R^K:\ \norm{w}_\infty \le M\abs{b}\right\}.
\]
Introduce the two sign cones
\[
\mathcal C^+ := \{(b,w): b\ge 0,\ \norm{w}_\infty \le Mb\},
\qquad
\mathcal C^- := \{(b,w): b\le 0,\ \norm{w}_\infty \le -Mb\},
\]
so that $\mathcal C = \mathcal C^+\cup \mathcal C^-$.
Each $\mathcal C^\pm$ is a closed convex cone, but their union is nonconvex.

\subsubsection{Branchwise proximal maps and firm nonexpansiveness}
Define the one-feature branch objective
\[
G_s(b,w) := \lambda s\, b + \ind_{\mathcal C^s}(b,w),
\qquad s\in\{+1,-1\}.
\]
Since $\mathcal C^s$ is convex and $b\mapsto \lambda s\,b$ is linear, $G_s$ is proper, closed, and convex.

\begin{definition}[One-feature branch optimization problem]
For $s\in\{\pm1\}$ define
\begin{equation}
\label{eq:psi-branch}
\psi^{(s)}(x) := \argmin_{(b,w)\in \R\times \R^K}\left\{G_s(b,w)+\frac12\norm{(b,w)-x}^2\right\}.
\end{equation}
\end{definition}
We will show the branch problem's optimum is well defined.

\begin{lemma}[Two-branch convexification and uniqueness on each branch]
\label{lem:branch_convex_unique}
For each $s\in\{+1,-1\}$, the set $\mathcal C^s$ is a nonempty closed convex set.
For each $(v,u)$, the branch problem $\min\psi^{(s)}$ has a unique minimizer $\psi^{(s)}(v,u)$.
\end{lemma}
\begin{proof}

\textbf{Step 1 (Closed convexity).}
Consider $s=+1$; the case $s=-1$ follows by the change of variables $(b,W)\mapsto(-b,W)$.
We can write
\[
\mathcal C^+ = \{(b,W): b\ge 0\}\cap \bigcap_{i=1}^K \{(b,W): -Mb\le W_i\le Mb\}.
\]
Each set on the right is closed and convex; intersections preserve both. Nonemptiness holds since $(0,0)$ is feasible.

\textbf{Step 2 (Strong convexity).}
The quadratic term $(b,W)\mapsto \tfrac12(v-b)^2+\tfrac12\norm{u-W}^2$ is $1$-strongly convex.
Adding the linear term $\lambda s b$ preserves strong convexity. Restricting to a closed convex set by adding $\iota_{\mathcal C^s}$ retains strong convexity on the domain.

\textbf{Step 3 (Existence and uniqueness).}
Coercivity of the quadratic ensures existence of a minimizer; strong convexity implies uniqueness.
\end{proof}

\begin{definition}[Global minimizer set]
\label{def:global_min_set}
Let $\psi_{\mathrm{set}}(v,u)$ be the (possibly set-valued) global minimizer set of \eqref{eq:psi-branch}:
\[
\psi_{\mathrm{set}}(v,u):=\argmin_{(b,W)\in\mathcal C}\ \frac12(v-b)^2+\frac12\norm{u-W}^2 + \lambda\abs b.
\]
\end{definition}

\begin{lemma}[Global minimizers lie on one of the two branches]
\label{lem:global_is_one_branch}
For any $(v,u)$,
\[\psi_{\mathrm{set}}(v,u) \subseteq \{\psi^{(+1)}(v,u),\ \psi^{(-1)}(v,u)\}.\]
Moreover, if the two branch optimal values differ, then the global minimizer is unique.
\end{lemma}
\begin{proof}
Since $\mathcal C=\mathcal C^+\cup\mathcal C^-$, the global optimum equals the minimum of the two branch optima.
On $\mathcal C^+$, $\abs b=b$ and the objective equals $G_{+1}(b,w)+\frac12\norm{(b,w)-x}^2$; on $\mathcal C^-$, $\abs b=-b$ and the objective equals $G_{-1}(b,w)+\frac12\norm{(b,w)-x}^2$.
Any global minimizer must lie in one branch and therefore must coincide with that branch's unique minimizer (Lemma~\ref{lem:branch_convex_unique}).
If branch values differ, only one branch attains the minimum, yielding uniqueness.
\end{proof}

\begin{lemma}[Branch map is a convex proximal operator]
\label{lem:branch-prox}
For each $s\in\{\pm1\}$, the map $\psi^{(s)}$ is the proximal operator of the convex function $G_s$:
\[
\psi^{(s)}(x) = \mathrm{prox}_{G_s}(x).
\]
In particular, $\psi^{(s)}$ is \emph{firmly nonexpansive} and thus $1$-Lipschitz:
\[
\norm{\psi^{(s)}(x)-\psi^{(s)}(x')}^2 \le \ip{\psi^{(s)}(x)-\psi^{(s)}(x')}{x-x'} \le \norm{x-x'}^2.
\]
\end{lemma}

\begin{proof}
Fix the sign $s$. By definition \eqref{eq:psi-branch} and lemma \ref{lem:global_is_one_branch}, $\psi^{(s)}(x)$ is exactly the minimizer of $G_s(\cdot)+\frac12\|\cdot-x\|^2$,
which is the definition of $\mathrm{prox}_{G_s}(x)$. Firm nonexpansiveness is a standard property of proximal operators
of proper closed convex functions.
\end{proof}

\subsubsection{The global one-feature map and sign selection}
The \textup{\textsc{Hier-Prox}} rule selects a sign branch based on the input $v$.
We define $\sign(0):=+1$ as a deterministic tie-breaking convention.

\begin{definition}[Global one-feature map]
\label{def:psi-global}
Define $\psi:\R\times\R^K\to\R\times\R^K$ by
\[
\psi(v,u) := \psi^{(\sign(v))}(v,u).
\]
\end{definition}

\begin{lemma}[Sign selection is globally optimal except at ties]
\label{lem:sign-optimal}
Let $x=(v,u)$ and let $(b^\star,w^\star)$ be any minimizer of the one-feature proximal problem
\[
\min_{(b,w)\in\mathcal C}\left\{\lambda\abs{b}+\frac12(b-v)^2+\frac12\norm{w-u}^2\right\}.
\]
If $v\neq 0$, then every minimizer satisfies $\sign(b^\star)=\sign(v)$.
When $v=0$, both signs can be optimal, and we select the $+$ branch by convention.
\end{lemma}

\begin{proof}
Fix any feasible pair $(b,w)\in\mathcal C$.
Because the constraint depends on $\abs{b}$, the pair $(-b,w)$ is also feasible and has the same $\lambda\abs{b}$ and
$\norm{w-u}^2$ terms. The difference in objective values is entirely in the quadratic:
\[
\frac12(b-v)^2 - \frac12(-b-v)^2
= \frac12\bigl(b^2-2bv+v^2 - (b^2+2bv+v^2)\bigr)
= -2bv.
\]
If $v>0$ and $b>0$, this difference is negative, so $(b,w)$ is strictly better than $(-b,w)$.
If $v>0$ and $b<0$, then $(-b,w)$ is strictly better.
Thus for $v>0$ any minimizer must have $b\ge 0$. An analogous argument gives $b\le 0$ when $v<0$.
When $v=0$, the two objectives are equal for $\pm b$, so ties are possible.
\end{proof}

\subsubsection{\texorpdfstring{Lifting to $d$ features}{Lifting to d features}}
Define the full mapping $\Psi$ in \eqref{eq:psi-def} by applying the one-feature map $\psi$ featurewise.
Equivalently, define the sign pattern of a point $z=(\theta,W)$ by
\[
s(z) := (\sign(\theta_1),\dots,\sign(\theta_d))\in\{\pm1\}^d,
\]
and let $\Psi^{(s)}$ denote the map obtained by applying the branch maps $\psi^{(s_j)}$ to each feature $j$.
Then
\[
\Psi(z)=\Psi^{(s(z))}(z),
\qquad
\text{and for each fixed }s,\ \Psi^{(s)}\text{ is firmly nonexpansive (hence $1$-Lipschitz).}
\]

\subsubsection{A negative result: the global map need not be Lipschitz}
The next proposition makes precise what \emph{cannot} be extended from standard proximal mappings.

\begin{proposition}[Discontinuity at the sign boundary; not globally Lipschitz]
\label{prop:not-lipschitz}
Consider the one-feature case $K=1$ with parameters $M>0$ and $\lambda\ge 0$.
Fix an input $u>0$ and define the one-feature \textup{\textsc{Hier-Prox}} map $v\mapsto \psi(v,u)=(b(v),w(v))$.
If $Mu>\lambda$, then
\[
\lim_{v\downarrow 0} b(v)= +\frac{Mu-\lambda}{1+M^2},
\qquad
\lim_{v\uparrow 0} b(v)= -\frac{Mu-\lambda}{1+M^2},
\]
while $b(0)=+\frac{Mu-\lambda}{1+M^2}$ under the convention $\sign(0)=+1$.
In particular, the map $v\mapsto \psi(v,u)$ is discontinuous at $v=0$ and therefore not Lipschitz.
\end{proposition}

\begin{proof}
With $K=1$, the constraint is $\abs{w}\le M\abs{b}$ and the one-feature prox problem is
\[
\min_{b,w}\left\{\lambda\abs{b}+\frac12(b-v)^2+\frac12(w-u)^2:\ \abs{w}\le M\abs{b}\right\}.
\]
For any fixed $b$, the minimizer over $w$ is the Euclidean projection of $u$ onto $[-M\abs{b},\,M\abs{b}]$, hence
\[
w^\star(b)=\min\{u,\,M\abs{b}\}\quad(\text{since }u>0),
\]
and the resulting reduced objective is
\[
\phi_v(b)=\lambda\abs{b}+\frac12(b-v)^2+\frac12\bigl(u-M\abs{b}\bigr)_+^2,
\]
where $(t)_+=\max\{t,0\}$. This function is convex in $b$ (sum of convex terms), hence it has a unique minimizer
except for a possible tie at $v=0$.

We first analyze $v\ge 0$ and minimize over $b\ge 0$ (Lemma~\ref{lem:sign-optimal} shows the minimizer has $b\ge 0$).
On $b\in[0,u/M]$, we have $(u-Mb)_+=u-Mb$ and
\[
\phi_v(b)=\lambda b+\frac12(b-v)^2+\frac12(u-Mb)^2.
\]
Differentiating gives
\[
\phi_v'(b)=\lambda+(b-v)-M(u-Mb)=(1+M^2)b+\lambda - v - Mu.
\]
Setting $\phi_v'(b)=0$ yields the unique critical point
\[
b_+(v)=\frac{Mu+v-\lambda}{1+M^2}.
\]
When $Mu>\lambda$ and $v\ge 0$ is sufficiently small, $b_+(v)>0$.
Moreover $b_+(v) < u/M$ holds for all $u>0$ (a direct rearrangement), so this critical point lies in the region
where the expression is valid. Since $\phi_v$ is convex, this $b_+(v)$ is the unique minimizer for $v\ge 0$ in this regime.

By symmetry of $\phi_v$ in the pair $(b,v)$ through $\abs{b}$ and $\abs{v}$ (and again Lemma~\ref{lem:sign-optimal}),
for $v\le 0$ the minimizer is $b(v)=-b_+(\abs{v})$.
Therefore, as $v\downarrow 0$,
\[
b(v)\to \frac{Mu-\lambda}{1+M^2},
\quad
\text{and as }v\uparrow 0,\quad
b(v)\to -\frac{Mu-\lambda}{1+M^2}.
\]
Under the convention $\sign(0)=+1$, $b(0)=b_+(0)=\frac{Mu-\lambda}{1+M^2}$, so the map is discontinuous at $0$.
\end{proof}

\begin{remark}[Interpretation]
The example shows that if the first-layer weight input $u$ is large, the constraint $\abs{w}\le M\abs{b}$ can force
the optimal gate magnitude $\abs{b}$ to remain bounded away from $0$, even when $v\approx 0$.
Then an infinitesimal change of sign in $v$ can flip the sign of $b$ while keeping its magnitude almost constant,
creating a jump. This is exactly the ``branch-switch'' phenomenon that our stochastic analysis must track.
\end{remark}

\subsection{Piecewise Lipschitz + jump: the instability decomposition}

We now state the key inequality that replaces global Lipschitz continuity.
The proof is short because it reduces to ``nonexpansive on each branch'' plus ``bounded jump across branches.''

\begin{remark}[How boundedness is justified in practice]
In practical deep-learning pipelines, boundedness is often enforced implicitly or explicitly:
(i) warm-start training brings parameters into a stable region before enabling \textup{\textsc{Hier-Prox}};
(ii) gradient clipping or weight decay prevents divergence;
(iii) early stopping prevents excursions to extreme regions.
For a theoretical note, we state boundedness as an assumption because it is exactly what controls the jump size.
\end{remark}

\begin{lemma}[Piecewise nonexpansive with a jump term]
\label{lem:piecewise-jump}
Let $\Psi$ be the full \textup{\textsc{Hier-Prox}} map and assume Assumption~\ref{ass:bounded}.
Then for all $x,x'\in\R^{d+dK}$,
\[
\norm{\Psi(x)-\Psi(x')}^2 \;\le\; \norm{x-x'}^2 \;+\; J^2\,\ind\{s(x)\neq s(x')\},
\qquad\text{with }J:=2R.
\]
\end{lemma}

\begin{proof}
If $s(x)=s(x')$, then $\Psi(x)=\Psi^{(s)}(x)$ and $\Psi(x')=\Psi^{(s)}(x')$ for the same sign pattern $s$.
Since each branch map is firmly nonexpansive and the full map is a product of branch maps,
$\norm{\Psi(x)-\Psi(x')} \le \norm{x-x'}$, hence the squared inequality holds with no $J$ term.

If $s(x)\neq s(x')$, then by Assumption~\ref{ass:bounded},
\[
\norm{\Psi(x)-\Psi(x')}\le \norm{\Psi(x)}+\norm{\Psi(x')}\le 2R,
\]
so $\norm{\Psi(x)-\Psi(x')}^2\le 4R^2 = J^2$.
Combining the two cases yields the claim.
\end{proof}

\begin{corollary}[Expectation form with conditional branch-switch probability]
\label{cor:expect_piecewise}
Fix a $\sigma$-field $\mathcal F$ and random variables $X,X'$ measurable w.r.t.\ $\mathcal F$.
Let $B:=\{s(X)\neq s(X')\}$ and define $q:=\Pp(B\mid \mathcal F)$.
Then
\[
\E\bigl[\norm{\Psi(X)-\Psi(X')}^2\mid \mathcal F\bigr]
\;\le\;
\E\bigl[\norm{X-X'}^2\mid \mathcal F\bigr] + J^2 q.
\]
\end{corollary}

\begin{proof}
Apply Lemma~\ref{lem:piecewise-jump} pointwise and take conditional expectation.
\end{proof}

\subsubsection{A simple signal-to-noise bound for branch switching (Chebyshev)}
We now connect $q_t$ to a familiar ML notion: \emph{margin} versus \emph{noise}.

\begin{definition}[Noise and margin for the gate coordinates]
\label{def:margin}
Let $\xi_t := g_t-\nabla f(y_t)$ be the stochastic gradient noise, and define the ``noiseless'' pre-prox point
\[
\bar z_t := y_t - \eta \nabla f(y_t),
\qquad
z_t = \bar z_t - \eta \xi_t.
\]
Write $\bar z_t=(\bar\theta_t,\bar W_t)$ and define the coordinate-wise margins
\[
m_{j,t} := \abs{(\bar\theta_t)_j}\quad (j=1,\dots,d).
\]
Define the branch mismatch event
\[
B_t := \{s(z_t)\neq s(\bar z_t)\},
\qquad
q_t(y_t):=\Pp(B_t\mid y_t).
\]
\end{definition}

\begin{lemma}[Chebyshev SNR bound for branch switches]
\label{lem:chebyshev}
Under Definition~\ref{def:margin} and Assumption~\ref{ass:coord-var},
\[
q_t(y_t)
\;\le\;
\sum_{j=1}^d \frac{\eta^2 \nu_{j,t}^2}{m_{j,t}^2},
\quad\text{with the convention } \frac{0}{0}=0.
\]
\end{lemma}

\begin{proof}
A branch switch occurs if at least one coordinate changes sign:
\[
B_t \subseteq \bigcup_{j=1}^d \Bigl\{ \sign((\bar\theta_t)_j-\eta\xi_{\theta,j,t}) \neq \sign((\bar\theta_t)_j)\Bigr\}.
\]
For a fixed $j$ with $m_{j,t}>0$, a sign change implies $\abs{\eta\xi_{\theta,j,t}} \ge m_{j,t}$.
Therefore, by the union bound and Chebyshev's inequality,
\[
q_t(y_t)
\le \sum_{j=1}^d \Pp\left(\abs{\xi_{\theta,j,t}} \ge \frac{m_{j,t}}{\eta}\,\middle|\,y_t\right)
\le \sum_{j=1}^d \frac{\E[\xi_{\theta,j,t}^2\mid y_t]}{(m_{j,t}/\eta)^2}
\le \sum_{j=1}^d \frac{\eta^2 \nu_{j,t}^2}{m_{j,t}^2}.
\]
If $m_{j,t}=0$, the bound is vacuous; we adopt the convention $0/0=0$.
\end{proof}

\begin{remark}[Margin collapse is exactly when instability can appear]
When some margin $m_{j,t}$ becomes small (feature enters/exits, or gate hovers near $0$),
even moderate noise can create frequent sign flips, making $q_t$ large.
This is where the extra sign-switch penalty can dominate.
\end{remark}

\subsection{Deterministic descent in function values (no convexity needed)}

\subsubsection{The deterministic prox-gradient map}
Define the deterministic map
\[
T(y) := \Psi\bigl(y-\eta \nabla f(y)\bigr),
\qquad
G(y):= y - T(y).
\]
The vector $G(y)$ is the \emph{prox-gradient mapping}; it vanishes exactly at fixed points of $T$.

\begin{lemma}[Prox optimality inequality]
\label{lem:prox-opt}
Let $y^+=\Psi(z)$ for some $z$. Then for all $y$,
\[
\Omega(y^+)+\frac12\norm{y^+-z}^2 \le \Omega(y)+\frac12\norm{y-z}^2.
\]
\end{lemma}

\begin{proof}
This is immediate from the definition \eqref{eq:psi-def} of $\Psi$ as an argmin.
\end{proof}

\begin{lemma}[Deterministic descent for $F$]
\label{lem:det-descent}
Under Assumption~\ref{ass:smooth}, if $\eta\le 1/L$, then for all $y$,
\[
F(T(y)) \le F(y) - \frac{1-\eta L}{2}\norm{G(y)}^2.
\]
\end{lemma}

\begin{proof}
Let $z=y-\eta\nabla f(y)$ and $y^+=T(y)=\Psi(z)$.
Apply Lemma~\ref{lem:prox-opt} with $y$ as the comparator:
\[
\Omega(y^+) + \frac12\norm{y^+-z}^2 \le \Omega(y)+\frac12\norm{y-z}^2.
\]
Expanding $z=y-\eta\nabla f(y)$ gives
\[
\norm{y^+-z}^2 = \norm{(y^+-y)+\eta\nabla f(y)}^2
= \norm{y^+-y}^2 + 2\eta \ip{\nabla f(y)}{y^+-y} + \eta^2\norm{\nabla f(y)}^2,
\]
and $\norm{y-z}^2=\eta^2\norm{\nabla f(y)}^2$.
Canceling the common $\eta^2\norm{\nabla f(y)}^2/2$ term yields
\[
\Omega(y^+) + \frac12\norm{y^+-y}^2 + \eta \ip{\nabla f(y)}{y^+-y} \le \Omega(y).
\]
By $L$-smoothness,
\[
f(y^+) \le f(y) + \ip{\nabla f(y)}{y^+-y} + \frac{L}{2}\norm{y^+-y}^2.
\]
Multiply by $\eta$ and add to the previous inequality to obtain
\[
F(y^+) \le F(y) - \frac{1-\eta L}{2}\norm{y^+-y}^2.
\]
Finally, $\norm{y^+-y}=\norm{G(y)}$ by definition, completing the proof.
\end{proof}

\begin{remark}[No convexity is used here]
Lemma~\ref{lem:det-descent} uses only $L$-smoothness of $f$ and exact optimality of the proximal subproblem.
It does \emph{not} assume convexity of $f$ or $\Omega$.
The role of proximal-PL in the next section is to turn ``descent'' into a \emph{linear rate}.
\end{remark}

\subsection{A proximal-PL condition and how to certify it locally}

\begin{remark}[What $F^\star$ represents]
If $F$ attains a minimum on the relevant region, one can take $F^\star=\min_y F(y)$.
More generally (in nonconvex settings), proximal-PL can be stated relative to a set of global minimizers
or even relative to a best-achievable value along the trajectory. For simplicity, we keep the scalar $F^\star$.
\end{remark}

\subsubsection{Local certification heuristics}
The proximal-PL condition is an assumption, but it is not arbitrary.
Below are two practically relevant routes that can justify it \emph{locally}.

\begin{itemize}[leftmargin=2.2em]
\item \textbf{Stable sign pattern + margin.}
If the sign pattern $s(y-\eta\nabla f(y))$ remains fixed on a region and the margins are bounded away from $0$,
then $\Psi$ reduces to a \emph{single convex proximal operator} $\Psi^{(s)}$ on that region.
In that case, standard arguments for composite problems on convex sets can be used to derive PL-type inequalities.
\item \textbf{Empirical local regularity.}
In deep learning, it is common that training dynamics enter a well-conditioned region after warm-start.
Proximal-PL can be viewed as a formalization of ``loss landscape is locally sharp'' for the relevant trajectory, after the hier-prox operation.
\end{itemize}

\begin{remark}[We do not require global PL]
Nothing in the stochastic analysis requires proximal-PL to hold for all $y\in\R^{d+dK}$.
It only needs to hold along (or near) the iterates that the algorithm visits.
\end{remark}

\subsection{Stochastic recursion: noise + sign-switch penalty}

\subsubsection{A conditional descent inequality}
\begin{lemma}[Stochastic descent with a noise cross term]
\label{lem:stoch1}
Assume Assumptions~\ref{ass:smooth} and \ref{ass:noise}. If $\eta\le 1/L$, then almost surely
\[
F(y_{t+1})
\le
F(y_t)
-\frac{1-\eta L}{2}\norm{y_{t+1}-y_t}^2
-\eta \ip{\xi_t}{y_{t+1}-y_t}.
\]
\end{lemma}

\begin{proof}
Apply Lemma~\ref{lem:prox-opt} to $y_{t+1}=\Psi(z_t)$ with comparator $y=y_t$ and $z_t=y_t-\eta g_t$:
\[
\Omega(y_{t+1})+\frac12\norm{y_{t+1}-z_t}^2 \le \Omega(y_t)+\frac12\norm{y_t-z_t}^2.
\]
Expand as in Lemma~\ref{lem:det-descent}:
\[
\Omega(y_{t+1}) + \frac12\norm{y_{t+1}-y_t}^2 + \eta \ip{g_t}{y_{t+1}-y_t} \le \Omega(y_t).
\]
By $L$-smoothness,
\[
\eta f(y_{t+1}) \le \eta f(y_t) + \eta \ip{\nabla f(y_t)}{y_{t+1}-y_t} + \frac{\eta L}{2}\norm{y_{t+1}-y_t}^2.
\]
Add the two inequalities, use $F=\eta f+\Omega$, and substitute $g_t=\nabla f(y_t)+\xi_t$ to get the claim.
\end{proof}

\begin{lemma}[Young inequality removes the cross term]
\label{lem:young}
Fix any $\alpha>0$. Then
\[
-\eta \ip{\xi_t}{y_{t+1}-y_t}
\le
\frac{\eta^2}{2\alpha}\norm{\xi_t}^2 + \frac{\alpha}{2}\norm{y_{t+1}-y_t}^2.
\]
\end{lemma}

\begin{proof}
This is the inequality $-\ip{a}{b}\le \frac1{2\alpha}\norm{a}^2+\frac{\alpha}{2}\norm{b}^2$
with $a=\eta \xi_t$ and $b=y_{t+1}-y_t$.
\end{proof}

\begin{corollary}[Conditional descent in terms of $\E\norm{y_{t+1}-y_t}^2$]
\label{cor:stoch-descent}
Under Assumptions~\ref{ass:smooth} and \ref{ass:noise}, if $\eta\le 1/L$ and $\alpha=(1-\eta L)/2$, then
\[
\E\bigl[F(y_{t+1})\mid y_t\bigr]
\le
F(y_t)
-\frac{1-\eta L}{4}\E\bigl[\norm{y_{t+1}-y_t}^2\mid y_t\bigr]
+\frac{\eta^2}{1-\eta L}\E\bigl[\norm{\xi_t}^2\mid y_t\bigr].
\]
\end{corollary}

\begin{proof}
Combine Lemma~\ref{lem:stoch1} and Lemma~\ref{lem:young} and take conditional expectation.
\end{proof}

\subsubsection{\texorpdfstring{Relating $\E\norm{y_{t+1}-y_t}^2$ to the deterministic prox-gradient mapping}{Relating E||y(t+1)-y(t)||² to the deterministic prox-gradient mapping}}
Let $\bar z_t=y_t-\eta\nabla f(y_t)$ and $\bar y_{t+1}:=T(y_t)=\Psi(\bar z_t)$.
Then $G(y_t)=y_t-\bar y_{t+1}$.

\begin{lemma}[Lower bound with a branch-switch penalty]
\label{lem:lowerbound}
Assume Assumptions~\ref{ass:bounded} and \ref{ass:noise}. Then
\[
\E\bigl[\norm{y_{t+1}-y_t}^2\mid y_t\bigr]
\ge
\frac12\norm{G(y_t)}^2 - \eta^2\sigma^2 - J^2 q_t(y_t),
\qquad J=2R,
\]
where $q_t(y_t)=\Pp(s(z_t)\neq s(\bar z_t)\mid y_t)$.
\end{lemma}

\begin{proof}
Write
\[
y_{t+1}-y_t = (\bar y_{t+1}-y_t) + (y_{t+1}-\bar y_{t+1}) = -G(y_t) + (y_{t+1}-\bar y_{t+1}).
\]
Using the inequality $\norm{a+b}^2 \ge \frac12\norm{a}^2 - \norm{b}^2$ (valid for all vectors $a,b$),
we obtain
\[
\norm{y_{t+1}-y_t}^2
\ge
\frac12\norm{G(y_t)}^2 - \norm{y_{t+1}-\bar y_{t+1}}^2.
\]
Now apply Corollary~\ref{cor:expect_piecewise} with $\mathcal F=\sigma(y_t)$,
$X=z_t$ and $X'=\bar z_t$:
\[
\E\bigl[\norm{y_{t+1}-\bar y_{t+1}}^2\mid y_t\bigr]
=
\E\bigl[\norm{\Psi(z_t)-\Psi(\bar z_t)}^2\mid y_t\bigr]
\le
\E\bigl[\norm{z_t-\bar z_t}^2\mid y_t\bigr] + J^2 q_t(y_t).
\]
Finally, $z_t-\bar z_t = -\eta\xi_t$, so
$\E[\norm{z_t-\bar z_t}^2\mid y_t]=\eta^2\E[\norm{\xi_t}^2\mid y_t]\le \eta^2\sigma^2$.
Substituting into the previous inequality completes the proof.
\end{proof}

\subsubsection{The master recursion}
\begin{theorem}[One-step recursion: noise + sign-switch penalty]
\label{thm:master}
Assume Assumptions~\ref{ass:smooth}, \ref{ass:bounded}, \ref{ass:noise}, and \ref{ass:proxpl}, and let $\eta\le 1/L$.
Define $\Delta_t:=\E[F(y_t)-F^\star]$.
Then
\[
\E[F(y_{t+1})\mid y_t] - F^\star
\le
\Bigl(1-\frac{1-\eta L}{4}\mu_{\mathrm{PL}}\Bigr)\,(F(y_t)-F^\star)
\;+\;
C_{\mathrm{noise}}\,\eta^2\sigma^2
\;+\;
C_{\mathrm{sign}}\,J^2 q_t(y_t),
\]
where $J=2R$ and
\[
C_{\mathrm{noise}} = \frac{1-\eta L}{4} + \frac{1}{1-\eta L},
\qquad
C_{\mathrm{sign}} = \frac{1-\eta L}{4}.
\]
\end{theorem}

\begin{proof}
Start from Corollary~\ref{cor:stoch-descent} and plug in Lemma~\ref{lem:lowerbound}:
\begin{align*}
\E[F(y_{t+1})\mid y_t]
&\le
F(y_t)
-\frac{1-\eta L}{4}\Bigl(\frac12\norm{G(y_t)}^2 - \eta^2\sigma^2 - J^2 q_t(y_t)\Bigr)
+\frac{\eta^2}{1-\eta L}\sigma^2 \\
&=
F(y_t)
-\frac{1-\eta L}{8}\norm{G(y_t)}^2
+\Bigl(\frac{1-\eta L}{4}+\frac{1}{1-\eta L}\Bigr)\eta^2\sigma^2
+\frac{1-\eta L}{4}J^2 q_t(y_t).
\end{align*}
Now apply the proximal-PL inequality (Assumption~\ref{ass:proxpl}) at $y_t$:
$\frac12\norm{G(y_t)}^2 \ge \mu_{\mathrm{PL}}(F(y_t)-F^\star)$, i.e.\ $\norm{G(y_t)}^2\ge 2\mu_{\mathrm{PL}}(F(y_t)-F^\star)$.
Substitute this lower bound into the negative term to obtain the stated contraction.
\end{proof}

\begin{corollary}[Unconditional recursion]
\label{cor:uncond}
Under the assumptions of Theorem~\ref{thm:master},
\[
\Delta_{t+1}
\le
\Bigl(1-\frac{1-\eta L}{4}\mu_{\mathrm{PL}}\Bigr)\Delta_t
+
C_{\mathrm{noise}}\,\eta^2\sigma^2
+
C_{\mathrm{sign}}\,J^2\,\E[q_t(y_t)].
\]
\end{corollary}

\begin{proof}
Take expectation of Theorem~\ref{thm:master} over $y_t$.
\end{proof}

\subsection{Summary theorem and interpretations}

\subsubsection{A clean bound under a margin condition}
To obtain a fully explicit bound, we now assume a margin lower bound.
This is a standard ``signal-to-noise'' assumption: the deterministic pre-prox gates do not hover near zero.

\begin{assumption}[Uniform margin lower bound]
\label{ass:margin}
There exist constants $\underline m_1,\dots,\underline m_d>0$ such that for all $t$,
\[
m_{j,t}=\abs{(\bar\theta_t)_j} \ge \underline m_j
\quad\text{almost surely.}
\]
\end{assumption}

\begin{theorem}[Linear convergence to a ``noise + sign'' neighborhood]
\label{thm:main}
Assume Assumptions~\ref{ass:smooth}, \ref{ass:bounded}, \ref{ass:noise}, \ref{ass:proxpl},
\ref{ass:coord-var}, and \ref{ass:margin}, and let $\eta\le 1/L$.
Define
\[
\rho := 1-\frac{1-\eta L}{4}\mu_{\mathrm{PL}}\in(0,1).
\]
Then for all $t\ge 0$,
\[
\Delta_t
\le
\rho^t \Delta_0
+
\frac{1}{1-\rho}\left(
C_{\mathrm{noise}}\,\eta^2\sigma^2
+
C_{\mathrm{sign}}\,J^2\,\eta^2 \sum_{j=1}^d \frac{\sup_t \nu_{j,t}^2}{\underline m_j^2}
\right),
\qquad J=2R.
\]
\end{theorem}

\begin{proof}
By Lemma~\ref{lem:chebyshev} and Assumption~\ref{ass:margin},
\[
q_t(y_t)\le \sum_{j=1}^d \frac{\eta^2 \nu_{j,t}^2}{m_{j,t}^2}\le \eta^2 \sum_{j=1}^d \frac{\nu_{j,t}^2}{\underline m_j^2}.
\]
Taking expectation yields the same bound for $\E[q_t(y_t)]$.
Plug this into Corollary~\ref{cor:uncond} to obtain a linear recursion
$\Delta_{t+1}\le \rho \Delta_t + \beta\eta^2$, with $\beta$ equal to the parenthesized constant.
Unrolling the recursion gives the stated formula.
\end{proof}

\subsubsection{Interpretations}
\begin{remark}[When \textup{\textsc{Hier-Prox}} behaves like standard Prox-SGD]
If the margin lower bound holds (gates are not near zero),
then the branch-switch probability scales like $\eta^2$ and is typically small.
In that regime, the last term in Theorem~\ref{thm:main} is of the same order as the usual SGD noise term.
\end{remark}

\begin{remark}[When instability is expected]
If some margins collapse (gates hover near $0$), the bound in Lemma~\ref{lem:chebyshev} becomes large,
and sign flips can occur frequently. Theorem~\ref{thm:master} then predicts a large additive penalty,
which can dominate the dynamics even when $\sigma^2$ is moderate.
This matches the empirical observation that feature enter/exit events are periods of instability.
\end{remark}

\subsubsection{Variance transfer viewpoint}
The piecewise nonexpansiveness inequality also yields a ``variance transfer'' message.

\begin{proposition}[Variance transfer: within-branch nonexpansive, across-branch penalty]
\label{prop:var-transfer}
Let $X$ be a random pre-prox input and let $\bar X$ be an independent copy with the \emph{same sign pattern} (i.e.\ $s(X)=s(\bar X)$ almost surely).
Then $\Var(\Psi(X))\le \Var(X)$ (nonexpansive).
In general, for any $X$,
\[
\E\norm{\Psi(X)-\E\Psi(X)}^2 \le \E\norm{X-\E X}^2 + J^2\,\Pp(s(X)\neq s(\E X)),
\]
so variance can increase only through a branch-switch-type event.
\end{proposition}

\begin{proof}
The first statement follows from firm nonexpansiveness on a fixed branch.
The second statement is a direct application of Lemma~\ref{lem:piecewise-jump} with $x'= \E X$ and Jensen's inequality.
\end{proof}

\subsection{Discussion: why variance reduction is a promising direction}

The master recursion (Theorem~\ref{thm:master}) isolates two drivers of the final neighborhood size:
\begin{enumerate}[leftmargin=2.2em]
\item the classical noise term, scaling like $\eta^2\sigma^2$;
\item the sign-switch term, scaling like $J^2\,\E[q_t(y_t)]$, which itself can be controlled by a margin-to-noise ratio.
\end{enumerate}
This suggests a practical research direction: \emph{stabilize \textup{\textsc{Hier-Prox}} by reducing gradient noise},
thereby reducing both $\sigma^2$ and the switch probability.
While full variance-reduction methods (SVRG/SAGA) can be costly in deep learning,
lighter-weight options (larger batch size, gradient accumulation, EMA smoothing, etc.) may already reduce sign flips noticeably.

\paragraph{Takeaway.}
LassoNet's hierarchical feature-selection mechanism introduces a fundamental discrete instability channel. Our analysis proves that controlling this channel---particularly the variance induced by gate sign transitions---is critical for stable convergence. This theoretical insight directly motivates the exponential moving average (EMA) stabilization introduced in our Seq-Hier-Prox-EMA algorithm, which effectively dampens the branch-switch probability $q_t$ by reducing the effective gradient noise.

\subsection{LassoFlexNet Training}
\label{app:nonlinear_lasso_train}

\begin{algorithm}
\caption{Update Algorithm for LassoFlexNet}
\label{alg:nonlinear_lasso_train}
\begin{algorithmic}[1]
\State \textbf{Input:} training dataset $X \in \mathbb{R}^{n \times d}$, training labels $Y$, linear skip module $\mathrm{LS}_{\beta}$, per-feature embedding ResNet $\mathrm{NN}_\phi(\cdot)$, main neural network $\mathrm{NN}_W(\cdot)$ (e.g. Mixer), 
number of epochs $B$, 
hierarchy constrain multiplier $M$, 
lambda path $\lambda_{\text{Seq}} = [\lambda_0, \cdots, \lambda_N]$, 
learning rate $\alpha$,
\State Initialize and train the joint model on the loss $L(\beta, \phi, W)$
\State Initialize the penalty, $\lambda = \lambda_0$, and the number of active features, $k = d$
\While{$k > 0$}
    \State Get $\lambda \in \lambda_{\text{Seq}}$ iteratively
    \For{$b = 1$ to $B$}
        \State Compute gradient of the loss w.r.t to $(\beta, \phi, W)$ using back-propagation
        \State Update $\beta \leftarrow \beta - \alpha \nabla_\beta L$, \quad
        $\phi \leftarrow \phi - \alpha \nabla_\phi L$, \quad
        and $W \leftarrow W - \alpha \nabla_W L$
        \State Update $(\tilde{\beta}, \tilde{\phi}, \tilde{W}^{(1)}) \leftarrow \text{Coordinate Hier-Prox}(\beta, \phi, W^{(1)}, \alpha\lambda, M)$
    \EndFor
    \State Update $k$ to be the number of non-zero coordinates of $\beta$
\EndWhile
\end{algorithmic}
\end{algorithm}

\subsection{Proximal Local Optimality: Coordinate Hierarchical Proximal Gradient Derivation}
\label{app:co_hier_prox_proof}
Our derivation is based on Proposition 1 in \cite{LassoNet}, which we provide here for completeness and reference:

\paragraph{Proposition 1 (Joint optimization):} Fix $v \in \mathbb{R}^k$ and $U \in \mathbb{R}^K$. (Note: the two integers $k, K$ can be different) Let us consider the problem:
$$
\begin{array}{ll}
\underset{(b, W)}{\operatorname{minimize}} & \frac{1}{2}\left(\|v-b\|_2^2+\|U-W\|_2^2\right)+\lambda\|b\|_2+\bar{\lambda}\|W\|_1 \\
\text { subject to } & \|W\|_{\infty} \leq M \cdot\|b\|_2 .
\end{array}
$$
\cite{LassoNet} derives the sufficient and necessary condition for characterizing the global optimum $\left(b^*, W^*\right)$ of the above optimization problem:

1. Let us order the coordinates $\left\{\left|U_i\right|\right\}_{i \in[K]}$ in decreasing order
$$
\left|U_{(1)}\right| \geq\left|U_{(2)}\right| \geq \ldots \geq\left|U_{(K)}\right| .
$$
and define for each $s \in[K]=\{0,1, \ldots, K\}$ the value $b_s$ by
$$
b_s=\frac{1}{1+s M^2}\left(1-\frac{a_s}{\|v\|_2}\right)_{+} v \text { where } a_s=\lambda-M \sum_{i=1}^s\left(\left|U_{(i)}\right|-\bar{\lambda}\right) .
$$
Then $b^*=b_{s^*}$ where $s^* \in[K]$ is the unique $s \in[K]$ such that
$$
M\left\|b_s\right\|_2 \in\left[\mathcal{S}_{\bar{\lambda}}\left(\left|U_{(s+1)}\right|\right), \mathcal{S}_{\bar{\lambda}}\left(\left|U_{(s)}\right|\right)\right) .
$$
By convention $\mathcal{S}_{\bar{\lambda}}\left(\left|U_{(K+1)}\right|\right)=0$ and $\mathcal{S}_{\bar{\lambda}}\left(\left|U_{(0)}\right|\right)=\infty$.

2. The $W^*$ must satisfy
$$
W^*=\operatorname{sign}(U) \min \left\{M\left\|b^*\right\|_2, \mathcal{S}_{\bar{\lambda}}(|U|)\right\} .
$$

\paragraph{Sequential Optimization:} In contrast to the joint optimization of Hier-Prox, we propose a sequential (coordinate-wise) proximal update, alternating between $b$ and $W$. By decoupling the variables, we replace the complex joint projection with two simpler, closed-form proximal steps. While this sequential approach yields different training dynamics and fixed points compared to joint Hier-Prox, it is highly effective in practice. We validate the empirical superiority of this Seq-Hier-Prox approach in our experiments.

For our application, we care about only the standard Lasso, instead of group Lasso, i.e. $b$ is a scalar. Our derivations will follow this case accordingly.

\paragraph{Linear skip layer step for $b$:} We first characterize the following global optimum with respect to $b$:
$$
\begin{aligned}
b^* &= \underset{b}{\operatorname{argmin}} \frac{1}{2}\left( (v-b)^2+\|U-W\|_2^2\right)+\lambda |b| + \bar{\lambda}\|W\|_1 \\
    &= \underset{b}{\operatorname{argmin}} \frac{1}{2}(v-b)^2+\lambda|b| \\
    &= \mathcal{S}_{\lambda}(v)
\end{aligned}
$$
This is the standard proximal mapping for Lasso regression \cite{parikh2014proximal}, whose optimum is given by the soft-thresholding operator $\mathcal{S}_{\lambda}(v)$.

\paragraph{Neural first layer step for $W$:} Next, we solve for the global optimum with respect to $W$:
$$
\begin{array}{ll}
\underset{W}{\operatorname{minimize}} & \frac{1}{2}\left(|v-b|_2^2+\|U-W\|_2^2\right)+\lambda|b|+\bar{\lambda}\|W\|_1 \\
\text { subject to } & \|W\|_{\infty} \leq M \cdot |b| 
\end{array}
$$

Since we are not solving $(b, W)$ jointly:
$$
\begin{array}{ll}
\underset{W}{\operatorname{minimize}} & \frac{1}{2}\left(\|U-W\|_2^2\right) + \bar{\lambda}\|W\|_1 \\
\text { subject to } & \|W\|_{\infty} \leq M \cdot |b| 
\end{array}
$$

\paragraph{Proposition 2 (Sequential or Iterative Optimization):} Fix $v \in \mathbb{R}^k$ and $U \in \mathbb{R}^K$ (Note: the two integers $k, K$ can be different). The solutions to the following sequential optimization problem are solved sequentially as follows:

1. $b^*$ is solved by standard proximal gradient method:
$$
\begin{aligned}
b^* 
& = \underset{b}{\operatorname{argmin}} \frac{1}{2}\left( (v-b)^2+\|U-W\|_2^2\right)+\lambda |b| + \bar{\lambda}\|W\|_1 \\
& = \underset{b}{\operatorname{argmin}} \frac{1}{2}(v-b)^2+\lambda|b| \\
& = \mathcal{S}_{\lambda}(v)
\end{aligned}
$$ where $\mathcal{S}_{\lambda}$ is the soft-thresholding operator.

2. $W^*$ is solved by (partial or coordinate) Hierarchical proximal gradient method:
$$
\begin{aligned}
W^* 
& = \underset{W \text{ s.t. } \|W\|_{\infty} \leq M \cdot |b^*|}{\operatorname{argmin}} \frac{1}{2}\left((v-b^*)^2+\|U-W\|_2^2\right)+\lambda|b^*|+\bar{\lambda}\|W\|_1 \\
& = \underset{W \text{ s.t. } \|W\|_{\infty} \leq M \cdot |b^*|}{\operatorname{argmin}} \frac{1}{2}\left(\|U-W\|_2^2\right)+\bar{\lambda}\|W\|_1 \\
& = \operatorname{sign}(U) \min \left\{M|b^*|, \mathcal{S}_{\bar{\lambda}}(|U|)\right\}
\end{aligned}
$$ where $\mathcal{S}_{\bar{\lambda}}$ is the soft-thresholding operator.

\begin{proof}
    The first optimization problem is well known and is addressed by the standard proximal gradient and we omit the details. Readers can refer to \cite{boyd2004convex}.
    
    The derivation technique for the second problem is identical to the one in LassoNet. The key observation is their derivation for $W$ does not depend on $b$ being jointly optimized. We provide here for clarity and completeness.

We start by proving the claim below: for some $w \in \mathbb{R}_{+}$
$$
\text { Claim : } W^*=\operatorname{sign}(U) \min \left\{M \cdot |b|, \mathcal{S}_{\bar{\lambda}}(|U|)\right\} \text {. }
$$

Denote $w=M\left |b^*\right|$, where $b^*$ is solved in prior step above, as a constant for current step. By definition, $W^*$ is the minimum of the below optimization problem:
$$
\begin{array}{ll}
\underset{W}{\operatorname{minimize}} & \frac{1}{2}\|U-W\|_F^2+\bar{\lambda}\|W\|_1 \\
\text { subject to } & \|W\|_{\infty} \leq M \left|b^*\right|=w .
\end{array}
$$

The key observation is that the constraint $\|W\|_{\infty} \leq M \left|b^*\right|$ can be written as separate inequality constraints where each constraint is described by a convex function:
$$\|W\|_{\infty} \leq M \left|b^*\right|$$
This is equivalent to:
$$ |w_k| \leq M \left|b^*\right| \text{ for all k} $$
where $W = [w_1, \cdots, w_K]$ and each $w_k$ is a scalar.

Since $\frac{1}{2}\|U-W\|_F^2+\bar{\lambda}\|W\|_1$ is also a convex function, the strong duality holds since Slater's condition holds (Section 5.2.3 in \cite{boyd2004convex}). Thereby, we know for some dual variable $s \in \mathbb{R}_{+}^K, W^*$ minimizes the Lagrangian function below:
$$
W^*=\underset{W \in \mathbb{R}^K}{\operatorname{argmin}} \frac{1}{2}\|U-W\|_2^2+\sum_{j=1}^K s_j\left|W_j\right|+\bar{\lambda}\|W\|_1 .
$$

Next, let's take subgradient and get that $W^*$ needs to satisfy,
$$
\begin{aligned}
& W_j^*-U_j+\left(\bar{\lambda}+s_j\right) v_j^*=0 \text { for some } v_j^* \in \partial\left(\left|W_j^*\right|\right) . \\
& s_j\left(\left|W_j^*\right|-w\right)=0 . \\
& s_j \geq 0, \text { and }\left|W_j^*\right| \leq w .
\end{aligned}
$$

Now we divide our discussion into two cases:

1. $s_j=0$. The KKT condition shows $U_j=W_j^*+\bar{\lambda} v_j^*$ for some $v_j^* \in \partial\left(\left|W_j^*\right|\right)$. This implies that $W_j^*=\mathcal{S}_{\bar{\lambda}}\left(U_j\right)$, which is possible if and only if $\left|\mathcal{S}_{\bar{\lambda}}\left(U_j\right)\right| \leq w$.

2. $s_j>0$. The KKT condition gives $\left|W_j^*\right|=w$. Since $U_j=W_j^*+\left(\bar{\lambda}+s_j\right) v_j^*$ for $v_j^* \in \partial\left(\left|W_j^*\right|\right)$, it implies $\operatorname{sign}\left(v_j^*\right)=\operatorname{sign}\left(W_j^*\right)=\operatorname{sign}\left(U_j\right)$. Hence $W_j^*=\operatorname{sign}\left(U_j\right) w$. Note if $w \neq 0$, then we must have $v_j^*=\operatorname{sign}\left(W_j^*\right)=\operatorname{sign}\left(U_j\right)$. Thus, having some $s_j>0$ with $U_j=W_j^*+\left(\bar{\lambda}+s_j\right) v_j^*$ is equivalent to $\left|\mathcal{S}_{\bar{\lambda}}\left(U_j\right)\right|>w$.

Summarizing the above discussion, we see that $W^*$ must satisfy
$$
W_j^*=\operatorname{sign}\left(U_j\right) \min \left\{w, \mathcal{S}_{\bar{\lambda}}\left(\left|U_j\right|\right)\right\}
$$
\end{proof}

\section{Appendix: Convergence of the Convex Hier-Prox Relaxation to a Smaller Noise Neighborhood}
\label{app:sec:soft-hier-prox}

\subsection{Background}

LassoNet's original hierarchical proximal operator couples a feature-gate $\theta_j$ and the first-layer weights $W_{j,\cdot}\in\R^K$
through the hard constraint
\[
\norm{W_{j,\cdot}}_\infty \le M\abs{\theta_j},
\]
and computes a per-feature prox by sorting the coordinates of $W_{j,\cdot}$.
This induces a \emph{piecewise} proximal map: under stochastic gradients, the update can jump when the active-set/sign pattern changes,
and a function-value recursion typically contains an additional ``noise $\,+$ sign-switch'' instability term.
This is due to the nonconvex nature of the problem: union of two convex cones is not convex.

\medskip
In this section we propose a \emph{convex, $M$-scaled, coordinate-wise (sequential) relaxation} that keeps the intended meaning of the constraint parameter $M$
(the scale linking the gate and the first-layer row), while yielding a globally well-behaved proximal mapping (single-valued, firmly nonexpansive).
The resulting stochastic recursion features \emph{noise only} (no sign-switch penalty).
Our goal is to use the calculation to propose more stable algorithmic variants.

\paragraph{Key modeling choice (explicit $M$ via scaled block).}
For each feature $j$, define the scaled block
\[
x_j := \begin{bmatrix}\theta_j\\ W_{j,\cdot}/M\end{bmatrix}\in\R^{K+1},
\qquad (M>0).
\]
We replace the hard hierarchy by the convex surrogate
\[
\Omega(\theta,W)
:= \alpha\|\theta\|_1 \;+\; \lambda_w\|W\|_1 \;+\; \beta\sum_{j=1}^d \|x_j\|_2,
\qquad (\alpha,\lambda_w,\beta\ge 0).
\]
To retain an iterative \emph{two closed-form proximal implementation}, we apply the proximal step in the scaled coordinate
$\widetilde W:=W/M$ (equivalently, a diagonal preconditioner in the original variables).
This yields a per-feature update that is literally ``$\ell_1$-soft-threshold on $\theta_j$'' followed by
``group $\ell_2$ shrink on $(\theta_j,\widetilde W_{j,\cdot})$'', in closed form.

We provide a proof of a master function-value recursion under a proximal-PL condition,
matching Appendix \ref{app:sec:hier-prox-converges}.

\subsection{Setup: scaled coordinates, stochastic update, and Lyapunov}

\subsubsection{\texorpdfstring{Original variables and $M$-scaled coordinates}{Original variables and M-scaled coordinates}}

Let $\theta\in\R^d$ denote feature gates and let $W\in\R^{d\times K}$ be first-layer weights (row $j$ is $W_{j,\cdot}\in\R^K$).
Let $f_{\mathrm{orig}}(\theta,W)$ be the (possibly nonconvex) smooth training loss as a function of this block.\footnote{Other network parameters are treated as frozen.}

Fix $M>0$ and define the scaled first-layer weights
\[
\widetilde W := W/M.
\]
Define the scaled loss
\begin{equation}
\label{eq:ftilde}
f(\theta,\widetilde W)\;:=\; f_{\mathrm{orig}}(\theta, M\widetilde W).
\end{equation}
We work in the scaled variable
\[
x := (\theta,\widetilde W)\in \R^d\times\R^{d\times K}.
\]
After computing $\widetilde W$, we map back to original weights by $W = M\widetilde W$.

\begin{remark}[How gradients rescale]
By the chain rule,
$\nabla_{\widetilde W} f(\theta,\widetilde W)=M\,\nabla_W f_{\mathrm{orig}}(\theta, M\widetilde W)$.
Thus implementing SGD on $(\theta,\widetilde W)$ corresponds to using a diagonal preconditioner on the $W$-block in the original variables.
This is precisely the mechanism that makes the $M$-scaled prox step admit a clean closed form.
\end{remark}

\subsubsection{Stochastic proximal-gradient update in scaled coordinates}

Training uses a stochastic gradient $g_t$ (e.g.\ a minibatch gradient) for the scaled loss $f$ followed by a proximal map:
\begin{equation}
\label{relaxed:eq:update}
z_t = x_t - \eta g_t,\qquad x_{t+1} = \Psi(z_t),
\end{equation}
where $\eta>0$ is a constant step size and
\begin{equation}
\label{relaxed:eq:psi-def}
\Psi(z)\in \argmin_{x}\left\{\Omega(x)+\half\norm{x-z}^2\right\}.
\end{equation}

\begin{remark}[Equivalent view as a preconditioned proximal step in original variables]
Writing $x=(\theta,\widetilde W)$ and $W=M\widetilde W$, the scaled proximal problem
\[
\Psi(z)\in\argmin_{x}\left\{\Omega(x)+\half\|x-z\|^2\right\}
\]
is equivalent (under the change of variables $\widetilde W=W/M$) to a \emph{preconditioned} proximal step in the original variables:
\[
(\theta^+,W^+)\in\argmin_{\theta,W}\left\{\Omega(\theta,W)\;+\;\half\|\theta-v\|_2^2\;+\;\frac{1}{2M^2}\|W-U\|_F^2\right\},
\]
where $(v,U)$ is the corresponding original-coordinate input to the prox.
Thus the algorithm uses the natural $M$-dependent metric on the $W$-block; this is exactly what preserves a two-closed-form implementation.
\end{remark}

\subsubsection{The Lyapunov function}
Following the existing value-function analysis, define
\begin{equation}
\label{relaxed:eq:Fdef}
F(x) := \eta f(x) + \Omega(x).
\end{equation}
The recursion we prove is in $F(x_t)$.

\subsection{\texorpdfstring{The $M$-scaled convex surrogate and the two-step coordinate-wise prox}{The M-scaled convex surrogate and the two-step coordinate-wise prox}}

\subsubsection{\texorpdfstring{Penalty with explicit $M$}{Penalty with explicit M}}

\begin{definition}[$M$-scaled sparse-group hierarchical surrogate]
\label{def:OmegaSG}
Fix parameters $\alpha\ge 0$, $\lambda_w\ge 0$, and $\beta\ge 0$, and fix $M>0$.
Define $\Omega:\R^{d+dK}\to\R$ (in \emph{scaled coordinates} $x=(\theta,\widetilde W)$) by
\begin{equation}
\label{eq:OmegaSG}
\boxed{
\Omega(\theta,\widetilde W)
:= \alpha \norm{\theta}_1
\;+\; \lambda_w M \norm{\widetilde W}_1
\;+\; \beta \sum_{j=1}^d \left\|
\begin{bmatrix}
\theta_j\\ \widetilde W_{j,\cdot}
\end{bmatrix}
\right\|_2.
}
\end{equation}
Equivalently, in original variables $(\theta,W)$,
\[
\Omega(\theta,W)=\alpha\|\theta\|_1+\lambda_w\|W\|_1+\beta\sum_{j=1}^d\left\|\begin{bmatrix}\theta_j\\ W_{j,\cdot}/M\end{bmatrix}\right\|_2.
\]
\end{definition}

\begin{remark}[Prox parameter on the scaled weights]
Because $\|W\|_1 = M\|\widetilde W\|_1$, the $\ell_1$ coefficient acting on $\widetilde W$ in the scaled-coordinate prox is aligned with the original parameter $\lambda_w$ via the shorthand:
\[
\bar\lambda := \lambda_w M.
\]
This is the parameter that appears in Proposition~\ref{prop:local-opt-M}.
\end{remark}

\begin{remark}[Interpretation]
The group term $\sum_j \|(\theta_j,\widetilde W_{j,\cdot})\|_2$ encourages \emph{feature-level coupling}:
if $\theta_j$ is pushed towards zero, then (in the same group) $\widetilde W_{j,\cdot}$ is also shrunk.
The parameter $M$ calibrates what it means for the row $W_{j,\cdot}$ to be ``large'' relative to the gate $\theta_j$.
\end{remark}

\subsubsection{Proximal mapping facts (convex case)}

\begin{lemma}[Firm nonexpansiveness of a convex prox]
\label{lem:firm}
If $\Omega$ is proper, closed, and convex, then the proximal mapping $\Psi$ in \eqref{relaxed:eq:psi-def} is single-valued and firmly nonexpansive:
\[
\norm{\Psi(z)-\Psi(z')}^2 \le \ip{\Psi(z)-\Psi(z')}{z-z'}
\le \norm{z-z'}^2,\qquad \forall z,z'.
\]
In particular, $\Psi$ is $1$-Lipschitz.
\end{lemma}

\begin{remark}
Lemma~\ref{lem:firm} is the key structural simplification compared to the original Hier-Prox:
here $\Omega$ is convex, hence $\Psi$ is globally well-behaved (no branch jumps).
\end{remark}

\subsubsection{Per-feature prox and a closed-form optimality proposition}

Because $\Omega$ in \eqref{eq:OmegaSG} is separable across features \citep{boyd2004convex}, the proximal problem \eqref{relaxed:eq:psi-def}
decomposes into $d$ independent $(K+1)$-dimensional subproblems.

\begin{definition}[Soft-thresholding]
For $\tau\ge 0$, define $\soft_\tau:\R\to\R$ by
\[
\soft_\tau(u) := \sign(u)\,(\abs{u}-\tau)_+,
\qquad (a)_+ := \max\{a,0\}.
\]
For vectors/matrices, apply $\soft_\tau$ entrywise.
\end{definition}

\begin{proposition}[Closed form for the per-feature prox with explicit $M$ and $\ell_1$ on $W$]
\label{prop:local-opt-M}
Fix $v\in\R$ and a row vector $\widetilde U\in\R^K$ (the scaled-row input), and parameters $\alpha\ge 0$, $\bar\lambda\ge 0$, and $\beta\ge 0$.
Consider the convex proximal subproblem in \emph{scaled coordinates}:
\begin{equation}
\label{eq:block-prox}
\min_{b\in\R,\ \widetilde w\in\R^K}\ 
\frac{1}{2}(b-v)^2+\frac{1}{2}\|\widetilde w-\widetilde U\|_2^2+\alpha|b|+\bar\lambda\|\widetilde w\|_1+\beta\left\|\begin{bmatrix}b\\ \widetilde w\end{bmatrix}\right\|_2.
\end{equation}
Then the unique minimizer $(b^\star,\widetilde w^\star)$ exists and is given by the following two-step closed form:
\begin{enumerate}[leftmargin=2em]
\item \textbf{(Coordinate soft-threshold)} Define
\[
\tilde b := \soft_{\alpha}(v),\qquad \tilde w := \soft_{\bar\lambda}(\widetilde U)\quad\text{(coordinatewise)}.
\]
\item \textbf{(Group $\ell_2$ shrink)} Let
\[
r:=\sqrt{\tilde b^2+\norm{\tilde w}_2^2},\qquad
\kappa := \left(1-\frac{\beta}{r}\right)_+ \quad (\kappa=0 \text{ if }r=0).
\]
Then
\[
\boxed{
(b^\star,\widetilde w^\star)=\kappa\,(\tilde b,\tilde w).
}
\]
\end{enumerate}
Equivalently,
\[
(b^\star,\widetilde w^\star)=
\begin{cases}
(0,0), & r\le \beta,\\[1mm]
\left(1-\beta/r\right)(\tilde b,\tilde w), & r>\beta.
\end{cases}
\]
\end{proposition}

\begin{proof}
Problem \eqref{eq:block-prox} is strictly convex because of the quadratic terms, hence has a unique minimizer (Appendix \ref{app:sec:hier-prox-converges}).

Let $x:=(b,\widetilde w)\in\R^{1+K}$ and $y:=(v,\widetilde U)\in\R^{1+K}$.
Define the convex regularizer on $\R^{1+K}$,
\[
h(x):=\alpha|b|+\bar\lambda\|\widetilde w\|_1+\beta\|x\|_2.
\]
Then \eqref{eq:block-prox} is exactly the proximal problem
\[
x^\star=\prox_h(y):=\argmin_x\ \half\|x-y\|_2^2+h(x).
\]
By standard subgradient optimality for proximal maps \citep{boyd2004convex, parikh2014proximal}, $x^\star$ is the unique point satisfying
\begin{equation}
\label{eq:opt-prop1}
y-x^\star\in \partial h(x^\star).
\end{equation}

\medskip
\noindent\textbf{Step 1: the coordinatewise $\ell_1$ prox.}
Let
\[
\tilde x:=(\tilde b,\tilde w):=\bigl(\soft_{\alpha}(v),\ \soft_{\bar\lambda}(\widetilde U)\bigr).
\]
The defining property of soft-thresholding is precisely the proximal optimality condition for the separable $\ell_1$ term:
\begin{equation}
\label{eq:soft-subgrad}
y-\tilde x\in \partial\bigl(\alpha|b|+\bar\lambda\|\widetilde w\|_1\bigr)\big|_{x=\tilde x}.
\end{equation}
(Indeed, for each coordinate, $\soft$ is the proximal map of the absolute value.)

\medskip
\noindent\textbf{Step 2: add the group $\ell_2$ term.}
If $\tilde x=0$, then $r=0$ and the claimed formula gives $x^\star=0$, which satisfies \eqref{eq:opt-prop1} because
$\partial(\beta\|x\|_2)\big|_{x=0}=\{u:\|u\|_2\le \beta\}$.

Assume now $\tilde x\neq 0$, so $r=\|\tilde x\|_2>0$.
If $r\le \beta$ then $\kappa=0$ and $x^\star=0$; this is optimal because $y\in\partial h(0)$ can be verified by combining
\eqref{eq:soft-subgrad} at $0$ with $\tilde x\in \beta\,\mathbb B_2$ (since $\|\tilde x\|_2=r\le \beta$).

Now consider $r>\beta$, so $\kappa=1-\beta/r\in(0,1)$ and define $x^\star:=\kappa\tilde x$.
Then $\|x^\star\|_2=\kappa r>0$ and therefore
\[
\partial(\beta\|x\|_2)\big|_{x=x^\star}=\left\{\beta\frac{x^\star}{\|x^\star\|_2}\right\}
=\left\{\beta\frac{\tilde x}{r}\right\}.
\]
Moreover,
\[
y-x^\star=(y-\tilde x)+(1-\kappa)\tilde x
\quad\text{with}\quad
1-\kappa=\beta/r.
\]
By \eqref{eq:soft-subgrad}, the first term belongs to $\partial(\alpha|b|+\bar\lambda\|\widetilde w\|_1)$ evaluated at $\tilde x$.
Since $x^\star=\kappa\tilde x$ has the same sign and sparsity pattern as $\tilde x$ (scaling by a positive scalar),
the same subgradient selection is valid at $x^\star$, i.e.
\[
y-\tilde x\in \partial(\alpha|b|+\bar\lambda\|\widetilde w\|_1)\big|_{x=x^\star}.
\]
The second term satisfies
\[
(1-\kappa)\tilde x=\frac{\beta}{r}\tilde x\in \partial(\beta\|x\|_2)\big|_{x=x^\star}.
\]
Combining the two pieces yields $y-x^\star\in \partial h(x^\star)$, proving \eqref{eq:opt-prop1} and hence optimality.
\end{proof}

\begin{remark}[Back to original variables and where $M$ appears]
At feature $j$, let $(v_j,\widetilde U_j)$ be the scaled-coordinate input to the prox (a block of $z_t$), and let
$U^{\mathrm{orig}}_j:=M\widetilde U_j$ denote the corresponding original-coordinate row.

With the shorthand $\bar\lambda=\lambda_w M$, Proposition~\ref{prop:local-opt-M} returns
\[
(\theta_j^+,\widetilde W_{j,\cdot}^+)
=
\kappa_j\bigl(\soft_{\alpha}(v_j),\ \soft_{\bar\lambda}(\widetilde U_j)\bigr),
\qquad
\kappa_j
=
\left(1-\frac{\beta}{\sqrt{\soft_{\alpha}(v_j)^2+\|\soft_{\bar\lambda}(\widetilde U_j)\|_2^2}}\right)_+.
\]
Mapping back yields
\[
W_{j,\cdot}^+ = M\widetilde W_{j,\cdot}^+ = \kappa_j\,M\,\soft_{\bar\lambda}(\widetilde U_j).
\]
Equivalently, since $\soft_{\bar\lambda}(\widetilde U_j)=\frac{1}{M}\soft_{\lambda_w M^2}(U^{\mathrm{orig}}_j)$,
\[
\boxed{
W_{j,\cdot}^+ = \kappa_j\,\soft_{\lambda_w M^2}(U^{\mathrm{orig}}_j),
}
\qquad
\kappa_j
=
\left(1-\frac{\beta}{\sqrt{\soft_{\alpha}(v_j)^2+\|\soft_{\lambda_w M^2}(U^{\mathrm{orig}}_j)\|_2^2/M^2}}\right)_+.
\]
When $\lambda_w=0$, the intermediate soft-thresholding disappears and we recover the previous special case
$W_{j,\cdot}^+=\kappa_j\,U^{\mathrm{orig}}_j$.
\end{remark}

\paragraph{Coordinate-wise (sequential) implementation.}
Per feature $j$, the proximal step is computed by two explicit substeps:
\begin{enumerate}[leftmargin=2em]
\item soft-threshold the gate coordinate: $\tilde v_j=\soft_\alpha(v_j)$;
\item shrink the concatenated block $(\tilde v_j,\widetilde U_j)$ by $\kappa_j=(1-\beta/r_j)_+$ with $r_j=\sqrt{\tilde v_j^2+\|\widetilde U_j\|_2^2}$.
\end{enumerate}
This is parallel across features.

\subsection{\texorpdfstring{Deterministic descent in function values (no convexity of $f$ needed)}{Deterministic descent in function values (no convexity of f needed)}}

\begin{assumption}[$L$-smoothness]
\label{relaxed:ass:smooth}
The function $f$ is $L$-smooth: for all $x,x'$,
\[
f(x')\le f(x)+\ip{\nabla f(x)}{x'-x}+\frac{L}{2}\norm{x'-x}^2.
\]
\end{assumption}

\begin{lemma}[A key proximal inequality]
\label{relaxed:lem:prox-opt}
Let $z\in\R^{d+dK}$ and let $x^+=\Psi(z)$ be defined by \eqref{relaxed:eq:psi-def}.
Then for every $x\in\R^{d+dK}$,
\begin{equation}
\label{eq:three-point}
\Omega(x^+) + \half\norm{x^+-z}^2
\;\le\;
\Omega(x) + \half\norm{x-z}^2 - \half\norm{x-x^+}^2.
\end{equation}
\end{lemma}

\begin{proof}
Define $\phi(x):=\Omega(x)+\half\|x-z\|^2$. Then $\phi$ is $1$-strongly convex.
Hence for all $x$ and all $g\in\partial\phi(x^+)$,
\[
\phi(x)\ge \phi(x^+)+\ip{g}{x-x^+}+\half\norm{x-x^+}^2.
\]
At the minimizer $x^+$ we can take $g=0$, yielding $\phi(x)\ge \phi(x^+)+\half\|x-x^+\|^2$, i.e.\ \eqref{eq:three-point}.
\end{proof}

\begin{lemma}[Deterministic one-step descent in $F$]
\label{relaxed:lem:det-descent}
Assume $f$ is $L$-smooth and $0<\eta\le \frac{1}{2L}$.
Let $g=\nabla f(x)$ and define $z=x-\eta g$ and $x^+=\Psi(z)$.
Then
\begin{equation}
\label{eq:det-descent}
F(x^+) \le F(x) - \frac{1}{4}\norm{x-x^+}^2.
\end{equation}
\end{lemma}

\begin{proof}
Apply Lemma~\ref{relaxed:lem:prox-opt} with $z=x-\eta\nabla f(x)$ and comparator $x$:
\[
\Omega(x^+) + \half\norm{x^+-(x-\eta\nabla f(x))}^2
\le \Omega(x) + \half\norm{x-(x-\eta\nabla f(x))}^2 - \half\norm{x-x^+}^2.
\]
The middle term is $\half\eta^2\|\nabla f(x)\|^2$.
Expand the left square:
\[
\norm{x^+-(x-\eta\nabla f(x))}^2
= \norm{x^+-x}^2 + 2\eta\ip{\nabla f(x)}{x^+-x} + \eta^2\norm{\nabla f(x)}^2.
\]
Cancel the $\eta^2\|\nabla f(x)\|^2$ terms and rearrange:
\[
\Omega(x^+) + \half\norm{x^+-x}^2 + \eta\ip{\nabla f(x)}{x^+-x}
\le \Omega(x) - \half\norm{x-x^+}^2.
\]
Rearrange:
\begin{equation}
\label{eq:key-rearrange}
\Omega(x^+) \le \Omega(x) - \norm{x-x^+}^2 - \eta\ip{\nabla f(x)}{x^+-x}.
\end{equation}
Now apply $L$-smoothness (Assumption~\ref{relaxed:ass:smooth}) with $x'=x^+$:
\[
f(x^+)\le f(x)+\ip{\nabla f(x)}{x^+-x}+\frac{L}{2}\norm{x^+-x}^2.
\]
Multiply by $\eta$ and add to \eqref{eq:key-rearrange}:
\[
F(x^+) \le F(x) - \norm{x-x^+}^2 + \frac{\eta L}{2}\norm{x^+-x}^2.
\]
With $\eta\le 1/(2L)$, we get $1-\eta L/2 \ge 3/4$, hence
\[
F(x^+) \le F(x) - \frac{3}{4}\norm{x-x^+}^2,
\]
which is stronger than \eqref{eq:det-descent}.
\end{proof}

\subsection{A proximal-PL condition (function-value landscape)}

\begin{definition}[Proximal-gradient mapping]
\label{def:PG}
Define the (scaled) proximal-gradient mapping for $F(x)=\eta f(x)+\Omega(x)$ as
\[
G(x):= x-\Psi\bigl(x-\eta\nabla f(x)\bigr).
\]
\end{definition}

\begin{definition}[Proximal-PL condition in function values]
\label{def:pPL}
We say $F$ satisfies a proximal-PL condition with constant $\mu>0$ if for all $x$,
\begin{equation}
\label{eq:pPL}
\half\norm{G(x)}^2 \ge \mu\bigl(F(x)-F^\star\bigr),
\qquad F^\star := \inf_{x} F(x).
\end{equation}
\end{definition}

\subsection{Stochastic recursion: noise only (no sign-switch penalty)}

\begin{assumption}[Unbiased gradients and bounded conditional variance]
\label{relaxed:ass:noise}
The stochastic gradients satisfy $\E[g_t\mid x_t]=\nabla f(x_t)$ and
\[
\E\bigl[\norm{\xi_t}^2\mid x_t\bigr] \le \sigma^2,
\qquad
\xi_t:=g_t-\nabla f(x_t).
\]
\end{assumption}

\subsubsection{A conditional descent inequality}

\begin{lemma}[Stochastic descent with a noise cross term]
\label{relaxed:lem:stoch1}
Assume Assumption~\ref{relaxed:ass:smooth} and Assumption~\ref{relaxed:ass:noise}. If $\eta\le 1/L$, then almost surely
\[
F(x_{t+1})
\le
F(x_t)
-\frac{1-\eta L}{2}\norm{x_{t+1}-x_t}^2
-\eta \ip{\xi_t}{x_{t+1}-x_t}.
\]
\end{lemma}

\begin{proof}
Apply Lemma~\ref{relaxed:lem:prox-opt} to $x_{t+1}=\Psi(z_t)$ with comparator $x=x_t$ and $z_t=x_t-\eta g_t$:
\[
\Omega(x_{t+1})+\half\norm{x_{t+1}-z_t}^2 \le \Omega(x_t)+\half\norm{x_t-z_t}^2.
\]
Expand the squares:
\[
\norm{x_{t+1}-z_t}^2
=
\norm{x_{t+1}-x_t}^2 + 2\eta\ip{g_t}{x_{t+1}-x_t} + \eta^2\norm{g_t}^2,
\quad
\norm{x_t-z_t}^2=\eta^2\norm{g_t}^2.
\]
Cancel the $\eta^2\|g_t\|^2$ terms and rearrange:
\[
\Omega(x_{t+1}) + \half\norm{x_{t+1}-x_t}^2 + \eta\ip{g_t}{x_{t+1}-x_t}\le \Omega(x_t).
\]
By $L$-smoothness (Assumption~\ref{relaxed:ass:smooth}) with $x'=x_{t+1}$:
\[
f(x_{t+1})\le f(x_t)+\ip{\nabla f(x_t)}{x_{t+1}-x_t}+\frac{L}{2}\norm{x_{t+1}-x_t}^2.
\]
Multiply by $\eta$ and add:
\[
F(x_{t+1})
\le
F(x_t)
-\frac{1-\eta L}{2}\norm{x_{t+1}-x_t}^2
-\eta\ip{g_t-\nabla f(x_t)}{x_{t+1}-x_t}.
\]
Since $\xi_t=g_t-\nabla f(x_t)$, this is the claim.
\end{proof}

\subsubsection{Removing the noise cross term (Young's inequality)}

\begin{corollary}[Conditional expectation form]
\label{relaxed:cor:stoch-descent}
Under the assumptions of Lemma~\ref{relaxed:lem:stoch1}, we have
\[
\E[F(x_{t+1})\mid x_t]
\le
F(x_t)
-\frac{1-\eta L}{4}\E[\norm{x_{t+1}-x_t}^2\mid x_t]
+\frac{\eta^2}{1-\eta L}\E[\norm{\xi_t}^2\mid x_t].
\]
\end{corollary}

\begin{proof}
Start from Lemma~\ref{relaxed:lem:stoch1} and take conditional expectation:
\[
\E[F(x_{t+1})\mid x_t]
\le
F(x_t)
-\frac{1-\eta L}{2}\E[\norm{x_{t+1}-x_t}^2\mid x_t]
-\eta\E[\ip{\xi_t}{x_{t+1}-x_t}\mid x_t].
\]
Use Cauchy--Schwarz and Young:
\[
-\eta\ip{\xi_t}{x_{t+1}-x_t}
\le
\eta\norm{\xi_t}\,\norm{x_{t+1}-x_t}
\le
\frac{1-\eta L}{4}\norm{x_{t+1}-x_t}^2+\frac{\eta^2}{1-\eta L}\norm{\xi_t}^2.
\]
Take conditional expectation and plug in.
\end{proof}

\subsubsection{Relating the stochastic step length to the deterministic prox-gradient mapping}

Define the \emph{deterministic} proximal-gradient step at $x_t$:
\[
\bar z_t := x_t-\eta\nabla f(x_t),\qquad
\bar x_{t+1} := \Psi(\bar z_t).
\]

\begin{lemma}[Lower bound on the stochastic step length (convex prox case)]
\label{lem:lowerbound-convex}
For any $t$,
\[
\E[\norm{x_{t+1}-x_t}^2\mid x_t]
\ge
\frac12\norm{G(x_t)}^2-\E[\norm{x_{t+1}-\bar x_{t+1}}^2\mid x_t],
\]
and moreover
\[
\E[\norm{x_{t+1}-\bar x_{t+1}}^2\mid x_t]\le \eta^2\E[\norm{\xi_t}^2\mid x_t].
\]
\end{lemma}

\begin{proof}
Note that $x_{t+1}-x_t = (x_{t+1}-\bar x_{t+1})+(\bar x_{t+1}-x_t)$ and $\bar x_{t+1}-x_t=-G(x_t)$ by definition.
Hence
\[
\norm{x_{t+1}-x_t}^2
=
\norm{(x_{t+1}-\bar x_{t+1})-G(x_t)}^2
\ge
\frac12\norm{G(x_t)}^2 - \norm{x_{t+1}-\bar x_{t+1}}^2.
\]
Take conditional expectation to get the first inequality.

For the second inequality, note that $z_t=\bar z_t-\eta\xi_t$ and by Lemma~\ref{lem:firm}, $\Psi$ is $1$-Lipschitz:
\[
\norm{x_{t+1}-\bar x_{t+1}}
=
\norm{\Psi(z_t)-\Psi(\bar z_t)}
\le
\norm{z_t-\bar z_t}
=
\eta\norm{\xi_t}.
\]
Square and take conditional expectation.
\end{proof}

\subsubsection{Master recursion in function values}

\begin{theorem}[Master recursion under proximal-PL: linear-to-noise-ball]
\label{relaxed:thm:master}
Assume $f$ is $L$-smooth (Assumption~\ref{relaxed:ass:smooth}) and $0<\eta\le \frac{1}{2L}$.
Assume the stochastic gradient model (Assumption~\ref{relaxed:ass:noise}).
Assume $F(x)=\eta f(x)+\Omega(x)$ satisfies proximal-PL with constant $\mu>0$ (Definition~\ref{def:pPL}).

Let $\Delta_t := \E[F(x_t)-F^\star]$, where $F^\star:=\inf_x F(x)$.
Then the iterates \eqref{relaxed:eq:update} satisfy
\begin{equation}
\label{eq:master}
\Delta_{t+1}
\le
\left(1-\frac{\mu(1-\eta L)}{4}\right)\Delta_t
+
\eta^2\sigma^2\left(\frac{1-\eta L}{4}+\frac{1}{1-\eta L}\right).
\end{equation}
In particular, $\Delta_t$ converges linearly to a neighborhood of radius
$O\!\left(\frac{\eta^2\sigma^2}{\mu}\right)$, and if $\sigma=0$ then $\Delta_t$ converges linearly to $0$.
\end{theorem}

\begin{proof}
Start from Corollary~\ref{relaxed:cor:stoch-descent} and bound the conditional step length using Lemma~\ref{lem:lowerbound-convex}:
\begin{align*}
\E[F(x_{t+1})\mid x_t]
&\le
F(x_t)
-\frac{1-\eta L}{4}\E[\norm{x_{t+1}-x_t}^2\mid x_t]
+\frac{\eta^2}{1-\eta L}\E[\norm{\xi_t}^2\mid x_t]\\
&\le
F(x_t)
-\frac{1-\eta L}{4}\left(\frac12\norm{G(x_t)}^2-\E[\norm{x_{t+1}-\bar x_{t+1}}^2\mid x_t]\right)
+\frac{\eta^2}{1-\eta L}\E[\norm{\xi_t}^2\mid x_t].
\end{align*}
Using the second bound in Lemma~\ref{lem:lowerbound-convex}, we have
$\E[\|x_{t+1}-\bar x_{t+1}\|^2\mid x_t]\le \eta^2\E[\|\xi_t\|^2\mid x_t]$.
Substitute this and simplify:
\[
\E[F(x_{t+1})\mid x_t]
\le
F(x_t)
-\frac{1-\eta L}{8}\norm{G(x_t)}^2
+
\eta^2\E[\norm{\xi_t}^2\mid x_t]\left(\frac{1-\eta L}{4}+\frac{1}{1-\eta L}\right).
\]
Now apply proximal-PL (Definition~\ref{def:pPL}):
\[
\half\norm{G(x_t)}^2 \ge \mu(F(x_t)-F^\star)
\quad\Longrightarrow\quad
\norm{G(x_t)}^2 \ge 2\mu(F(x_t)-F^\star).
\]
Therefore,
\[
\E[F(x_{t+1})-F^\star\mid x_t]
\le
\left(1-\frac{\mu(1-\eta L)}{4}\right)(F(x_t)-F^\star)
+
\eta^2\E[\norm{\xi_t}^2\mid x_t]\left(\frac{1-\eta L}{4}+\frac{1}{1-\eta L}\right).
\]
Finally, take total expectation and use $\E[\|\xi_t\|^2\mid x_t]\le \sigma^2$ to obtain \eqref{eq:master}.
\end{proof}

\subsection{Interpretations and implementation notes}

\begin{remark}[What improves relative to Hier-Prox]
The original Hier-Prox mapping is piecewise-defined and can jump when active sets/sign patterns change; this creates an additional instability term
(e.g.\ a ``jump size'' squared times a switch probability) in a function-value recursion.
With the convex surrogate \eqref{eq:OmegaSG}, the prox map is globally firmly nonexpansive (Lemma~\ref{lem:firm}),
so the recursion \eqref{eq:master} has \emph{noise only}.
\end{remark}

\begin{remark}[What we give up]
We no longer enforce the hard hierarchy $\|W_{j,\cdot}\|_\infty \le M|\theta_j|$.
Instead, we encourage the gate and the row weights to shrink \emph{jointly} via the convex sparse-group penalty on $(\theta_j, W_{j,\cdot}/M)$.
This is often favorable when stability and simple proximal geometry are prioritized.
\end{remark}

\paragraph{Two closed-form prox steps (featurewise).}
Given the pre-prox iterate in scaled coordinates for feature $j$, i.e.\ $(v_j,\widetilde U_j)$, do:
\begin{enumerate}[leftmargin=2em]
\item Soft-threshold the gate: $\tilde v_j=\soft_{\alpha}(v_j)$.
\item Group shrink: $r_j=\sqrt{\tilde v_j^2+\|\widetilde U_j\|_2^2}$ and $\kappa_j=(1-\beta/r_j)_+$.
Return $(\theta_j^+,\widetilde W_{j,\cdot}^+)=\kappa_j(\tilde v_j,\widetilde U_j)$ and map back $W_{j,\cdot}^+=M\widetilde W_{j,\cdot}^+$.
\end{enumerate}

\paragraph{$M$-tuning intuition.}
In the group-shrink step, $\|\widetilde U_j\|_2=\|U_j^{\mathrm{orig}}\|_2/M$.
Thus increasing $M$ makes $\|\widetilde U_j\|_2$ smaller and tends to \emph{increase} shrinkage (smaller $\kappa_j$),
which matches the intuition that large $M$ enforces a tighter coupling between the gate and the row weights.

\begin{remark}
    This convex relaxation approach provides theoretical evidence that decoupling the updates of the Lasso gates and the neural network parameters can lead to significantly better training stability.
\end{remark}

\end{document}